\documentclass[review]{siamart220329}

\usepackage{hyperref}
\usepackage{booktabs}
\usepackage{multirow}
\usepackage{subfig}
\usepackage{graphicx}
\usepackage{caption}
\usepackage{tikz}
\usepackage{tikz-qtree}
\usetikzlibrary{matrix}
\usepackage{amsfonts}
\usepackage{mathdots}
\usepackage{amsmath}
\usepackage{amssymb}
\usepackage{glossaries} 
\usepackage{mathtools}
\usepackage{enumerate} 
\usepackage{stmaryrd} 
\usepackage{xcolor} 
\usepackage{algorithmic}
\usepackage{algorithm}
\usepackage[margin=1in]{geometry} 
\usepackage{eqparbox}

\newcommand{\cdf}[1]{{#1}}

\DeclarePairedDelimiter\abs{\lvert}{\rvert}%
\DeclarePairedDelimiter\norm{\lVert}{\rVert}%

\newcommand{\nys}{Nystr{\"o}m }

\newcommand{\Var}{\text{Var}}

\newcommand{\ksig}{\kappa_{\rho}}
\newcommand{\RSsig}{R_{S,\rho}}
\newcommand{\cdcd}{(\cdot,\cdot)}

\newcommand{\dist}[1]{\text{dist}(#1)}

\nolinenumbers

\makeatletter
\let\oldabs\abs
\def\abs{\@ifstar{\oldabs}{\oldabs*}}
\let\oldnorm\norm
\def\norm{\@ifstar{\oldnorm}{\oldnorm*}}
\makeatother

\title{Posterior Covariance Structures in Gaussian Processes}
\author{Difeng Cai\thanks{Department of Mathematics, Southern Methodist University, Dallas, TX 75205 (\email{ddcai@smu.edu})}
\and Edmond Chow\thanks{School of Computational Science and Engineering, Georgia Institute of Technology, Atlanta, GA 30332 (\email{echow@cc.gatech.edu}) The research of E. Chow is
supported by NSF award OAC 2003683.}
\and Yuanzhe Xi\thanks{Department of Mathematics, Emory University, Atlanta, GA 30322 (\email{yxi26@emory.edu}). The research of Y. Xi is supported by
NSF awards DMS-2038118 and DMS-2338904.}
}
\date{}

\headers{Posterior Covariance in GP}{D. Cai, E. Chow and Y. Xi} 

\begin{document}
\maketitle
\begin{abstract}
In this paper, we present a comprehensive analysis of the posterior covariance function in Gaussian processes, 
with applications to the posterior covariance matrix.
Our geometric analysis reveals how the Gaussian covariance's bandwidth parameter and the spatial distribution of the observations influence the posterior covariance as well as the corresponding covariance matrix, enabling straightforward identification of areas with high or low covariance {in magnitude}. 
Drawing inspiration from the a posteriori error estimation techniques in adaptive finite element methods, we also propose several indicators to efficiently measure the absolute posterior covariance function, which can be used for efficient covariance matrix approximation and preconditioning. We conduct a wide range of experiments to illustrate our theoretical findings and their practical applications.


\end{abstract}




\section{Introduction} 
\label{sec:intro}
Let $\Omega$ be a domain in $\mathbb{R}^d$
and $f:\Omega\to \mathbb{R}$ be an unknown function. A fundamental problem in statistical machine learning is to learn $f$ from possibly noisy observations 
$y_i = f(x_i)+\epsilon_i$,
where $\epsilon_i\sim \mathcal{N}(0,\tau^2)$ is a Gaussian random noise with noise level $\tau\geq 0$ independent of $i$ and
$\tau=0$ corresponds to the noise-free observations. 
Gaussian processes (GPs) are widely utilized for prediction tasks that also require quantifying uncertainty. 
In this framework, GPs model the observations at any locations $S=\{x_i\}_{i=1}^N\subseteq \Omega$ and the unknown function values over any finite set $X_*\subseteq\Omega$ as a (prior) joint Gaussian distribution:
\begin{equation}
\label{eq:GPprior}
    \begin{bmatrix}
        y\\
        f_*
    \end{bmatrix}
    \;\sim \;
    \mathcal{N}\left( 0, \begin{bmatrix}
       K_{SS}+\tau^2 I & K_{SX_*} \\
        K_{X_* S} & K_{X_* X_*} 
    \end{bmatrix}\right),
\end{equation}
where $y\in\mathbb{R}^N$ is the vector of observations $y_i$ $(i=1,\dots,N)$ and $f_*$ is the vector of {predicted values} at $X_*$. This probabilistic approach allows for an effective handling of uncertainty in predictions.
Here 
$$K_{UV} := [\kappa(u,v)]_{\substack{u\in U\\ v\in V}}$$ 
is the (prior) covariance matrix whose rows correspond to $U\subseteq\Omega$ and columns correspond to $V\subseteq\Omega$.
The definition of $K_{UV}$ also allows $U$ or $V$ to be a single point in $\Omega$ and we write $K_{uv} = \kappa(u,v)$ to denote the prior covariance function evaluated at $u,v\in\Omega$.
The (prior) covariance function usually contains a positive parameter that represents the ``strength'' of spatial correlation.
For example, 
in the Gaussian covariance
\cdf{
\begin{equation}
\label{eq:cov}
    \ksig(u,v) = \sigma^2\exp\left(-\frac{\norm{u-v}^2}{2\rho^2}\right),
\end{equation}
the hyperparameter $\rho>0$ is known as the \emph{length scale} or the \emph{bandwidth} parameter, and the parameter $\sigma^2$ denotes the prior variance.
In radial basis function literature \cite{kansa1992,rbf2013fornberg,rbf2015fornberg,RBFbook},
$(\sqrt{2}\rho)^{-1}$ is called the shape parameter.
Note that $\sigma^2$ can be simply viewed as a multiplicative constant in the covariance $\ksig$, while $\rho$ has a more sophisticated effect on $\ksig$ and the posterior distribution of GP.
Since we are interested in studying the impact of $\rho$, unless otherwise stated, we assume in the following that $\sigma=1$.
}
The parameter $\rho$ plays a fundamental role in GPs,
as the performance of GP regression highly hinges on $\rho$ and the best value of $\rho$ is found through training. In this paper, we examine the impact of $\rho$ and the observation set $S$ on several computationally demanding operations involved in performing GPs. Our aim is to develop more efficient and robust matrix operations guided by these rigorous analyses. We primarily focus on the noise-free scenario with $\tau=0$ throughout our analysis. Additionally, we provide a brief discussion on scenarios where $\tau>0$.
\cdf{The Gaussian covariance in \eqref{eq:cov} is a special example of a Mat{\'e}rn covariance \cite{matern1960thesis} and the study in this paper constitutes the first step in analyzing the more general Mat{\'e}rn covariance family \cite{stein1993bayes,stein2012book} commonly used in GP models \cite{gp2006book}.}

Given the observations $(S,y)$ and the prior joint distribution in \eqref{eq:GPprior} with prior covariance in \eqref{eq:cov},
a GP model computes the posterior distribution of $f_*$ as follows:
\begin{equation}
    \label{eq:GPpost}
    f_*|X_*,S,y \;\sim\; \mathcal{N}\left(K_{X_* S}K_{SS}^{-1}y,\; \RSsig(X_*,X_*) \right),
\end{equation}
where the \emph{posterior covariance function} $\RSsig\cdcd$ conditioned on $S$ is defined as
\begin{equation}
   \label{eq:R}
    \RSsig\cdcd:\mathbb{R}^d\times\mathbb{R}^d\;\to\; \mathbb{R},\quad (u,v)\;\to\; \ksig(u,v)-K_{uS}K_{SS}^{-1}K_{Sv}. 
\end{equation}
\cdf{The \emph{posterior variance} at a point $x$ is 
\begin{equation}
\label{eq:Var}
    \Var(x) = \RSsig(x,x) = \ksig(x,x)-K_{xS}K_{SS}^{-1}K_{Sx}.
\end{equation}
}

The study of the posterior covariance function $\RSsig\cdcd$ is crucial to developing scalable covariance matrix calculations in GPs. For example, when discretized at finitely many locations $X_*\subseteq \Omega$, $\RSsig\cdcd$ gives rise to the posterior covariance matrix:
$$\RSsig(X_*,X_*)= K_{X_* X_*} - K_{X_*S}K_{SS}^{-1}K_{SX_*}.$$
If the set $X_*$ contains $S$, $\RSsig(X_*,X_*)$ can be viewed as the error matrix for the skeleton approximation \cite{Tyrtyshnikov1996,pseudo1997} or the \nys approximation \cite{nys2001} $K_{X_*S}K_{SS}^{-1}K_{SX_*}$ to $K_{X_*X_*}$, where $S$ is treated as the set of landmark points. 
A good understanding of the magnitude in $\RSsig(X_*,X_*)$ helps to construct a sparse correction to the low-rank format for better accuracy in approximating $K_{X_*X_*}$.
This is especially the case when the bandwidth is relatively small and low-rank methods lose their efficiency since $K_{X_*X_*}$ does \emph{not} have fast decaying singular values. Moreover, the study of the \emph{continuous} $\abs{\RSsig\cdcd}$ can be used to analyze the approximation property of the \emph{rectangular} matrix $K_{Y_1 Y_2}$, which often appears in the prediction stage of GPs when $Y_1$ and $Y_2$ correspond to the testing and training data, respectively. A discrete version of the posterior covariance function $\RSsig\cdcd$ also appears in preconditioning techniques such as factorized sparse approximate inverse (FSAI) preconditioners for solving Gaussian linear systems (cf. \cite{fsai1993,fsai2000,AFN,XU2022102956,additive,huang2023highp}).
For preconditioning based on FSAI, of fundamental importance is the choice of a nonzero pattern in the sparse factor.
A more straightforward understanding of $\abs{\RSsig(x,y)}$  can be crucial for accurately specifying the nonzero pattern. Finally,
the posterior variance at $x\in\Omega$, namely $\RSsig(x,x)$, is widely used in Bayesian optimization and optimal experimental design \cite{krause2008,bayes2010CG,bayes1995review,bayesbook2023,10.1145/3589335.3651456} where the study of the structure of $\abs{\RSsig(x,y)}$ will be instrumental in accelerating optimal design algorithms (in sensor placements, for example)
when there is limited storage for the full covariance matrix (prior or posterior). 

The primary challenge in analyzing the structures of the posterior covariance function stems from the wide range of possible hyperparameters. The optimal hyperparameter value is usually unknown initially and must be estimated via maximum likelihood estimation over training data \cite{bishopbook,gp2006book}:
\begin{equation}
\label{eq:MLE}
    \cdf{\max_{\rho,\sigma>0} -\frac{1}{2}\ln \vert K_{SS}\vert - \frac{1}{2}y^TK_{SS}^{-1}y-\frac{N}{2}\ln(2\pi).}
\end{equation}
\cdf{In Gaussian processes, the optimal values for $\rho$ and $\sigma$ are learned by the maximum likelihood estimation above.}
Frequent updates to the hyperparameter are required throughout the iterative optimization process, which causes the structures of $\ksig(x,y)$ and $\RSsig(x,y)$ to change correspondingly. This dynamic nature significantly complicates the analysis. 
A comprehensive understanding of how different factors, such as $\rho$ and $S$, affect the posterior distribution is crucial for ensuring the efficiency of numerical algorithms, such as approximating the dense matrix $K_{SS}$ or $K_{SS}^{-1}y$, throughout the entire optimization process.

Our theoretical analysis aims to reveal the relationships between $\abs{\RSsig(x,y)}$ and the parameters $S$, $\rho$, $x$, and $y$, focusing on obtaining straightforward geometric characterizations of $\abs{\RSsig(x,y)}$ that do not require calculation of $K_{SS}$ or matrix inversion. In particular, our analysis can be used to identify regions within $\Omega\times\Omega$ where high and low values of $\abs{\RSsig(x,y)}$ occur and to estimate the distribution of the absolute posterior covariance function $\abs{\RSsig\cdcd}$, based on the prior Gaussian covariance function $\ksig\cdcd$. This analysis draws inspiration from the concept of \emph{a posteriori error estimation} used in the adaptive finite element method for solving Partial Differential Equations (PDEs), as seen in references such as \cite{verf1994,ladeveze1983,ZZ1987,bernardi2000,verf2013book,ainsworth2011confusion,localL2,cai2020equi}, \cdf{where the basic idea is to design a computable quantity (called an error \emph{indicator}) to indicate the numerical approximation error without knowing the true solution and can be used to help identify where the error is relatively large or small in the domain.}  
Based on the theoretical insights, we develop practical indicators to assess the distribution of $\abs{\RSsig\cdcd}$, enhancing our ability to predict and understand this complex function.  The rest of the paper is organized as follows. Section \ref{sec:prelim} provides a few illustrating examples to show quite different posterior covariance patterns for different bandwidth values and observation data.
Section \ref{sec:theory} presents the theoretical analysis to understand the phenomenon as well as posterior covariance indicators to efficiently predict the posterior covariance.
Extensive numerical experiments are provided in Section \ref{sec:numerical} to discuss the diverse posterior covariance function patterns using the theory developed, and to illustrate several applications. Finally, concluding remarks are drawn in Section \ref{sec:conclusion}.

\paragraph{Notation}
We use $\norm{\cdot}$ to denote the $l_2$-norm of vectors or the spectral norm of matrices.
$\dist{x,Y}$ denotes the $l_2$ distance from $x$ to $Y$.
In case $Y$ is a set, $\dist{x,Y}:=\min_{t\in Y} \dist{x,t}$.

\section{Preliminary Observations}
\label{sec:prelim}
In this section, we provide illustrative examples of the posterior covariance function $\RSsig\cdcd$ to demonstrate the effects of the distribution of the dataset $S$ and the bandwidth parameter $\rho$ on the function $\abs{\RSsig\cdcd}$. These examples are crucial for highlighting the significant role played by these parameters in determining the intrinsic characteristics.

To illustrate the impact of the observation locations $S$ and bandwidth $\rho$ on $\abs{\RSsig\cdcd}$ over $\Omega\times\Omega$, we consider the following setup with $\Omega=[0,1]$.
The observation data $S$ in $\Omega$ is constructed to be either uniformly distributed or non-uniformly distributed as below:
$$\text{uniform:}\; S=\{0.02, 0.26, 0.5, 0.74, 0.98\}, \quad \text{non-uniform:}\; S=\{0.02,0.12,0.22,0.6,0.98\}.$$
\cdf{In this example, we choose $\rho$ to be either $0.1$ or $0.4$.
As shown in Figure \ref{fig:prelim}, these two values are able to demonstrate the substantially different structures of $\abs{\RSsig\cdcd}$ in the two regimes $\rho\to 0$ and $\rho\to\infty$, respectively.}

Given the setup above, the posterior covariance (in magnitude) $\abs{\RSsig\cdcd}$ can display dramatically different patterns for different $S$ and $\rho$. For example, 
Figure \ref{fig:prelim} illustrates the three cases below:
(a) uniform $S$ and $\rho=0.1$;
(b) uniform $S$ and $\rho=0.4$;
(c) non-uniform $S$ and $\rho=0.1$.
The five blue dots in each heat map plot of $\abs{\RSsig}$ are $(x_i,x_i)$ for $x_i\in S$.
Figure \ref{fig:prelim3} shows two curves corresponding to the cross sections at $y=0.2$ and $y=0.4$ in $\abs{\RSsig(x,y)}$.

It can be seen that both $S$ and $\rho$ have a substantial impact on the posterior covariance $\abs{\RSsig}$, in particular, the regions in $\Omega\times\Omega$ where larger values or smaller values of $\abs{\RSsig}$ occur.
Moreover, the value of $\RSsig(x,y)$ is also sensitive to the relative location between $x$ and $y$, and the mechanism is not clear yet, given the different patterns in Figure \ref{fig:prelim1} and Figure \ref{fig:prelim2}. 
In the following sections, we aim to delve into these phenomena through theoretical investigation and develop indicators to predict the posterior covariance distribution $\abs{\RSsig}\cdcd$ \emph{without} having to compute the function explicitly.

\begin{figure}[htbp] 
\centering 
\subfloat[$\rho=0.1$, uniform $S$]{ 
\label{fig:prelim1}
\includegraphics[scale=.4]{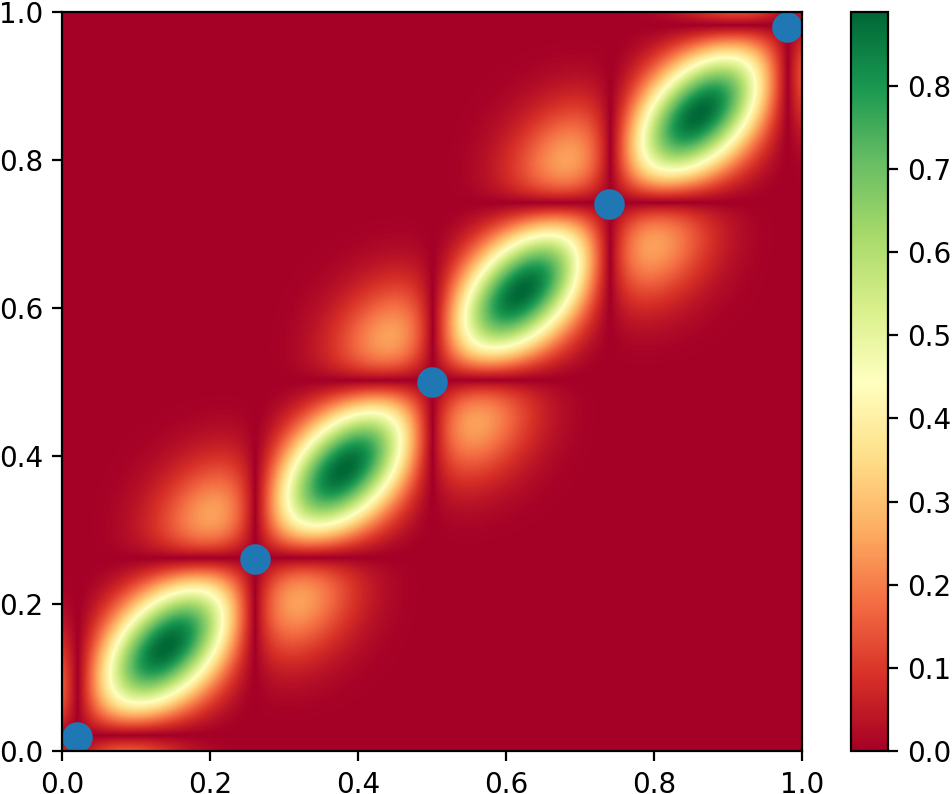}}
\hspace{30pt}
\subfloat[$\rho=0.4$, uniform $S$]{
\label{fig:prelim2}
\includegraphics[scale=.4]{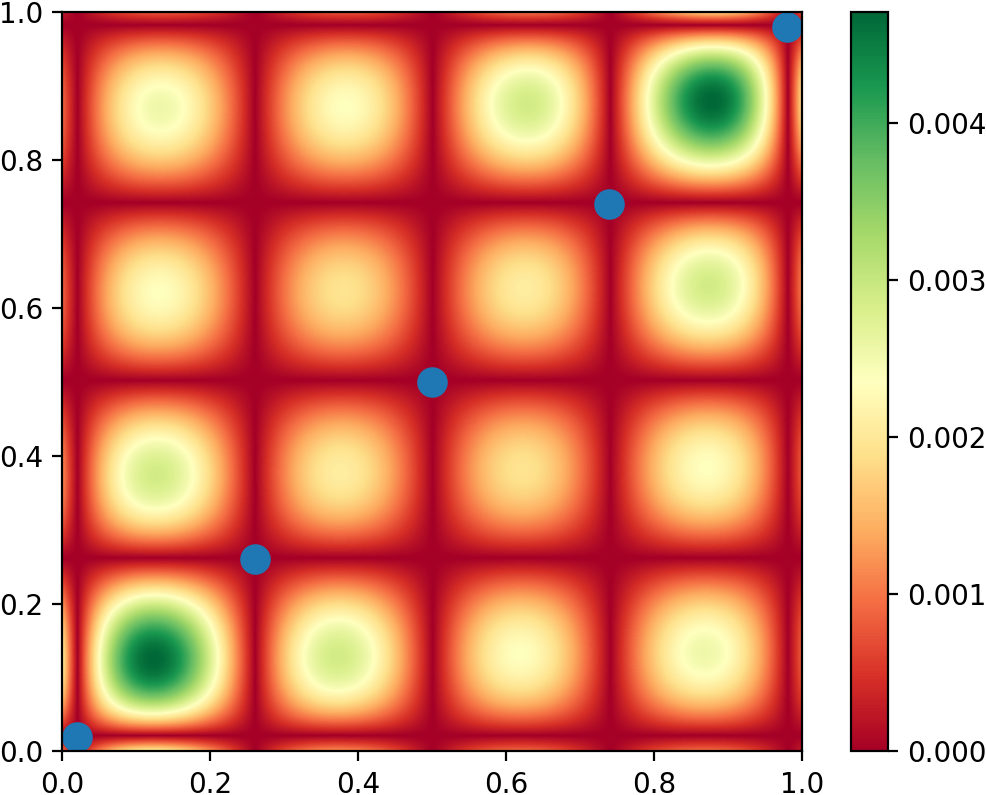}}
\hspace{30pt}
\subfloat[$\rho=0.1$, non-uniform $S=\{0.02,0.12,0.22,0.6,0.98\}$; Two curves show cross sections $|\RSsig(x,0.2)|$ and $|\RSsig(x,0.4)|$]{
\label{fig:prelim3}
\includegraphics[scale=.4]{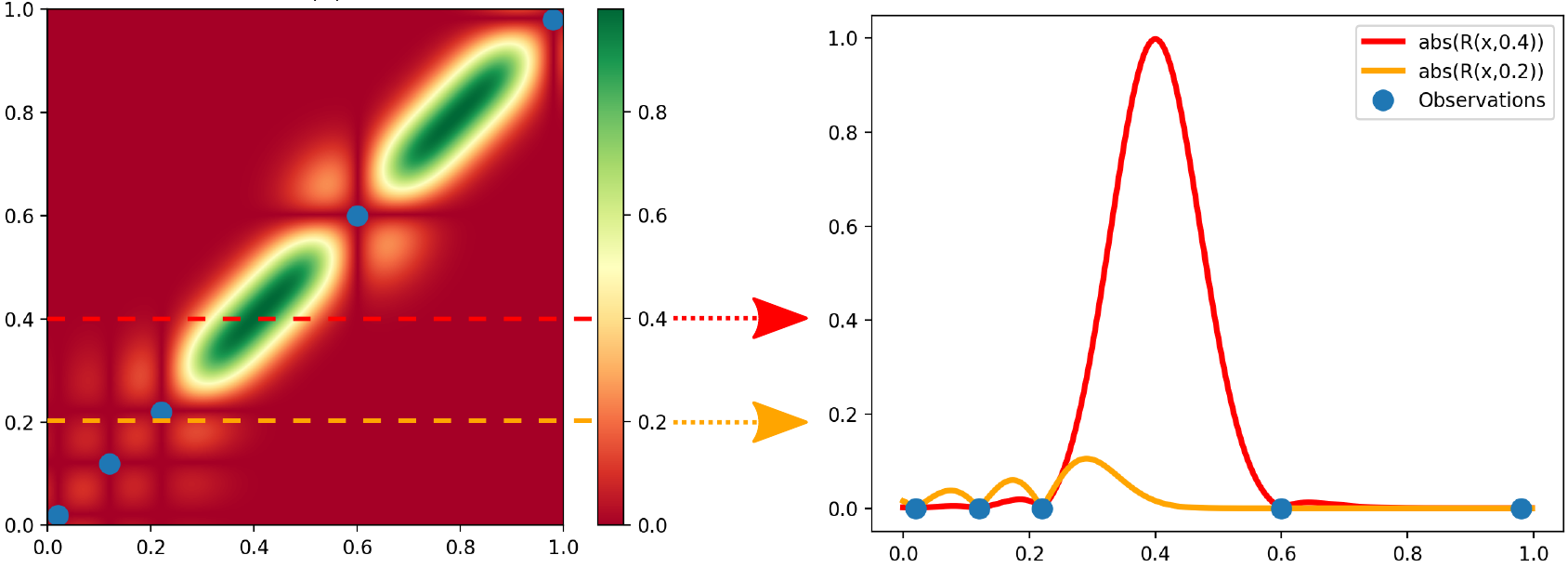}
}
\caption{$|\RSsig(x,y)|$ over $[0,1]\times [0,1]$: different $\rho$ and different $S$ (5 blue dots).} 
\label{fig:prelim}
\end{figure}

\section{Theoretical Analysis and Posterior Covariance Indicators} 
\label{sec:theory}
In this section, we analyze the magnitude of $\abs{\RSsig(x,y)}$ throughout the domain $(x,y) \in \Omega \times \Omega$. Our goal is to identify regions where this magnitude exhibits notably large or small values and to develop efficient indicators for its quantification.
Since the parameters $S$ and $\rho$ significantly affect the variations in $\RSsig$, as demonstrated in Figure \ref{fig:prelim}, our theoretical framework will emphasize the roles of $\rho$ and $S$. We will see that $\rho$ is a critical factor that affects the structure of $\RSsig(x,y)$ and we divide the discussion into two cases: \cdf{$\rho\to 0$ and $\rho\to\infty$, which, for simplicity, will be referred as ``small'' $\rho$ and ``large'' $\rho$ regimes, respectively.
The two regimes give representative structures of $\RSsig$ that will be crucial to understanding the distinct patterns as $\rho$ varies in the whole interval of $(0,\infty)$.
The usage of terminology (``large/small'') follows the widely adopted convention in mathematics, including the well-known ``law of large numbers'' in probability theory (where \emph{large} indicates the asymptotic regime of the number of samples approaching infinity) and ``small-scale parameter'' in numerical PDEs such as the study of convection-reaction-diffusion equations (where the diffusion coefficient can be arbitrarily \emph{small}, i.e. approaching $0$).
A value of $\rho$ will be categorized as ``large'' if it yields the typical structure found in the regime $\rho\to\infty$.
Analysis in this section will help to understand the two regimes quantitatively.
} 
\cdf{The two regimes are discussed in Section \ref{sub:Case I} and Section \ref{sub:Case II}, respectively. An illustration of different scenarios is presented in Section \ref{sub:plotbounds}.}
Firstly, it is easy to see where $\RSsig(x,y)$ must vanish, as stated in the theorem below.
\begin{theorem}
\label{thm:vanish}
For any finite subset $S\subseteq\mathbb{R}^d$,
define
\begin{equation}
\label{eq:RxyDef}
    \RSsig(x,y) := \ksig(x,y) - K_{xS}K_{SS}^{-1}K_{Sy},\quad x,y\in\mathbb{R}^d.
\end{equation}
Then for any $x,y$,
\begin{equation*}
   \RSsig(s,y) = \RSsig(x,s)=0\quad \forall s\in S. 
\end{equation*}
\end{theorem}
\begin{proof}
    For any $x$, consider the vector $\RSsig(x,S)$.
    We compute that
    $$\RSsig(x,S)=\ksig(x,S)-K_{xS}K_{SS}^{-1}K_{SS}=K_{xS}-K_{xS}=0.$$
    Similarly, it is easy to show that $\RSsig(S,x)=0$ for any $x\in\mathbb{R}^d$.
    This proves the theorem.
\end{proof}
From an analytical point of view, Theorem \ref{thm:vanish} states the interpolation property of the finite-rank approximation $f_S(x,y):=K_{xS}K_{SS}^{-1}K_{Sy}$ to the kernel $\ksig(x,y)$.
Namely, $f_S(x,y)$ coincides with $\ksig(x,y)$ whenever $x\in S$ or $y\in S$.
Furthermore, note that $\RSsig(x,y)$ is globally smooth due to the smoothness of the Gaussian kernel.
Therefore, Theorem \ref{thm:vanish} implies that $\RSsig(x,y)\approx 0$ if $\dist{x,S}\approx 0$ or $\dist{y,S}\approx 0$. More refined analysis will be presented in Section \ref{sub:Case I} for the small $\rho$ case and Section \ref{sub:Case II} for the large $\rho$ case. The analysis will rely on the Lipschitz constant for the Gaussian kernel. 
An estimate of the Lipschitz constant is included in the lemma below for completeness.

\begin{lemma}
\label{prop:Lip}
Consider $\ksig(x,y)=\exp(-\frac{\norm{x-y}^2}{2\rho^2})$ as a function of the first variable $x$. 
Namely, $f(x):=\ksig(x,y)$ where $y$ is viewed as a constant.
The Lipschitz constant $L$ satisfies the following estimate:
\begin{equation}
\label{eq:Lip}
   L:=\sup_{u\neq v}\frac{|f(u)-f(v)|}{\norm{u-v}}\leq \sup_{x\in \mathbb{R}^d} \norm{\nabla_x \ksig}\leq \frac{1}{\rho\sqrt{e}}.
\end{equation}
\end{lemma}
\begin{proof}
It can be computed that
\begin{equation*}
   \norm{\nabla f}^2 = \norm{\nabla_x \ksig}^2 = \frac{1}{\rho^4}\norm{x-y}^2\exp(-\norm{x-y}^2/\rho^2).
\end{equation*}
To bound the above quantity for all $x,y\in\mathbb{R}^d$,
we compute using elementary calculus that
\begin{equation*}
    \max_{r\geq 0} \frac{1}{\rho^4}r^2e^{-\frac{r^2}{\rho^2}} = \frac{1}{\rho^2 e}.
\end{equation*}
Therefore,
   $L\leq \sup\limits_{x\in \mathbb{R}^d} \norm{\nabla f(x)} =  \sup\limits_{x\in \mathbb{R}^d} \norm{\nabla_x \ksig} \leq \frac{1}{\rho\sqrt{e}}$,
which completes the proof.
\end{proof}

\subsection{Small Bandwidth Case}
\label{sub:Case I}

\cdf{In this section, we focus on the regime $\rho\to 0$ and derive estimates for the pointwise value}
$$\RSsig(x,y) := \ksig(x,y) - K_{xS}K_{SS}^{-1}K_{Sy}.$$
For a set $S$ with $r$ points, exact evaluation of $\RSsig(x,y)$ at each pair $(x,y)$ costs $O(r^3)$ for factorizing $K_{SS}$ and $O(r^2)$ for computing $K_{SS}^{-1}K_{Sy}$ or $K_{xS}K_{SS}^{-1}$ based on the computed factors. The estimates derived in this section offer an intuitive geometric characterization of $\abs{\RSsig(x,y)}$ and can be used to efficiently identify the locations in $\Omega\times\Omega$ where $\abs{\RSsig(x,y)}$ achieves smaller values (see Theorem \ref{thm:error1}) or larger values (see Theorem \ref{thm:lowerbound}) without the exact evaluation of $\RSsig(x,y)$.
Illustrations of the estimates compared to the true pattern of $\abs{\RSsig(x,y)}$ are shown in Section \ref{sub:plotbounds}.


\begin{theorem}
\label{thm:error1}
    Let $\ksig(x,y)=\exp(-\frac{\norm{x-y}^2}{2\rho^2})$ be the Gaussian kernel over $\mathbb{R}^d\times\mathbb{R}^d$.
For any subset $S=\{s_1,\dots,s_r\}\subseteq\mathbb{R}^d$ with $r\geq 1$, 
let $\RSsig(x,y)$ be the posterior covariance conditioned on $S$, as defined in \eqref{eq:RxyDef}.
{If $\dist{x,S}\geq \sqrt{2}\hat{\omega}\rho$ or $\dist{y,S}\geq \sqrt{2}\hat{\omega}\rho$} for some $\hat{\omega}>0$ \emph{AND} $\dist{x,y}\geq \sqrt{2}\omega\rho$ for some $\omega>0$,
then
\begin{equation}
\label{eq:RxyCase1far}
    \abs{\RSsig(x,y)}\leq e^{-\omega^2} + e^{-\hat{\omega}^2}\sqrt{r}\norm{K_{SS}^{-1}K_{Sy}}_2.
\end{equation}
Consequently, 
if we define the quantity
\begin{equation}
\label{eq:Gamma}
\Gamma_p := \max\limits_{y\in\mathbb{R}^d} \norm{K_{SS}^{-1}K_{Sy}}_p,
\end{equation}
where $\norm{\cdot}_p$ denotes the $p$-norm of a vector,
then
\begin{equation}
\label{eq:RxyCase1Gamma}
    \abs{\RSsig(x,y)}\leq e^{-\omega^2} + e^{-\hat{\omega}^2}\sqrt{r}\Gamma_2.
\end{equation}
\end{theorem}
\begin{proof}
We first prove the estimate in \eqref{eq:RxyCase1far} under the condition that $\dist{x,S}\geq \sqrt{2}\hat{\omega}\rho$ and $\dist{x,y}\geq \sqrt{2}\omega\rho$.
The case of $\dist{y,S}\geq \sqrt{2}\hat{\omega}\rho$ is similar.
Since $\dist{x,S}\geq \sqrt{2}\hat{\omega}\rho$, it follows immediately that, for any $s\in S$,
$e^{-\frac{\norm{x-s}^2}{2\rho^2}}\leq e^{-\hat{\omega}^2}$.
As a result, $\norm{K_{xS}}\leq \sqrt{r}e^{-\hat{\omega}^2}$.
Now we deduce that 
\begin{equation}
\label{eq:RxyCase1Naive}
\begin{aligned}
    \abs{\RSsig(x,y)} &= \abs{\ksig(x,y) - K_{xS}K_{SS}^{-1}K_{Sy}}\\
    &\leq \abs{\ksig(x,y)} + \abs{K_{xS}K_{SS}^{-1}K_{Sy}}\\
    &\leq e^{-\frac{\norm{x-y}^2}{2\rho^2}} + \norm{K_{xS}} \norm{K_{SS}^{-1}K_{Sy}}\\
     &\leq e^{-\omega^2} + \sqrt{r}e^{-\hat{\omega}^2}\norm{K_{SS}^{-1}K_{Sy}},
\end{aligned}
\end{equation}
which proves \eqref{eq:RxyCase1far}.

The inequality \eqref{eq:RxyCase1Gamma} follows immediately from \eqref{eq:RxyCase1far} since $\norm{K_{SS}^{-1}K_{Sy}}\leq \Gamma_2$.
However, we still need to prove that $\Gamma_p$ is well-defined.
It suffices to show that the maximum can be achieved in a closed ball.
Note that 
\begin{equation}
\label{eq:GammaS}
\Gamma_p\geq \norm{K_{SS}^{-1}K_{Sy}}_p = 1,\quad \text{if}\; y\in S.
\end{equation}
On the other hand, 
it is easy to see that $\norm{K_{Sy}}_p\to 0$ as $\norm{y}_p\to\infty$ since for any $s\in S$, the corresponding entry in $K_{Sy}$, $\exp(-\frac{\norm{s-y}^2}{2\rho^2})\to 0$ as $\norm{y}_p\to\infty$.
Thus we can choose a closed ball $B$ centered at the origin with a sufficiently large radius such that $S\subseteq B$ and
\begin{equation*}
\sup\limits_{y\notin B}\norm{K_{SS}^{-1}K_{Sy}}_p < 0.5.  
\end{equation*}
Note that \eqref{eq:GammaS} implies
$$
\max\limits_{y\in B}\norm{K_{SS}^{-1}K_{Sy}}_p \geq 1.
$$
It follows that $\sup\limits_{y\in\mathbb{R}^d}\norm{K_{SS}^{-1}K_{Sy}}_p$ must be achieved at some $y\in B$.
Therefore,
$$\Gamma_p = \max\limits_{y\in B}\norm{K_{SS}^{-1}K_{Sy}}_p\ < \infty.$$
The proof of the theorem is now complete.
\end{proof}

Next, we show the limits of $\RSsig(x,y)$ defined in \eqref{eq:RxyCase1far} and $\Gamma_p$ defined in \eqref{eq:Gamma} as the bandwidth $\rho$ goes to $0$.

\begin{proposition}
\label{prop:limitCase1}
Under the assumption in Theorem \ref{thm:error1} about $\ksig$, $S$, $x$, $y$, we have
$$\RSsig(x,y)\to 0\quad\text{and}\quad \Gamma_p\to 1\quad\text{as}\quad \rho\to 0.$$
\end{proposition}
\begin{proof}
We first show that $\Gamma_p\to 1$.
Note that $\norm{K_{SS}^{-1}K_{Sy}}_p=1$ whenever $y\in S$.
If $y\notin S$,
as $\rho\to 0$,
$\norm{K_{SS}^{-1}K_{Sy}}_p\to 0$
since $K_{SS}$ approaches the identity matrix and $K_{Sy}$ approaches the zero matrix.
Thus we have
$\Gamma_p = \max\limits_{y\in\mathbb{R}^d} \norm{K_{SS}^{-1}K_{Sy}}_p \to 1\;\,\text{as}\;\, \rho\to 0.$

To show $\RSsig(x,y)\to 0$, we use \eqref{eq:RxyCase1Gamma}:
$$\abs{\RSsig(x,y)}\leq e^{-\omega^2} + \sqrt{r}e^{-\hat{\omega}^2}\Gamma_2.
$$
Note that both $\omega$ and $\hat{\omega}$ can be chosen to be arbitrarily large in the limit $\rho\to 0$, for example $\omega=\hat{\omega}=\rho^{-0.5}$. This is because $x,y$ are fixed and $\dist{x,y}\geq \sqrt{2}\omega\rho$, $\dist{x,S}\geq \sqrt{2}\hat{\omega}\rho$ always hold as $\rho\to 0$, where the left-hand sides are positive (independent of $\rho$) and the right-hand sides approach $0$ in the limit.
This completes the proof.
\end{proof}

Theorem \ref{thm:error1} and Theorem \ref{thm:vanish} help identify where the \emph{small} values of $\abs{\RSsig(x,y)}$ occur in $\Omega\times\Omega$.
They indicate that,
if $\dist{x,S}/\rho$ or $\dist{y,S}/\rho$ is large,
then $\RSsig(x,y)$ will be insignificant for all $x,y$ such that $\norm{x-y}/\rho$ is large.
Note that Theorem \ref{thm:error1} assumes that $x$ or $y$ is far from $S$, and does not consider the case when $\dist{x,S}/\rho$ and $\dist{y,S}/\rho$ are small.
This case ($x,y$ close to $S$) is in fact discussed after Theorem \ref{thm:vanish} and we have
$\RSsig(x,y)\approx 0$ because $\RSsig(x,S)=\RSsig(S,y)=0$ and $\RSsig(x,y)\in C(\mathbb{R}^d\times \mathbb{R}^d)$.
Hence we see that:
\vspace{2mm}
\begin{center}
    \emph{For small bandwidth $\rho$, $\abs{\RSsig(x,y)}$ will be small as long as $\norm{x-y}/\rho$ is large.}
\end{center}
\vspace{2mm}
This is consistent with Figure \ref{fig:prelim1} (see also Figure \ref{fig:case1small}), where $\Omega=[0,1]$, $S$ is uniformly distributed in $\Omega$, $\rho=0.1$ is considered substantially smaller than the data spacing $0.24$ in $S$. In Figure \ref{fig:case1small}, the two blue triangles enclose the points $(x,y)$ where $\norm{x-y}$ is larger than $0.24$. It is easy to see that $\RSsig(x,y)$ is close to zero when $(x,y)$ lies within these two regions.
The larger $\norm{x-y}$ is, the smaller $\abs{\RSsig(x,y)}$ will be.

\begin{figure}[htbp] 
    \centering 
    \includegraphics[scale=.4]{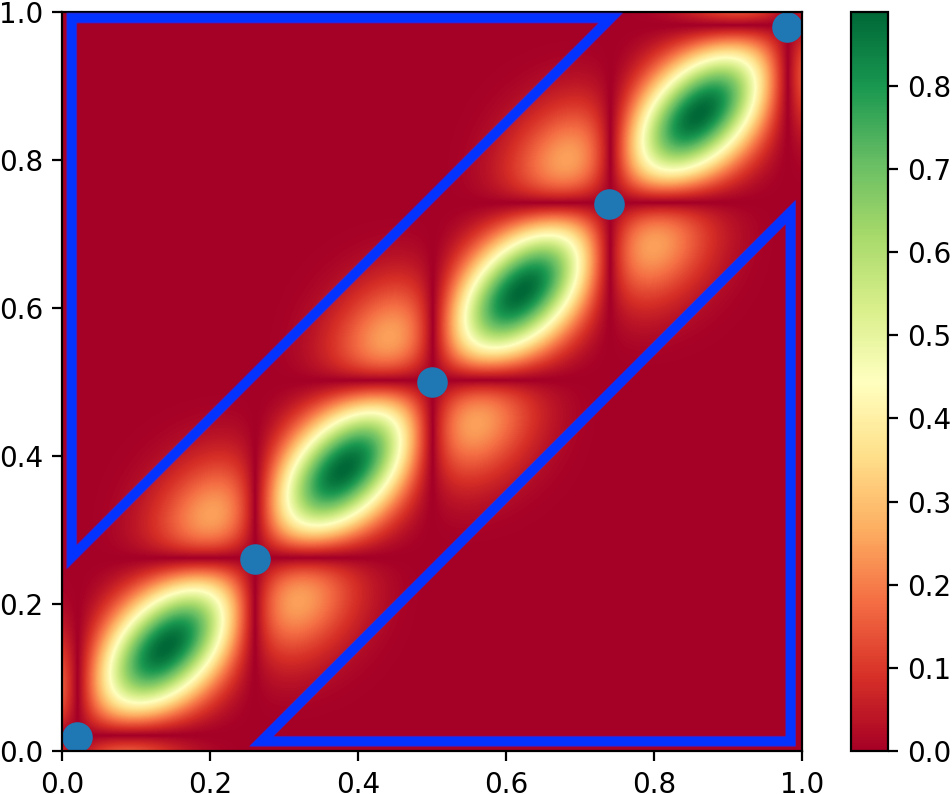} 
    \caption{$|\RSsig(x,y)|$ over $[0,1]\times [0,1]$ with $\rho=0.1$. Blue triangles enclose locations with $\lVert x-y\rVert/\rho\geq 1$.}
    \label{fig:case1small} 
\end{figure}

The next theorem helps identify the locations $(x,y)$ that yield \emph{large} values of $\abs{\RSsig(x,y)}$.
\begin{theorem}
\label{thm:lowerbound}
Let $\ksig(x,y)$, $S$, $r$, $\RSsig(x,y)$, $\Gamma_p$ be defined as in Theorem \ref{thm:error1}.
{If $\dist{x,S}\geq \sqrt{2}\hat{\omega}\rho$ or $\dist{y,S}\geq \sqrt{2}\hat{\omega}\rho$} for some $\hat{\omega}>0$, then 
\begin{equation}
\label{eq:RxyCase1near}
    \abs{\RSsig(x,y)}\geq \ksig(x,y) - \sqrt{r}e^{-\hat{\omega}^2}\norm{K_{SS}^{-1}K_{Sy}}.
\end{equation}
Furthermore, if there exists $c\geq 0$ such that $\dist{x,y}\leq \rho \sqrt{2\ln(1+c)}$, then
\begin{equation}
\label{eq:RxyLower}
    \abs{\RSsig(x,y)}\geq \frac{1}{1+c} - \sqrt{r}e^{-\hat{\omega}^2}\Gamma_2.
\end{equation}
\end{theorem}
\begin{proof}
Without loss of generality, we assume that 
$\dist{x,S}\geq \sqrt{2}\hat{\omega}\rho$ and estimate the residual
$$\RSsig(x,y) = \ksig(x,y) - K_{xS}K_{SS}^{-1}K_{Sy}.$$
According to \eqref{eq:RxyCase1Naive}, we know that 
$$\abs{K_{xS}K_{SS}^{-1}K_{Sy}} \leq \sqrt{r}e^{-\hat{\omega}^2} \norm{K_{SS}^{-1}K_{Sy}}.$$
Therefore, triangle inequality implies that 
$$\abs{\RSsig(x,y)}\geq \ksig(x,y)-\abs{K_{xS}K_{SS}^{-1}K_{Sy}} \geq \ksig(x,y)-\sqrt{r}e^{-\hat{\omega}^2} \norm{K_{SS}^{-1}K_{Sy}},$$
which proves \eqref{eq:RxyCase1near}.

The proof of \eqref{eq:RxyLower} is straightforward from \eqref{eq:RxyCase1near}.
The condition $\dist{x,y}\leq \rho \sqrt{2\ln(1+c)}$
implies that 
$$\ksig(x,y)=\exp(-\frac{\norm{x-y}^2}{2\rho^2})\geq \exp(-\ln(1+c))=\frac{1}{1+c}.$$
This completes the proof together with \eqref{eq:RxyCase1near} and the fact that $\norm{K_{SS}^{-1}K_{Sy}}\leq \Gamma_2$.
\end{proof}

Theorem \ref{thm:lowerbound} discusses conditions and estimates for \emph{large} values of $\abs{\RSsig(x,y)}$.
\cdf{It states that, when $x,y$ are close and either $\dist{x,S}/\rho$ or $\dist{y,S}/\rho$ is sufficiently large,
then $\abs{\RSsig(x,y)}$ will be away from $0$.}
In fact, this can be seen from the formula
$\RSsig(x,y) = \ksig(x,y) - K_{xS}K_{SS}^{-1}K_{Sy}$.
When $x,y$ are close, $\ksig(x,y)\approx 1$;
When $\dist{x,S}/\rho$ is large, $K_{xS}\approx 0$.
This leads to $\RSsig(x,y)\approx 1$.
Theorem \ref{thm:lowerbound} can help understand the preliminary observation in Figure \ref{fig:prelim1}.
See also Figure \ref{fig:case1} for convenience.
Recall the setting for $\RSsig(x,y)$ over $\Omega\times\Omega$:
$\Omega=[0,1]$ and evenly spaced observations
$S=\{0.02,0.26,0.5,0.74,0.98\},$ with spacing $0.24$.
The bandwidth is $\rho=0.1$, which is \emph{small} compared to 0.24, the data spacing in $S$.
The four green regions in Figure \ref{fig:case1} represent the largest values in $\abs{\RSsig(x,y)}$ over the domain.
The locations $(x,y)$ of these dominant regions imply that:
\vspace{2mm}
\begin{center}
\emph{For small bandwidth $\rho$, $\abs{\RSsig(x,y)}$ will be large when $\norm{x-y}/\rho$ is small \emph{and} $x$, $y$ are not close to $S$.}
\end{center}
\vspace{2mm}
This is consistent with Theorem \ref{thm:lowerbound}.
In fact,
Theorem \ref{thm:lowerbound} indicates that,
small $\norm{x-y}/\rho$ and large $\dist{x,S}/\rho$ allows
a small $c\approx 0$ and a large $\hat{\omega}$, which will result in $|\RSsig(x,y)|=O(1)$ according to the estimate in \eqref{eq:RxyLower}.
{ 
In the limit $\rho\to 0$, the condition in Theorem \ref{thm:lowerbound} on $\dist{x,S}/\rho$ will be satisfied by almost all $x\in\Omega$ with a large $\hat{\omega}$ since $\dist{x,S}/\rho\to\infty$.
The estimate \eqref{eq:RxyLower} implies that for almost all $x\in \Omega$,
$\abs{\RSsig(x,y)}=O(1)$ in the small band where $\norm{y-x}=O(\rho)$.
Figure \ref{fig:case1smallsigma} illustrates $\abs{\RSsig(x,y)}$ for the case of small $\rho$: $\rho=0.01$ and $\rho=0.05$.
We see that the largest entries concentrate near the diagonal and the diagonal bandwidth is approximately $\norm{x-y}=O(\rho)$.
The pattern for the limit case $\rho\to 0$ is nicely illustrated by the left plot in Figure \ref{fig:case1smallsigma}.
}

\begin{figure}[htbp]
\centering
\includegraphics[scale=.4]{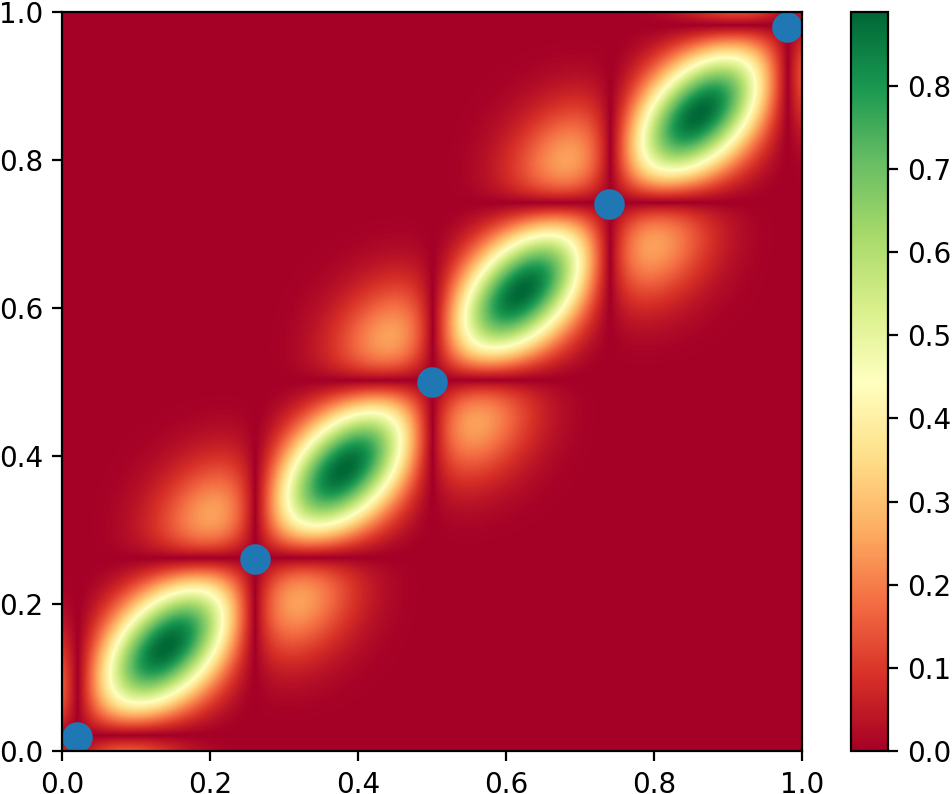}
\caption{$|\RSsig(x,y)|$ over $[0,1]\times [0,1]$ with $\rho=0.1$.}
\label{fig:case1}
\end{figure}

\begin{figure}[htbp]
\centering
\includegraphics[scale=.4]{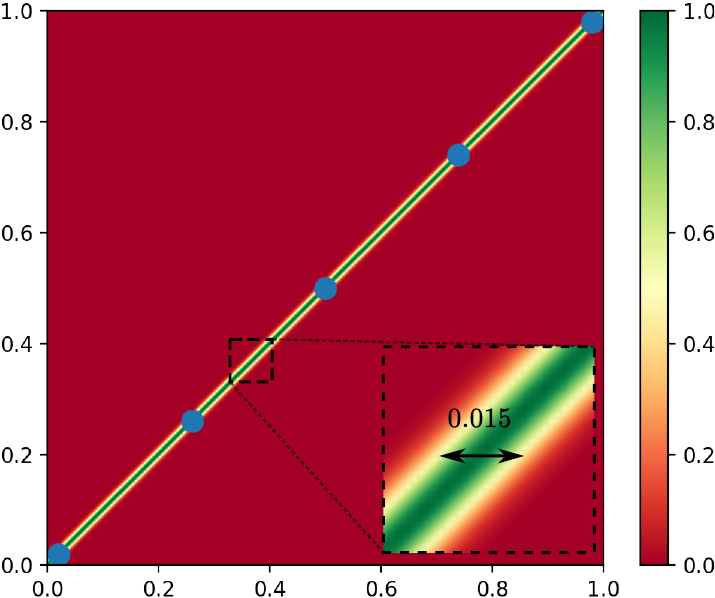}
\hspace{0.2in}
\includegraphics[scale=.4]{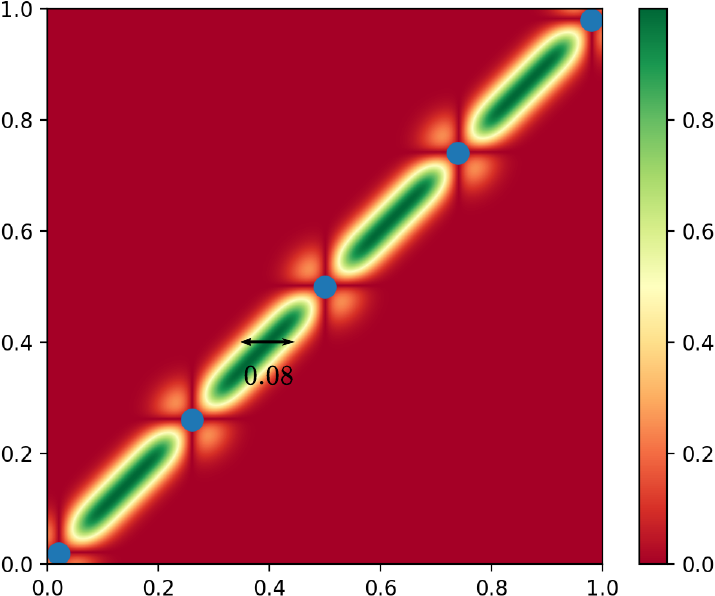}
\caption{$|\RSsig(x,y)|$ with $\rho=0.01$ (left), $\rho=0.05$ (right).}
\label{fig:case1smallsigma}
\end{figure}

A special case of Theorem \ref{thm:lowerbound} is when $x=y$.
We single out the special case as a corollary below, since it is useful to \cdf{characterize the behavior of the posterior variance $\Var(x)=\RSsig(x,x)$ in \eqref{eq:Var}.}
Corollary \ref{cor:variance} implies that if $x$ is relatively far from $S$ (in terms of $\rho$), then the posterior variance at $x$ will be large.
\begin{corollary}
\label{cor:variance}
Let $\ksig(x,y)$, $S$, $r$, $\RSsig(x,y)$, $\Gamma_p$ be defined as in Theorem \ref{thm:error1}.   
If $\dist{x,S}\geq \sqrt{2}\hat{\omega}\rho$ for some $\hat{\omega}>\left(\ln[\sqrt{r}\Gamma_2]\right)^{1/2}$,
then 
\begin{equation*}
    \Var(x)\geq 1-e^{-\hat{\omega}^2}\sqrt{r}\Gamma_2 > 0.
\end{equation*}
\end{corollary}
\begin{proof}
    This is a direct result of Theorem \ref{thm:lowerbound} with $c=0$.
\end{proof}

\subsection{Large Bandwidth Case}
\label{sub:Case II}

\cdf{In this section, we focus on the case $\rho\to\infty$.}
Our analysis shows that, different from the \emph{small} bandwidth case, 
$\abs{\RSsig(x,y)}$ mainly depends on $\dist{x,S}$ and $\dist{y,S}$, as opposed to $\norm{x-y}$.  
Since the kernel $\ksig(x,y)$ associated with large $\rho$ is smoother than that in the \emph{small} bandwidth case,  $\abs{\RSsig(x,y)}$ will generally be closer to zero. As a result, the analysis in this section only aims to identify where the maxima of $\abs{\RSsig(x,y)}$ are likely to occur, and \cdf{``large'' is used in a relative sense, to refer to larger values of $\abs{\RSsig(x,y)}$ over the entire domain}.

\begin{theorem}
\label{thm:error2}
    Let $\ksig(x,y)=\exp(-\frac{\norm{x-y}^2}{2\rho^2})$ be the Gaussian kernel.
Define $\RSsig(x,y)$ as in Theorem \ref{thm:error1}.
For any non-empty subset $S\subseteq \mathbb{R}^d$ with $r$ points, the following estimate holds
\begin{equation}
\label{eq:kxyerror}
\abs{\RSsig(x,y)} \leq
\min\left[ (1+\sqrt{r}\norm{K_{SS}^{-1}K_{Sy}})\frac{\dist{x,S}}{\rho\sqrt{e}},\quad (1+\sqrt{r}\norm{K_{SS}^{-1}K_{Sx}})\frac{\dist{y,S}}{\rho\sqrt{e}} \right].
\end{equation}
\end{theorem}
\begin{proof}
%
%


    We first show that $\abs{\RSsig(x,y)}$ is no larger than the first entry in the minimum in \eqref{eq:kxyerror}.
    Let $s_x$ denote a point in $S$ such that 
    $\norm{x-s_x}=\dist{x,S}$.
    Since $\RSsig(s,y)=\RSsig(x,s)=0$ for any $s\in S$,
    we deduce from the Lipschitz continuity of $\RSsig(x,y)$ in $x$ that 
\begin{equation}
\label{eq:case2err1}
    \abs{\RSsig(x,y)} = \abs{\RSsig(x,y)-\RSsig(s_x,y)}
    \leq \sup_{t\in\mathbb{R}^d} \norm{\nabla_t \RSsig(t,y)} \norm{x-s_x} = \sup_{t\in\mathbb{R}^d} \norm{\nabla_t \RSsig(t,y)} \dist{x,S}.
\end{equation}
    Recall the definition $\RSsig(x,y) = \ksig(x,y) - K_{xS}K_{SS}^{-1}K_{Sy}$.
    To estimate $\sup\limits_{t\in\mathbb{R}^d} \norm{\nabla_t \RSsig(t,y)}$, by using the gradient estimate in \eqref{eq:Lip} and H{\"o}lder's inequality, we have
    \begin{align*}
        \sup_{t\in\mathbb{R}^d} \norm{\nabla_t \RSsig(t,y)} &\leq 
        \sup_{t\in\mathbb{R}^d} \norm{\nabla_t \ksig(t,y)} + \sup_{t\in\mathbb{R}^d} \norm{\nabla_t K_{tS}}\norm{K_{SS}^{-1}K_{Sy}}\\
        &\leq \frac{1}{\rho\sqrt{e}} + \frac{\sqrt{r}}{\rho\sqrt{e}}\norm{K_{SS}^{-1}K_{Sy}}.
    \end{align*}
    Therefore, we see that 
    $$ \abs{\RSsig(x,y)}\leq (1+\sqrt{r}\norm{K_{SS}^{-1}K_{Sy}})\frac{\dist{x,S}}{\rho\sqrt{e}}.$$
    Similarly, we can show that
    $$\abs{\RSsig(x,y)}\leq (1+\sqrt{r}\norm{K_{SS}^{-1}K_{Sx}})\frac{\dist{y,S}}{\rho\sqrt{e}}$$
    by viewing $\RSsig(x,y)$ as a function of $y$ and using the fact that $$\norm{K_{xS}K_{SS}^{-1}} = \norm{(K_{xS}K_{SS}^{-1})^T} = 
    \norm{K_{SS}^{-T}K_{xS}^T} = \norm{K_{SS}^{-1}K_{Sx}}.$$
    Taking the minimum of the two upper bounds yields \eqref{eq:kxyerror}.
\end{proof}

It should be noted that, though there is no restriction on $\rho$ in Theorem \ref{thm:error2}, the estimate \eqref{eq:kxyerror} becomes more meaningful when $\rho$ is relatively large compared to $\dist{x,S}$ and $\dist{y,S}$. 
This can be seen from the limit case: \cdf{as $\rho\to 0$, the bound blows up but $\RSsig(x,y)$ remains well-bounded. Hence the bound is not consistent with $\RSsig(x,y)$.}
On the other hand, as $\rho\to\infty$, the bound approaches zero, consistent with $\RSsig(x,y)$.
This indicates that the estimate is more suitable for the \emph{large} bandwidth case.
Numerically, as shown later in the plots in Section \ref{sub:plotbounds}, the estimate helps to capture the behavior of $\abs{\RSsig(x,y)}$ in case of large $\rho$ but completely misses when $\rho$ is small.
Additionally, unlike the ``small bandwidth case" discussed in Section \ref{sub:Case I}, the distance $\dist{x,y}$ does not appear in the estimate of $\RSsig(x,y)$ or in the assumptions.
The estimate implies that the quantity $\abs{\RSsig(x,y)}$ is mainly affected by the distance of $x,y$ to $S$ in the large bandwidth case.

Now let us demonstrate Theorem \ref{thm:error2} by reconsidering the example from Section \ref{sub:Case I} with $\rho=0.4$ 
\cdf{(which represents the \emph{large} bandwidth regime)}.
Figure \ref{fig:case2} shows the heat map of $\abs{\RSsig(x,y)}$ when $\rho=0.4$.
First of all, we can see a clear  ``tensor product'' structure (with 4-by-4 blocks) and 
different from Figure \ref{fig:case1}, there is no special structure near the diagonal, which indicates that $\dist{x,y}$ has little to no influence on $\abs{\RSsig(x,y)}$.
The bright (yellow or green) regions in the 4-by-4 grid correspond to larger values of $\abs{\RSsig(x,y)}$ and imply that:
\vspace{2mm}
\begin{center}
\emph{For large bandwidth $\rho$, $\abs{\RSsig(x,y)}$ will be larger when $\dist{x,S}/{\rho}$ and $\dist{y,S}/{\rho}$ are larger.}
\end{center}
\vspace{2mm}
Such a geometric characterization is indeed consistent with the theoretical result in Theorem \ref{thm:error2}.
In fact, the upper bound estimate of $|\RSsig(x,y)|$ in Theorem \ref{thm:error2} increases with $\dist{x,S}$ and $\dist{y,S}$.
This indicates that $|\RSsig(x,y)|$ is likely to increase as  $\dist{x,S}$ and $\dist{y,S}$ increase.
Meanwhile, the upper bound in Theorem \ref{thm:error2} vanishes when $x\in S$ or $y\in S$,
which is consistent with the interpolation property in Theorem \ref{thm:vanish} that $\RSsig(s,y)=\RSsig(x,s)=0$ whenever $s\in S$.
In Figure \ref{fig:case2}, this property corresponds to the axis-parallel dark red grid lines, located at $x=0.02, 0.26, \dots$, and $y=0.02, 0.26, \dots$.
In addition to the ``tensor product'' structure, another noticeable observation is that the two green regions, regions with the most dominant values, are near the corners of the plot.
The two regions are where $x$ and $y$ are near the boundary of the domain $[0,1]$, i.e. $x$ and $y$ are close to $0$ or $1$.
We call this ``the boundary effect'' and provide a tentative explanation below.

\begin{figure}[htbp] 
\centering 
\includegraphics[scale=.35]{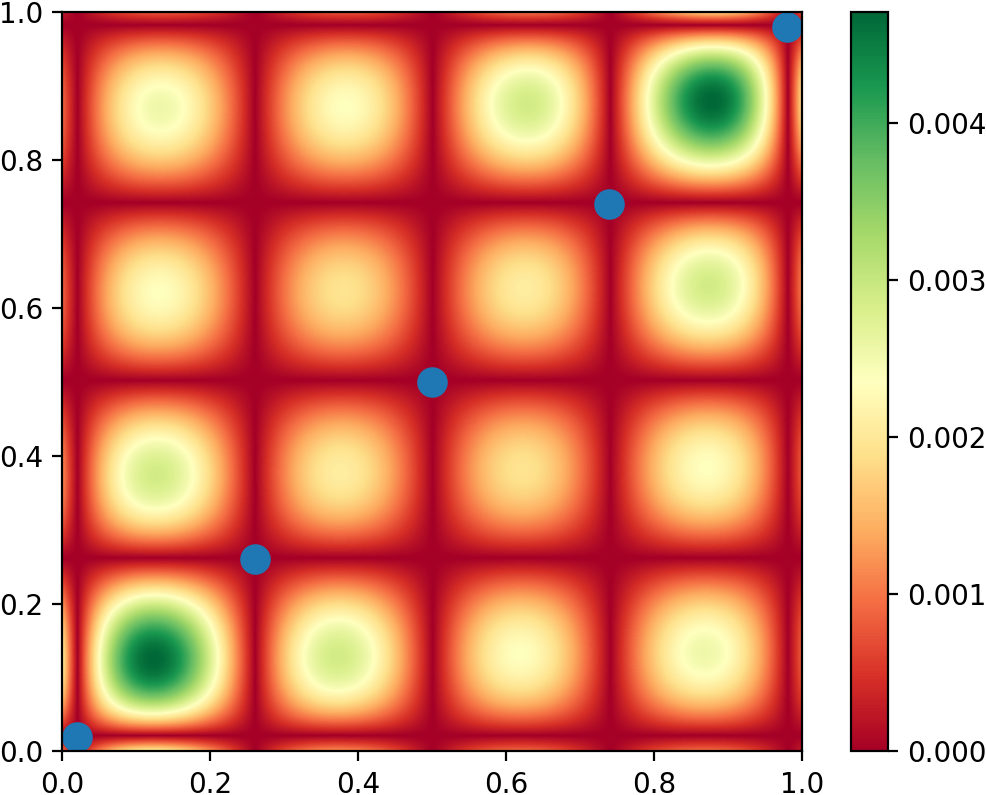}
\caption{$|\RSsig(x,y)|$ over $[0,1]\times [0,1]$ with $\rho=0.4$.}
\label{fig:case2}
\end{figure}

\paragraph{The ``boundary effect''}
In addition to the distance 
$\dist{x,S}:=\inf\limits_{s\in S}\norm{x-s}$,
it is possible to use other metrics.
For example, inspired by the 2-norm, we can use 
\begin{equation}
\label{eq:2metrics}
\sum_{i=1}^r \norm{x-s_i}^2\quad\text{or a weighted version}\;\, \sum_{i=1}^r \alpha_i\norm{x-s_i},\quad\text{where}\;\, S=\{s_i\}_{i=1}^r,
\end{equation}
as a metric to evaluate how far $x$ is to $S$.
Here $\alpha_i\geq 0$ and $\sum\limits_{i=1}^r \alpha_i=1$.

The metrics in \eqref{eq:2metrics} exhibit a ``cumulative'' effect, unlike \(\dist{x,S}\), which only considers the distance to the closest point in \(S\). To illustrate, consider two points \(x\) and \(x'\) with \(\dist{x,S} = \dist{x',S}\). If most points in \(S\) are far from \(x\) but only a few points in \(S\) are far from \(x'\), \(x\) will be farther from \(S\) compared to \(x'\) in the metrics defined in \eqref{eq:2metrics}. In Figure \ref{fig:case2}, points near the lower left corner and upper right corner have large distances to $S$ according to these metrics.

The weighted metric in \eqref{eq:2metrics} can be used to derive a new estimate which generalizes \eqref{eq:case2err1} in the following way: 
\begin{equation}
\label{eq:case2errWeights}
    \begin{aligned}
    \abs{\RSsig(x,y)} = \abs{\RSsig(x,y) - \sum_{i=1}^r \alpha_i \RSsig(s_i,y)}
    &= \abs{\sum_{i=1}^r \alpha_i \left(\RSsig(x,y) - \RSsig(s_i,y)\right)} \\
    &\leq \sum_{i=1}^r \alpha_i \norm{x - s_i} \sup_{t \in Q_i} \norm{\nabla_t \RSsig(t,y)},
    \end{aligned}    
\end{equation}
where $Q_i$ denotes an open cover that contains the line from \(x\) to \(s_i\). 

If $s_m \in S$ is a point in $S$ closest to $x$, by setting $\alpha_m = 1$ and all other weights $\alpha_i = 0$ ($i\neq m$), only one term remains in the sum in \eqref{eq:case2errWeights}, resulting in a bound that matches \eqref{eq:case2err1} if we choose $Q_m = \mathbb{R}^d$.

The above new estimate of $\abs{\RSsig(x,y)}$ offers a different perspective from \eqref{eq:case2err1} by considering the influence of all points in $S$. This estimate is particularly relevant in the \emph{large} bandwidth case as $\rho\to\infty$. In this scenario, the kernel $\ksig(x,y)$ and $\RSsig(x,y)$ both become more ``flat''. As a result, with \(\RSsig\) exhibiting small variation, its gradient will be small, making the Lipschitz estimate of \(\left(\RSsig(x,y) - \RSsig(s_i,y)\right)\) more accurate. 


\subsection{Illustration of the bounds in different cases}
\label{sub:plotbounds}
In this section, we show that the geometric upper bound estimates in \eqref{eq:RxyCase1far} and \eqref{eq:kxyerror} 
and lower bound estimate in \eqref{eq:RxyCase1near} can accurately capture the pattern of $\vert \RSsig(x,y)\vert$. Moreover, we demonstrate why it is necessary to divide the discussion into two different cases and how much distinctive the patterns from the two cases are.
For example, we clarify why the estimate derived from \eqref{eq:RxyCase1far} in Theorem \ref{thm:error1}  is valid in the context of ``small bandwidth and large $\dist{x,y}$" but becomes invalid in other situations (as detailed in Theorem \ref{thm:lowerbound} and Theorem \ref{thm:error2}). We also illustrate why the estimate derived from \eqref{eq:kxyerror}  in Theorem \ref{thm:error2}  is not suitable for cases involving small bandwidth.

We consider the same problem setup as in Section \ref{sub:Case I} for Figure \ref{fig:case1small}.
That is, $\Omega=[0,1]$ and $S=\{0.02,0.26,0.5,0.74,0.98\}$.
To visualize the pattern of $\abs{\RSsig(x,y)}$, 
we fix $y_*=0.15$ away from $S$ and plot $\abs{\RSsig(x,y_*)}$ for $x$ under three different scenarios, which reflect the assumptions in Theorem \ref{thm:error1}, Theorem \ref{thm:lowerbound} and Theorem \ref{thm:error2}, respectively:\vspace{1mm}
\cdf{
\begin{itemize}
    \item Condition 1 (Theorem \ref{thm:error1}): $\rho=0.05$ and $x\in (0.3,1)$ (``small $\rho$, large $\dist{x,y_*}$'');\vspace{1mm}
    \item Condition 2 (Theorem \ref{thm:lowerbound}): $\rho=0.05$ and $x\in(0,0.3)$ (``small $\rho$, small $\dist{x,y_*}$'');\vspace{1mm}
    \item Condition 3 (Theorem \ref{thm:error2}): $\rho=0.4$ and $x\in\Omega$ (``large $\rho$'').\vspace{1mm}
\end{itemize}
}
Note that in the case of small $\rho$, $\dist{y_*,S}$ is considered noticeably large and thus fulfills the condition $\dist{y,S}\geq \sqrt{2}\hat{\omega}\rho$ with a decent $\hat{\omega}$ (away from $0$) in Theorem \ref{thm:error1} and Theorem \ref{thm:lowerbound}.

For each scenario above, we plot in Figure \ref{fig:bounds} the true $\abs{\RSsig(x,y_*)}$, and the following estimates. 
As mentioned above, the goal is to show that $\RSsig$ behaves quite differently in different scenarios, and how the estimates indicate the different behaviors. Here, the estimates $\eta_1$ and $\eta_2$ for the small bandwidth case are obtained by taking ${\omega}=\frac{\dist{y_*,x}}{\sqrt{2}\rho}$ in \eqref{eq:RxyCase1far} and $\hat{\omega}=\frac{\dist{y_*,S}}{\sqrt{2}\rho}$ in \eqref{eq:RxyCase1near}. The estimate $\eta_3$ for the large bandwidth case is obtained by simply taking the distance metric $\dist{x,S}$ from the upper bound in Theorem \ref{thm:error2} in order to visualize the change in $x$. The factor in front of $\dist{x,S}$ in the bound is multiplicative and is independent of $x$.\vspace{1mm}
\begin{enumerate}
    \item Case I Estimate 1: 
    $\eta_1(x) = \ksig(x,y_*) + \sqrt{r}e^{-\hat{\omega}^2}\norm{K_{SS}^{-1}K_{Sy_*}}$;\vspace{1mm}
    \item Case I Estimate 2: 
    $\eta_2(x) = \ksig(x,y_*) - \sqrt{r}e^{-\hat{\omega}^2}\norm{K_{SS}^{-1}K_{Sy_*}}$;\vspace{1mm}
    \item Case II Estimate: $\eta_3(x) = \dist{x,S}$ from Theorem \ref{thm:error2}.\vspace{1mm}
\end{enumerate}
Since we are comparing the pattern of each estimate to $|\RSsig(x,y_*)|$, we re-scale each quantity such that the maximum of the estimate is equal to that of $|\RSsig(x,y_*)|$ over $x$ in each condition. For example, for the third estimate, we plot 
$$\frac{\eta_3(x)}{\max\limits_{x\in\Omega} \eta_3(x)}\times \max\limits_{x\in\Omega} |\RSsig(x,y_*)|.$$
For the first estimate, we plot 
$$\frac{\eta_1(x)}{\max\limits_{\dist{x,y_*}>3\rho} \eta_1(x)} \times \max_{\dist{x,y_*}>3\rho} |\RSsig(x,y_*)|.$$
For $\rho=0.05$ in the small bandwidth case, $\{x\in [0,1]: \dist{x,y_*}>3\rho\}=(0.3,1]$ in Condition 1
and $\{x\in [0,1]: \dist{x,y_*}<3\rho\}=[0,0.3)$ in Condition 2.
We emphasize that $\eta_2(x)$ (developed for the small bandwidth case) can be meaningless (i.e. $\eta_2(x)<0$) if $\rho$ is large, which already implies why $\eta_2$ is not suitable for Condition 3.
Nonetheless, we include the plot of $|\eta_2(x)|$ for Condition 3 in Figure \ref{fig:bounds}(right) for completeness.

The patterns of $|\RSsig(x,y_*)|$ and re-scaled estimates are shown in Figure \ref{fig:bounds}, where the three plots correspond to the three scenarios from Condition 1 to Condition 3, respectively.
It is easily seen that for Condition 1 (left plot), the estimate $\eta_1(x)$ captures the behavior of $\abs{\RSsig(x,y_*)}$ nicely, while $\eta_3(x)$, developed for the \emph{large} bandwidth case, completely misses the correct pattern.
For Condition 2 (middle plot), the behavior of $\eta_2$ is similar to $|\RSsig(x,y_*)|$, while again, $\eta_3$ does not capture the behavior correctly.
For the \emph{large} bandwidth case in Condition 3 (right plot), only $\eta_3$ reflects the behavior of $|\RSsig(x,y_*)|$, and $\eta_1$, $\eta_2$, which are derived for the \emph{small} bandwidth case, are not suitable here.
We see that even though Theorem \ref{thm:error2} does not impose any condition on $\rho$, it is suitable for the case of \emph{large} bandwidth only.
\cdf{It should be emphasized that the three quantities $\eta_1$, $\eta_2$, $\eta_3$ above are used to demonstrate the effectiveness of the estimates in different contexts. They are not necessarily computationally efficient or useful in practice. Practical indicators that are efficient to compute will be presented in Section \ref{sub:post}.}

\begin{figure}
    \centering
    \includegraphics[scale=0.33]{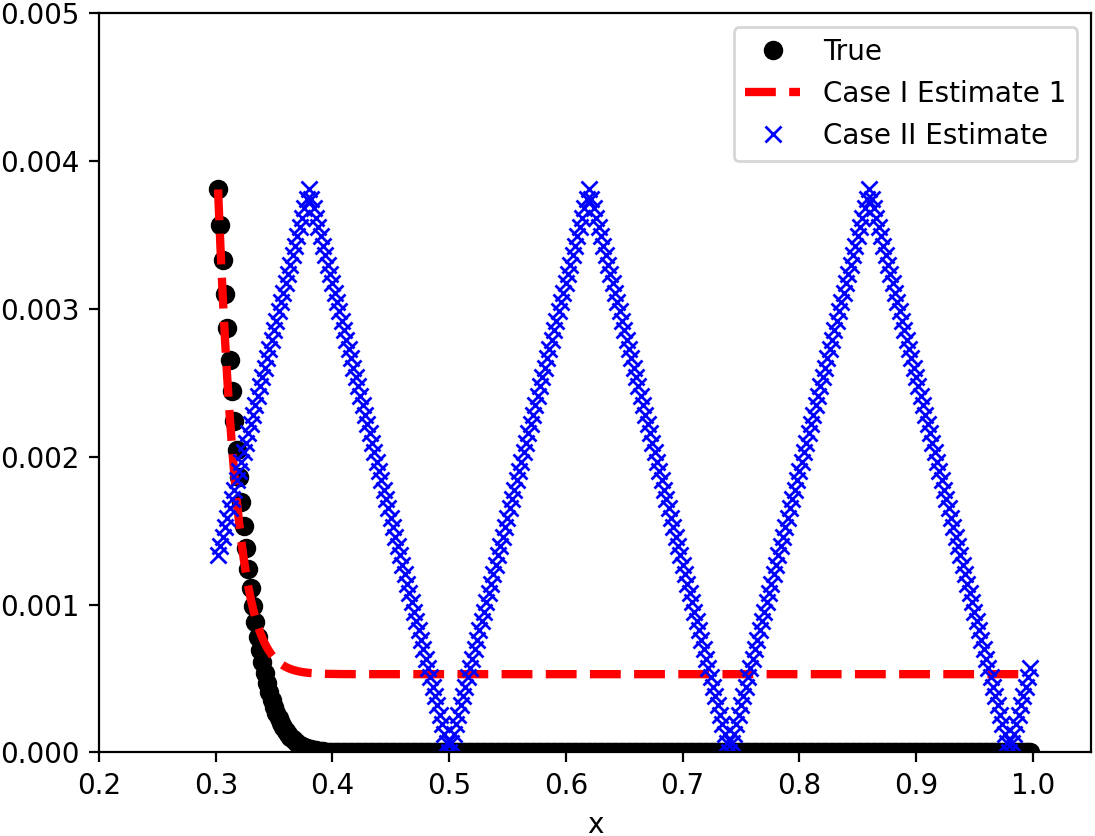}
    \includegraphics[scale=0.33]{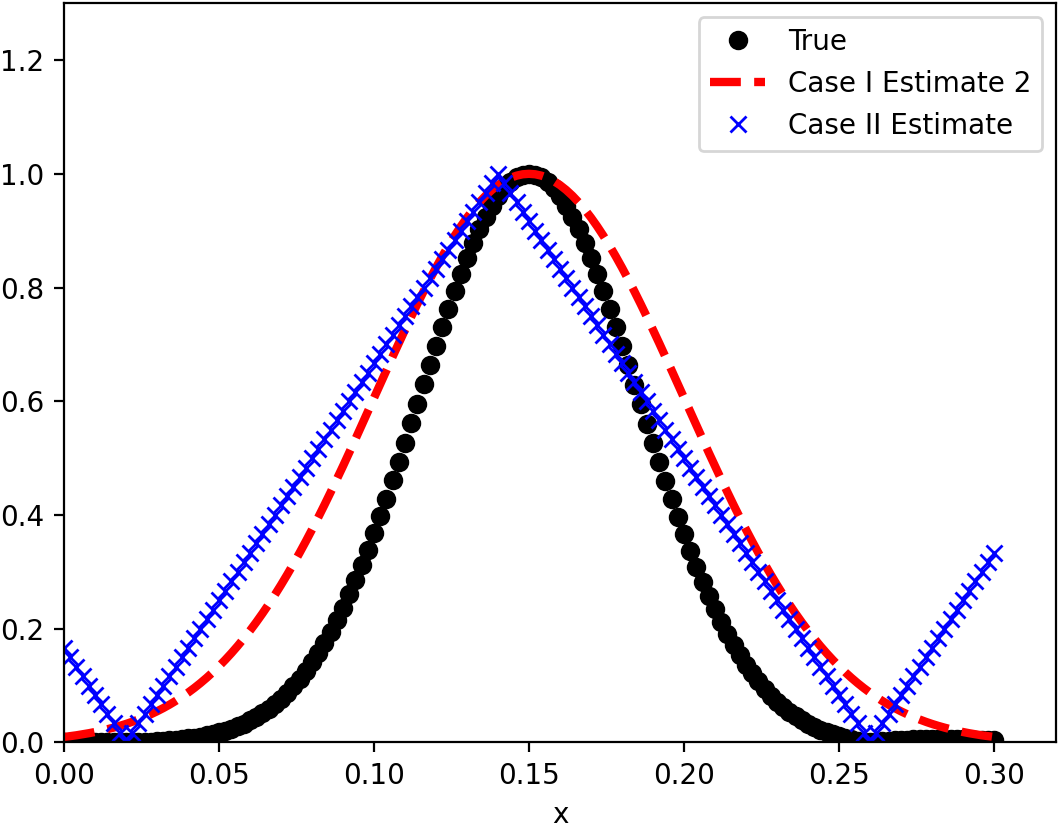}
    \includegraphics[scale=0.33]{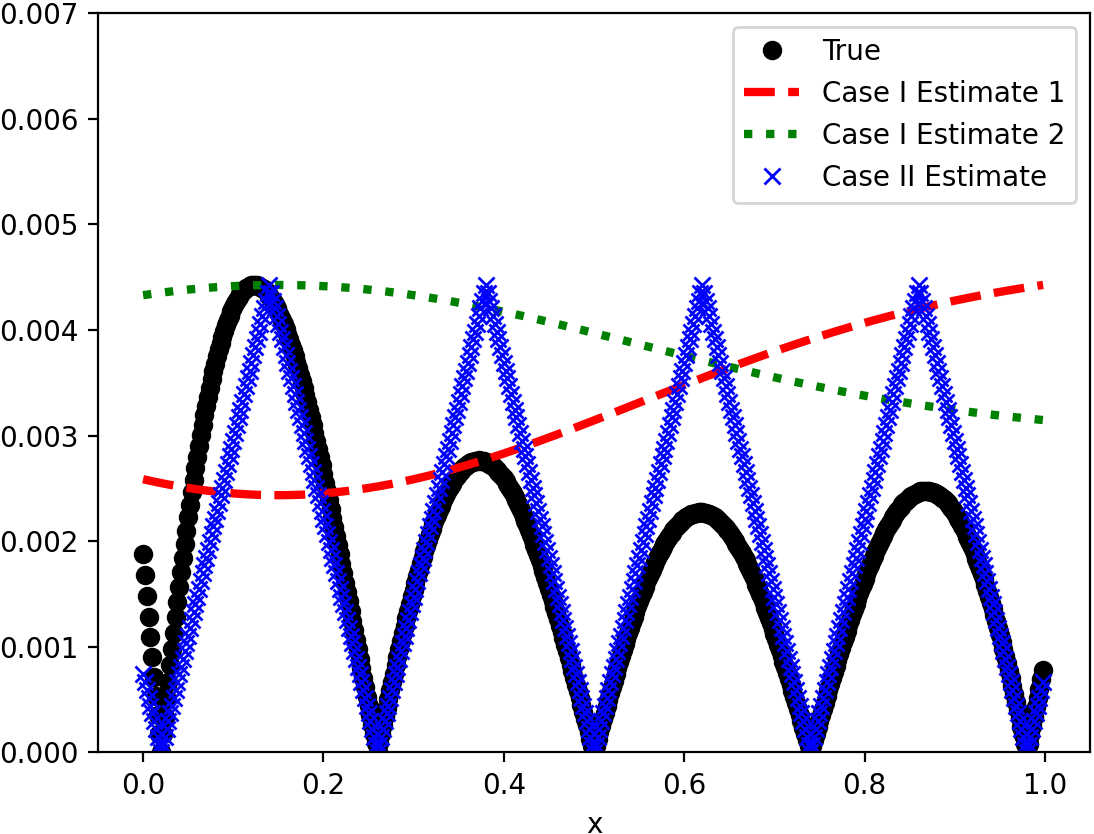}
    \caption{Illustration of $|\RSsig(x,y_*)|$(black dots) and re-scaled estimates in three scenarios (reflecting the conditions in Theorem \ref{thm:error1}, Theorem \ref{thm:lowerbound}, Theorem \ref{thm:error2}). Left: small $\rho$ and large $\dist{x,y_*}$; Middle: small $\rho$ and small $\dist{x,y_*}$; Right: large $\rho$.}
    \label{fig:bounds}
\end{figure}

\subsection{Noisy observation case: $\tau>0$}
In case of noisy observations, i.e. $\tau>0$ in the prior distribution \eqref{eq:GPprior},
the posterior distribution given the observations $(S,y)$ becomes
$$f_*|X_*,S,y \;\sim\;  \mathcal{N}\left(K_{X_* S}(K_{SS}+\tau^2 I)^{-1}y,\; \RSsig(X_*,X_*) \right),$$
where the posterior covariance is
$$\RSsig(u,v):=\ksig(u,v)-K_{uS}(K_{SS}+\tau^2 I)^{-1}K_{Sv}.$$
To illustrate the impact of $\tau$ on $\RSsig$, we show in Figure \ref{fig:tau} the plots of $\abs{\RSsig\cdcd}$ in several cases with $\tau=0,0.01$ and different $\rho$ values.
It can be seen from Figure \ref{fig:tau} that the impact of $\tau$ is not noticeable except when $\rho$ is sufficiently large ($\rho=0.6$ in the right-most column). This is because when $\rho$ is sufficiently large, the prior covariance $\ksig(x,y)$ is so smooth that the matrix $K_{XX}$ has rapidly decaying singular values, thus quite small numerical rank. 
Figure \ref{fig:tauSVD} illustrates the singular value decay of $K_{XX}$ for two cases: $\rho=0.6$ and $\rho=0.1$,
where $X$ contains 500 equispaced points in $[0,1]$.
It is easy to see that $\rho=0.6$ yields a much faster singular value decay than $\rho=0.1$.
Therefore, in the case of large bandwidth, if $\tau>0$ is not close to zero, the matrices $K_{SS}$ and $K_{SS}+\tau^2 I$ involved in $\RSsig$ have quite different spectral structures:
$K_{SS}$ is nearly singular with most singular values close to zero while $K_{SS}+\tau^2 I$ is much better conditioned with all singular values above $\tau^2$. 
This makes a substantial difference in the matrix inverse $K_{SS}^{-1}$ (noise-free case) or $(K_{SS}+\tau^2 I)^{-1}$ (noisy case) in the definition of $\RSsig$.

On the other hand, when $\rho$ is small, small $\tau$ has little influence on the structure of $\ksig(x,y)$. This is due to the fact that the inverse of $K_{SS}$ is much less sensitive to small perturbations such as $K_{SS} + \tau^2 I$ compared to when $\rho$ is large. Specifically, for a small $\tau > 0$, $K_{SS}^{-1}$ and $(K_{SS} + \tau^2 I)^{-1}$ will differ significantly if $\rho$ is large, but will be close if $\rho$ is small. 
Consequently, $\RSsig$ will not change much from $\tau=0$ to $\tau>0$ in the case of small $\rho$, while it will change a lot if $\rho$ is quite large.
This explains why the two plots in each column in Figure \ref{fig:tau} for $\rho \leq 0.4$ show almost identical patterns and the column for $\rho=0.6$ shows different patterns for $\tau=0$ and $\tau=0.01$. Furthermore, when $\rho$ is large, as seen with $\rho = 0.6$ in Figure \ref{fig:tau}, the magnitude of $\RSsig(x,y)$ is globally close to zero, around the order of $10^{-4}$. Thus, the case most affected by $\tau$ is not particularly interesting. However, a high noise level, such as $\tau = 1$, will substantially influence $\RSsig$. Since the consideration of $\tau$ adds another layer of complexity in addition to the discussion on $\rho$, $S$, $x,y$, it will be studied in future work.

\begin{figure}
    \centering
    \includegraphics[scale=0.5]{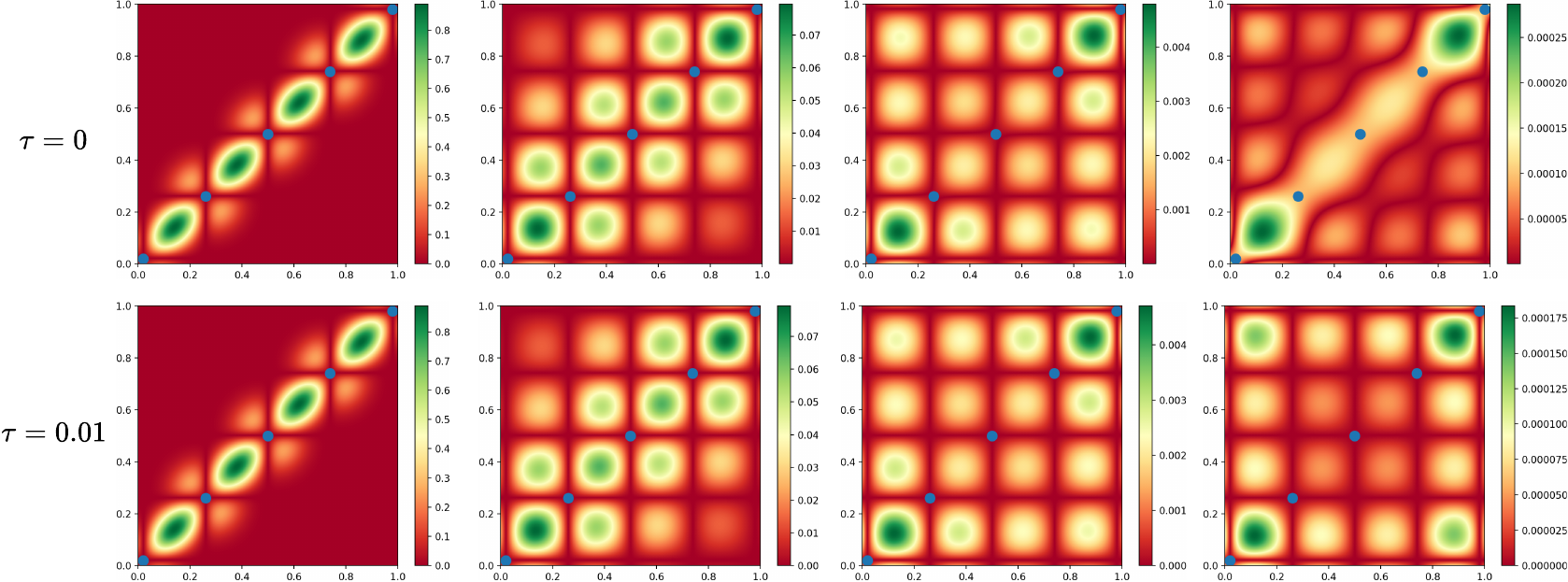}
    \caption{$|\RSsig\cdcd|$ for $\tau=0$(top) and $\tau=0.01$(bottom) with different $\rho$ values (left to right): $\rho=0.1, 0.25, 0.4, 0.6$}
    \label{fig:tau}
\end{figure}

\begin{figure}
    \centering
    \includegraphics[scale=0.4]{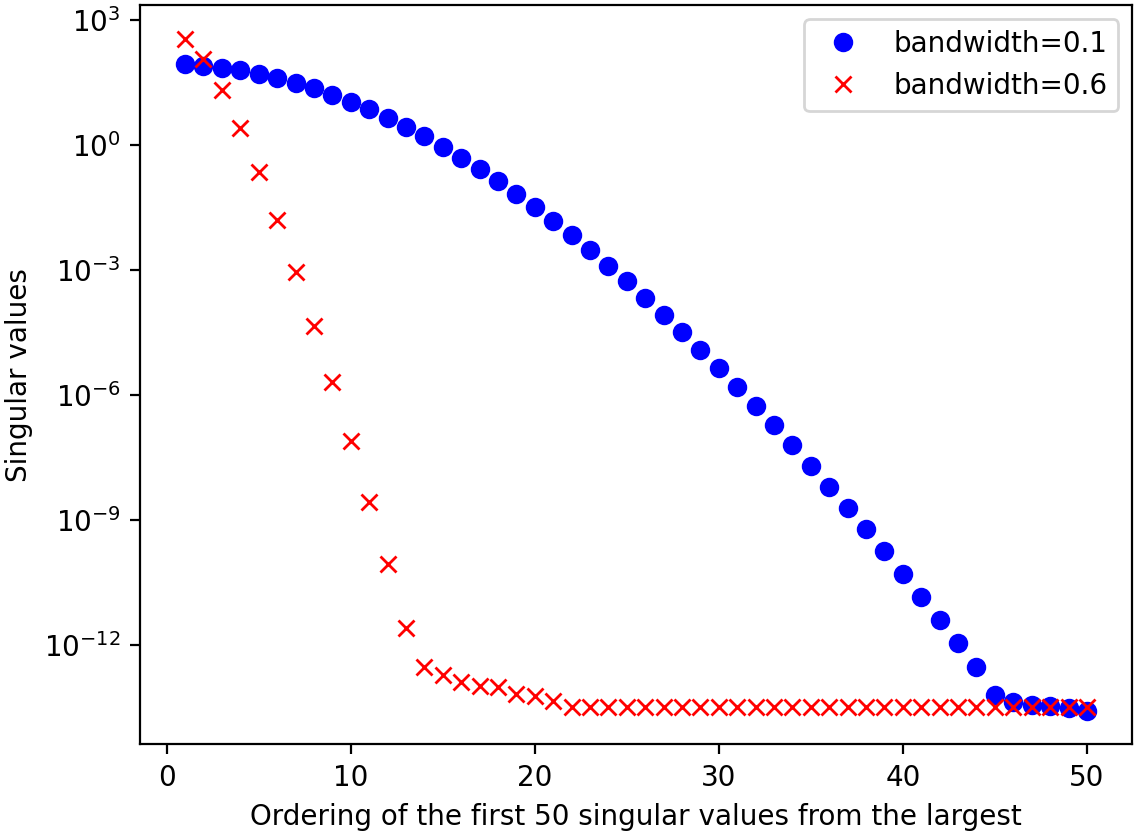}
    \caption{The largest 50 singular values of $K_{XX}$: $\rho=0.1$ (blue disks) and $0.6$ (red crosses). $X$ contains $500$ equispaced points in $[0,1]$.}
    \label{fig:tauSVD}
\end{figure}

Computationally, the added noise acts as a benign regularization, making numerous matrix operations easier to perform compared to the noise-free case. We can see this by looking at the two regimes discussed above: small $\rho$ and large $\rho$.
When $\rho$ is small, $K_{SS}$ is symmetric positive definite and away from being ill-conditioned.
The added noise $\tau^2$ makes the linear system associated with $(K_{SS}+\tau^2 I)$ 
easier to solve by iterative methods such as CG than in the noise-free case since the matrix $K_{SS}+\tau^2 I$ is even better conditioned with a spectrum farther away from zero than $K_{SS}$ in the noise-free case.
When $\rho$ is larger, $K_{SS}$ tends to have smaller rank and can be approximated well by a low-rank factorization. In this case, a good approximation to  
$(K_{SS}+\tau^2 I)^{-1}$ can be efficiently computed  based on the Sherman-Morrison-Woodbury formula. Overall, a nonzero noise parameter $\tau>0$ ``regularizes'' the problem and often leads to easier computational tasks than the noise-free case with $\tau=0$.

\subsection{Geometric Posterior Covariance Indicators}
\label{sub:post}
\cdf{The effectiveness of the estimates has been illustrated in Section \ref{sub:plotbounds}, but these estimates may not be efficient to compute at a large number of locations, as explained at the end of Section \ref{sub:plotbounds}. 
This section aims to leverage the theoretical estimates to develop practical indicators for $\abs{\RSsig(x,y)}$} to help discover the distribution of large values, i.e. finding locations with more dominant values.
For applications in Gaussian processes, we also propose posterior variance indicators, i.e. for $\RSsig(x,x)$. When designing indicators, the ease of computation is of critical importance in engineering practice (cf. \cite{ZZ1987,ZZ1992,verf2013book}).
Hence the computational efficiency of the indicators is often the top priority.
We will show that the proposed indicators are easy to compute, without the need to solve linear systems.
In terms of complexity, assuming that $S$ contains $r$ observations, the evaluation cost of these indicators at one point can be as low as $O(r)$, as compared to $O(r^3)$ using the direct calculation of $\abs{\RSsig(x,y)}$. 

The analysis in Section \ref{sub:Case I} and Section \ref{sub:Case II} shows that the distance to the observation data $S$ (relative to $\rho$) plays an important role in inferring the magnitude $\abs{\RSsig(x,y)}$.
Inspired by the $l_p$ norm, we define the metrics below to measure the distance from an arbitrary point $x$ to $S$ with respect to $\rho$:
\begin{equation}
\begin{aligned}
    h_{\infty}(x) &:= \rho^{-1}\dist{x,S},\\
    h_2(x) &:= \rho^{-1} \left( \sum_{s_i\in S} \norm{x-s_i}^2  \right)^{1/2}.
\end{aligned}
\end{equation}
Based on the distance metrics above,
we design \emph{relative indicators} $g(x,y)$ to capture the distribution of $\abs{\RSsig(x,y)}$ over $\Omega\times\Omega$
in the sense that { $\frac{g(x,y)}{\max\limits_{x,y\in\Omega} g(x,y)}$ is similar to $\frac{\abs{\RSsig(x,y)}}{\max\limits_{x,y\in\Omega}\abs{\RSsig(x,y)}}$.}
Then an \emph{absolute indicator} can be defined as 
\begin{equation*}
\max_{x,y\in\Omega}\abs{\RSsig(x,y)} 
\times \frac{g(x,y)}{\max\limits_{x,y\in\Omega} g(x,y)}.
\end{equation*}
If one is interested in the location of larger values of $\abs{\RSsig}$, then the relative indicator is sufficient.

For the small bandwidth case,
we propose a relative indicator of $\abs{\RSsig(x,y)}$ as
\begin{equation}
\label{eq:gsmall}
    g(x,y) := \sqrt{h_{\infty}(x)h_{\infty}(y)}\ksig(x,y).
\end{equation}

For the large bandwidth case,
we define the relative indicator as
\begin{equation}
\label{eq:glarge}
    g(x,y) := h_{\infty}(x)h_{\infty}(y)h_2(x)h_2(y).
\end{equation}
Note that the indicator in \eqref{eq:glarge} does \emph{not} involve the kernel $\ksig(x,y)$. 
This makes sense since for large $\rho$, $\ksig(x,y)$ varies slowly over the domain and the pattern of $\RSsig(x,y)$ is quite different from $\ksig(x,y)$ according to the experiments in Section \ref{sub:Case II}. It is easy to see that the computational cost of $g(x,y)$ grows \emph{linearly} in the number of observations in $S$.

In Figures \ref{fig:est-unif} and \ref{fig:est-non},
we plot the true absolute posterior covariance function $\abs{\RSsig(x,y)}$ and the estimated function 
\begin{equation}
\max_{x,y\in\Omega}\abs{\RSsig(x,y)} 
\times \frac{g(x,y)}{\max_{x,y\in\Omega} g(x,y)}
\quad (x,y)\in\Omega\times\Omega,
\end{equation}
where $\Omega=[0,1]$ and $S$ contains 5 points.
Two cases of $S$ are shown: uniform and non-uniform, in Figure \ref{fig:est-unif} and Figure \ref{fig:est-non}, respectively.
For $\rho=0.4$,
we use the indicator $g(x,y)$ in \eqref{eq:glarge}.
For smaller values $\rho<0.3$, we use the indicator $g(x,y)$ in \eqref{eq:gsmall}.
It can be seen that the indicator is able to approximately capture the pattern of the true posterior covariance, particularly in areas with large values.
The indicators will be useful when the exact function $\RSsig$ is too costly to compute due to the large-scale observation data $S$ or the numerical difficulty in dealing with $K_{SS}^{-1}$ in $\RSsig$.
In many applications, the task is often not to calculate $\RSsig(x,y)$, but to determine where $\abs{\RSsig(x,y)}$ is large over the domain $\Omega\times\Omega$.
The theory and relative indicators provide a geometric characterization of the relatively ``important'' locations, which allows straightforward calculations of these locations without directly evaluating $\RSsig(x,y)$ using the formula in Definition \ref{eq:R}.
The computational complexity for \eqref{eq:gsmall} or \eqref{eq:glarge} at each location is \emph{optimal}, i.e. $O(r)$ for $r$ observations in $S$.
Hence the indicators are useful in quickly identifying the locations with large variance or covariance magnitude in the posterior distribution.
Some applications in numerical linear algebra are presented in Sections \ref{sub:app-Approximation} and \ref{sub:app-Preconditioning}.

\begin{figure}[htbp]
\centering 
\includegraphics[scale=.32]{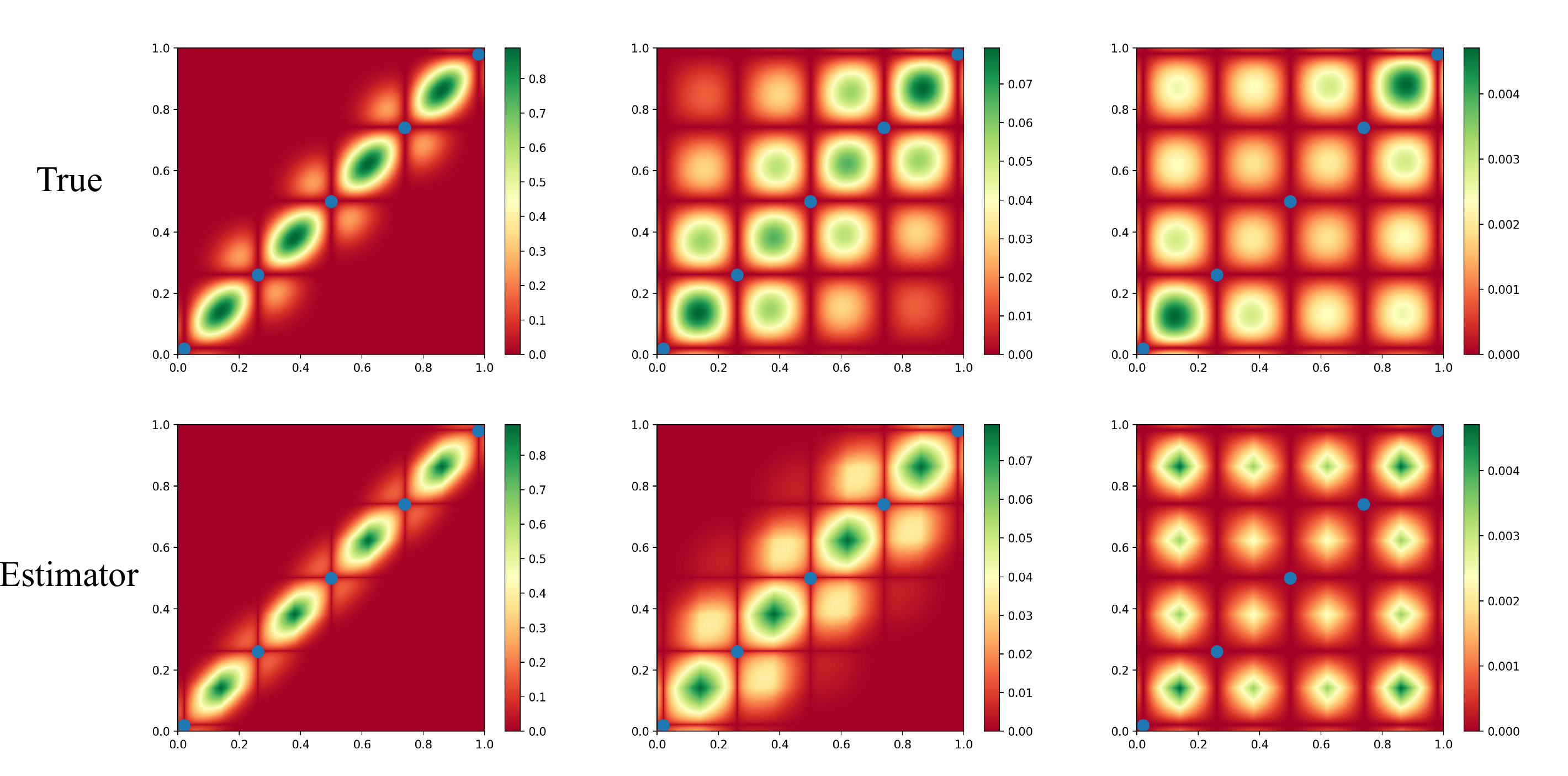}
\caption{True posterior covariance (top) vs indicator (bottom): uniform observations and different $\rho$ values (left to right): $\rho=0.1, 0.25, 0.4$.}
\label{fig:est-unif}
\end{figure}

\begin{figure}[htbp]
\centering 
\includegraphics[scale=.32]{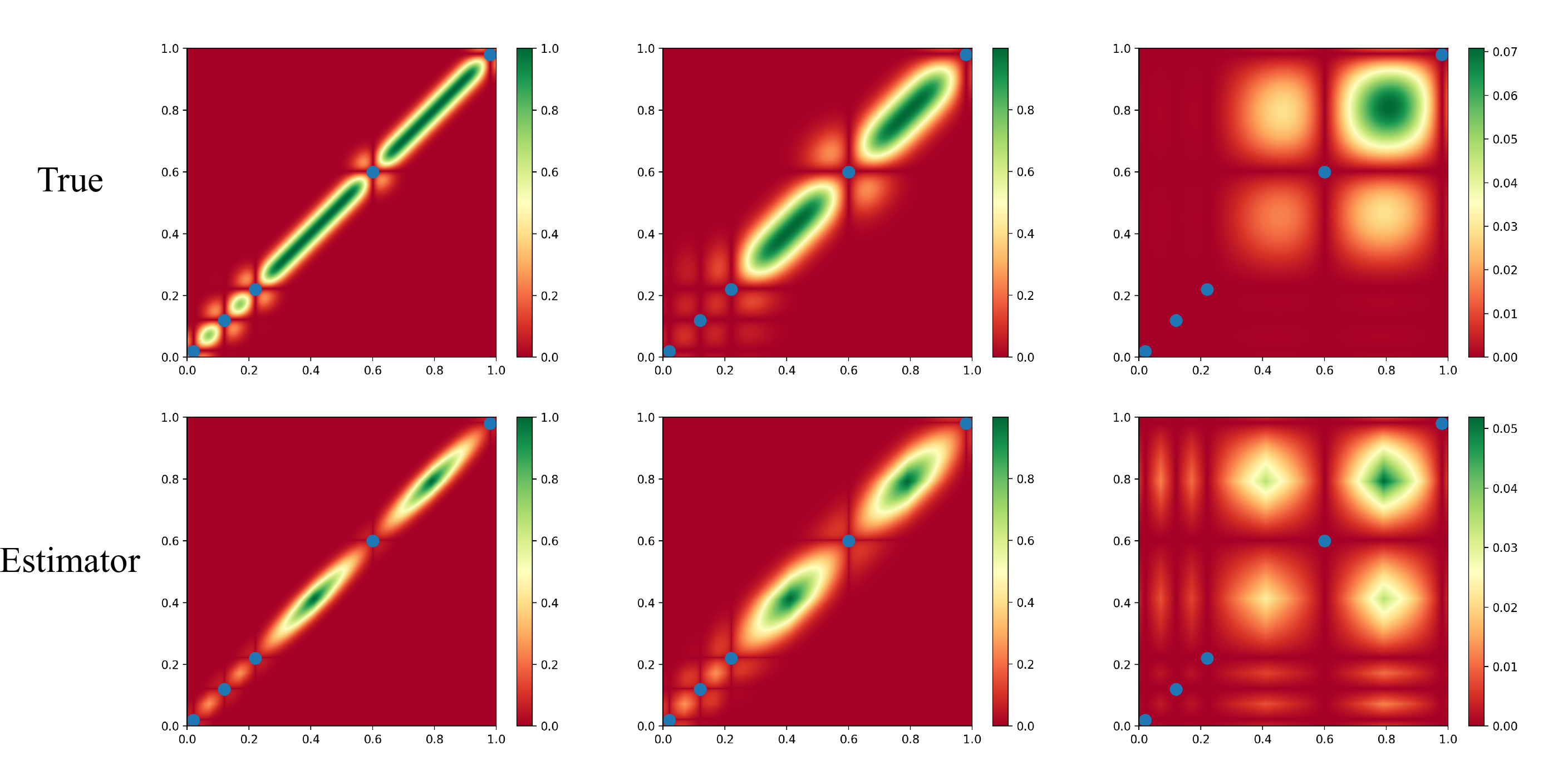}
\caption{True posterior covariance (top) vs indicator (bottom): non-uniform observations and different $\rho$ values (left to right): $\rho=0.1, 0.25, 0.4$.}
\label{fig:est-non}
\end{figure}

{
\textbf{Posterior variance indicators.}
\cdf{In Gaussian processes, the posterior variance $\Var(x)$ at $x$ (conditioned on $S$) is defined in \eqref{eq:Var}.}
To facilitate the computation of $\Var(x)$ without having to evaluate $K_{SS}^{-1}K_{Sx}$ for \emph{all} $x$, in the following, we propose the indicators below, depending on the case of small or large bandwidth.}

For the case of \emph{small} bandwidth,
inspired by Corollary \ref{cor:variance},
we construct the posterior variance indicator:
\begin{equation}
\label{eq:VarEst}
    \mathcal{V}(x) := 1-\exp\left(-\frac{\nu^2(x)}{2\rho^2}\right)\quad 
    \text{with}\quad\nu(x)=\dist{x,S}.
\end{equation}
\cdf{We explain how Corollary \ref{cor:variance} inspires the above choice of $\mathcal{V}$ as follows.
Note that Corollary \ref{cor:variance} accounts for the small bandwidth regime, and in the small bandwidth limit, the constant $\Gamma_2\to 1$ according to Proposition \ref{prop:limitCase1}. Hence to simplify the computation in this case, it makes sense to set $\Gamma_2=1$ in view of the estimate in Corollary \ref{cor:variance}. We also set $r=1$ so that $1-e^{-\hat{\omega}^2}$ is always nonnegative. 
It remains to explain the choice of $\hat{\omega}$. Here we simply choose $\hat{\omega}$ to be the value that achieves the threshold in the condition: 
$\dist{x,S}\geq\sqrt{2}\hat{\omega}\rho$. That is, $\hat{\omega}=\dist{x,S}/(\sqrt{2}\rho)$. This is the largest $\hat{\omega}$ possible that fulfills the condition and gives a more meaningful estimate than other choices (for example, $\hat{\omega}=0$). This choice is also used in deriving the estimates in Section \ref{sub:plotbounds} for the small bandwidth case.}

Note that $\mathcal{V}(x)$ can be computed easily.
The complexity for computing $\mathcal{V}(x)$ is $O(r)$ for $r$ observations in $S$ while the complexity is $O(r^3)$ for directly computing $\Var(x)$ due to the inversion $K_{SS}^{-1}$.
More generally, for the prior covariance in \eqref{eq:cov},
the optimal values of $\sigma^2$ and $\rho$ will to be computed via maximum likelihood estimation in \eqref{eq:MLE}.
In this general case (with possibly non-unit variance $\sigma^2$),
to estimate the \emph{posterior} variance,
we modify the indicator in \eqref{eq:VarEst} slightly to account for the $\sigma^2$ factor:
\begin{equation}
    \label{eq:GP-var}
    \mathcal{V}_{\sigma^2}(x) := \sigma^2\left[1-\exp\left(-\frac{\nu^2(x)}{2\rho^2}\right)\right]\quad 
    \text{with}\quad\nu(x)=\dist{x,S}.
\end{equation}
Finally, $\sqrt{\mathcal{V}_{\sigma^2}(x)}$ serves as an indicator for the posterior standard deviation.
The indicator \eqref{eq:GP-var} is exact if $x$ is an observation point, i.e. 
$\mathcal{V}_{\sigma^2}(x)=\sigma^2(1-1)=0$ if $x\in S$.
If $x$ is far from any observation point, then $\mathcal{V}_{\sigma^2}(x)$ converges to the \emph{prior} variance:
$$\mathcal{V}_{\sigma^2}(x) \to \sigma^2(1-0)=\sigma^2\quad\text{as}\;\, \dist{x,S}\to\infty.$$

For the case of \emph{large} bandwidth,
Theorem \ref{thm:error2} indicates that the variance at a point increases with the distance to the observation set.
Inspired by Theorem \ref{thm:error2}, we use the exact posterior variance at certain points (called reference points) as reference values and construct the posterior variance by comparing the location to the closest reference point.
The reference points are chosen to occupy the ``gap'' between points in $S$ and to stay away from $S$. 
For $S\subseteq \mathbb{R}$, the reference points can be chosen as midpoints between adjacent observations in $S$.
Note that these points can be computed off-line efficiently.
Let $z_x$ denote the closest reference point to $x$.
We define the posterior variance indicator as
\begin{equation}
\label{eq:GP-var2}
   \mathcal{V}_{\sigma^2}(x):= \frac{\dist{x,S}}{\dist{z_x,S}} \Var(z_x).
\end{equation}
Note that the indicator $\mathcal{V}_{\sigma^2}$ is exact for reference points as $\mathcal{V}_{\sigma^2}(z)=\Var(z)$ for any reference point $z$, and observation points as $\mathcal{V}_{\sigma^2}(s)=0$ for $s\in S$.

It should be emphasized that the indicator in \eqref{eq:VarEst} or \eqref{eq:GP-var}, derived from Corollary \ref{cor:variance} in Section \ref{sub:Case I}, applies to the case of \emph{small} bandwidth and is generally not suitable for \emph{large} bandwidth (when $\RSsig(x,y)$ is generally close to zero).
Similarly, the indicator in \eqref{eq:GP-var2} is developed for the large bandwidth case, and is not suitable for the small bandwidth case.
This is illustrated in Section \ref{sub:plotbounds}.


We present an experiment to illustrate the proposed indicators.
Consider $\Omega=[0,1]$, $f(x)=\cos(25 x^2)$.
The set of observations $S$ contains 15 randomly distributed points in $\Omega$ shown as red dots in Figure \ref{fig:GP}.
The prior covariance kernel follows \eqref{eq:cov} with variance $\sigma^2$ and bandwidth $\rho$.
The parameter values determined after training are:
$$\sigma^2 = 0.9453058162554949,\quad \rho = 0.06332725946674625.$$
We estimate the posterior standard deviation by taking the square root of the variance indicators in \eqref{eq:GP-var} and \eqref{eq:GP-var2}.
The Case I indicator \eqref{eq:GP-var} is used in regions where the distance between adjacent observations exceeds $2\rho$.
Otherwise, the Case II indicator \eqref{eq:GP-var2} is used.
Figure \ref{fig:GP} illustrates the true standard deviation (top) and the estimate (bottom), indicated by shaded regions.
It can be seen that the estimated standard deviation from the variance indicators captures the behavior of the true deviation quite well across the entire computational domain.
Figure \ref{fig:GP} also shows that the posterior variance is large in regions with few observations, while being very small in regions with many observations.

\begin{figure}
    \centering
    \includegraphics[scale=0.55]{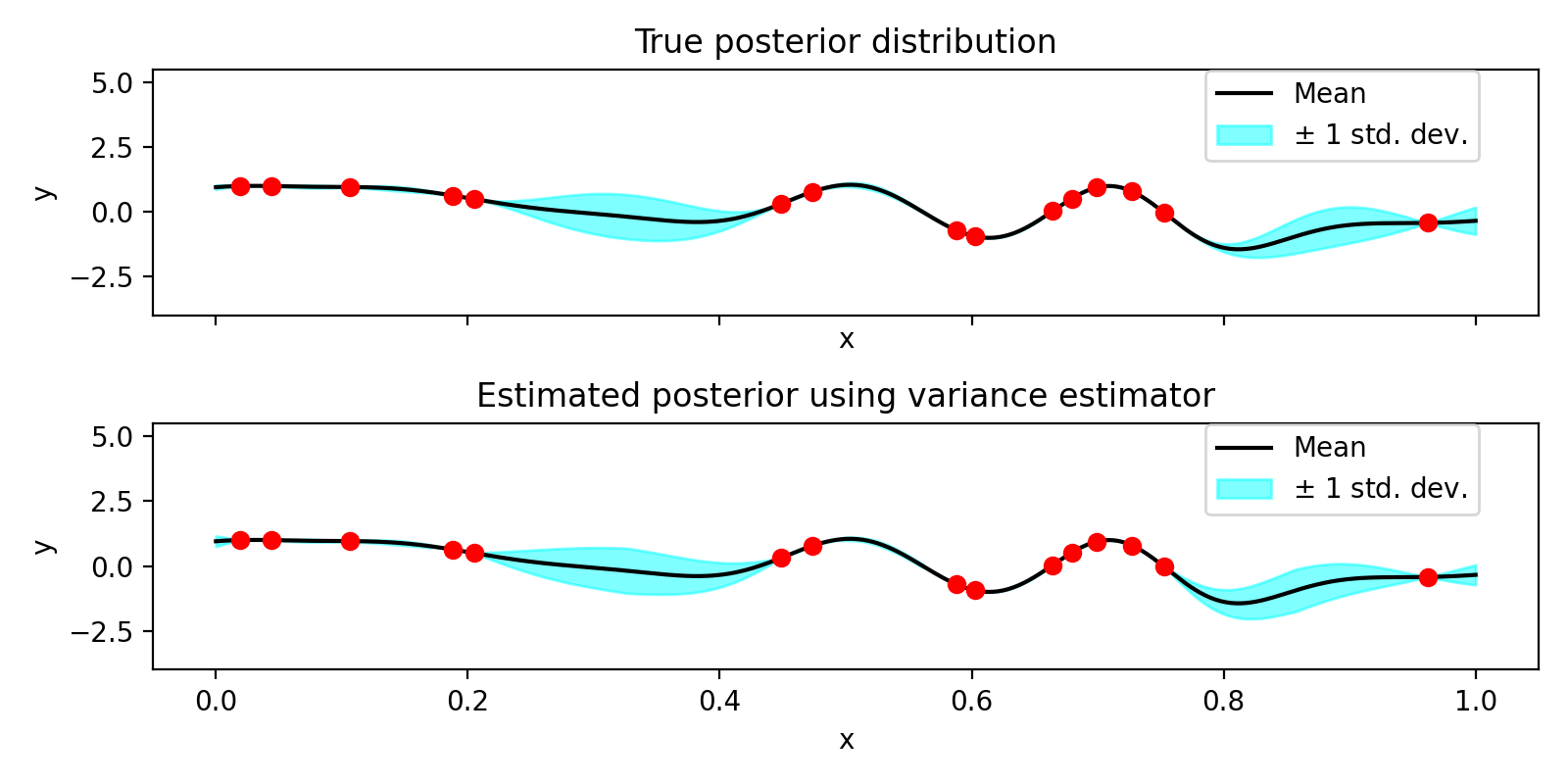}
    \caption{{Posterior regression curve from Gaussian process. Shaded uncertainty: Top: true standard deviation; Bottom: indicator $\sqrt{\mathcal{V}}$ from \eqref{eq:GP-var} and \eqref{eq:GP-var2}.}}
    \label{fig:GP}
\end{figure}

\cdf{
\textbf{Summary of practical indicators.}
We summarize the theoretical results and the proposed practical indicators in the table below.

\begin{tabular}{c|c|c|c}
\hline 

\hline
    Condition & Effect & Theory & Indicator \\
\hline 

\hline    
    large $\dist{x,y}/\rho$& small $\abs{\RSsig(x,y)}$ & Theorem 3.3 & (3.15) \\
     small $\dist{x,y}/\rho$, large $\dist{x,S}/\rho$& large $\abs{\RSsig(x,y)}$ & Theorem 3.5 & (3.15)\\
     small $\max\limits_{x\in\Omega}\dist{x,y}/\rho$ & ``tensor-product'' $\abs{\RSsig}$ & Theorem 3.7 & (3.16) \\
     large $\dist{x,S}/\rho$& large Var$(x)$ & Corollary 3.6 & (3.19) \\
     small $\dist{x,S}/\rho$& small Var$(x)$ & Theorem 3.7 & (3.20) \\
\hline

\hline
\end{tabular}
}

\section{Numerical Experiments} 
\label{sec:numerical}
In this section, we present several numerical experiments. 
The first set of experiments from Section \ref{sub:Error pattern1} to Section \ref{sub:Error pattern2D} aims to use the theory in Section \ref{sec:theory} to understand the pattern of the posterior covariance function $\RSsig\cdcd$.
Then we present applications in matrix approximation and linear system preconditioning, to show how the geometric understanding of the distribution in $\RSsig\cdcd$ can be used to achieve better accuracy or efficiency.
All experiments were conducted in MATLAB R2021a on a MacBook Pro with Apple M1 chip and 8GB of RAM.

\subsection{Pattern of posterior covariance: uniform data}
\label{sub:Error pattern1}
In this experiment, we consider the posterior covariance function $\RSsig\cdcd$ over $[0,1]\times [0,1]$ with uniformly distributed observation points below
$$S=\{0.02,0.26,0.5,0.74,0.98\}.$$
This complements the examples illustrated in Figures \ref{fig:prelim1} -- \ref{fig:prelim2} in Section \ref{sec:prelim}.
The one-dimensional setting in this experiment allows us to visualize the pattern on the plane, which can help develop a more straightforward understanding of how the function behaves in different scenarios.
To see how the bandwidth $\rho$ in the Gaussian kernel affects the distribution $\RSsig\cdcd$,
we consider three values:
$$\rho = 0.1, 0.25, 0.4,$$
representing small to large bandwidth values
compared to the data spacing in $S$, which is 0.24. 
The heat map of the posterior covariance $\abs{\RSsig(x,y)}$ for each case is plotted in Figure \ref{fig:uf-error}.
The color at each point $(x,y)$ corresponds to the value $\abs{\RSsig(x,y)}$, except for the \emph{five} blue points which are located at $(s_i,s_i)$ for each $s_i$ in $S$.

The ``banded'' pattern in Figure \ref{fig:uf-sigma01} and the ``tensor product'' pattern in Figure \ref{fig:uf-sigma04} have been explained in Section \ref{sub:Case I} and Section \ref{sub:Case II}, respectively.
Figure \ref{fig:uf-sigma025} renders a different pattern from the other two.
The bandwidth $\rho=0.25$ lies in between $0.1$ and $0.4$, and thus the pattern in Figure \ref{fig:uf-sigma025} looks like an intermediate stage between Figure \ref{fig:uf-sigma01} and Figure \ref{fig:uf-sigma04}.
The connections from Figure \ref{fig:uf-sigma025} to the two extreme cases ($\rho=0.1$ and $\rho=0.4$) are easy to see:
the ``tensor product'' structure is similar to Figure \ref{fig:uf-sigma04}; the dominant values are achieved near the diagonal where $\norm{x-y}$ is relatively small, and the magnitude of $\RSsig$ generally decays as $\norm{x-y}$ increases, analogous to Figure \ref{fig:uf-sigma01}.
The intermediate stage displays features from two limit cases but not as prominently.

\begin{figure}[htbp] 
    \centering 
\subfloat[$\rho=0.1$]{ 
\label{fig:uf-sigma01}
\includegraphics[scale=.35]{./fig/uf-sigma01.png}
}
\hspace{10pt}
\subfloat[$\rho=0.25$]{ 
\label{fig:uf-sigma025}
\includegraphics[scale=.35]{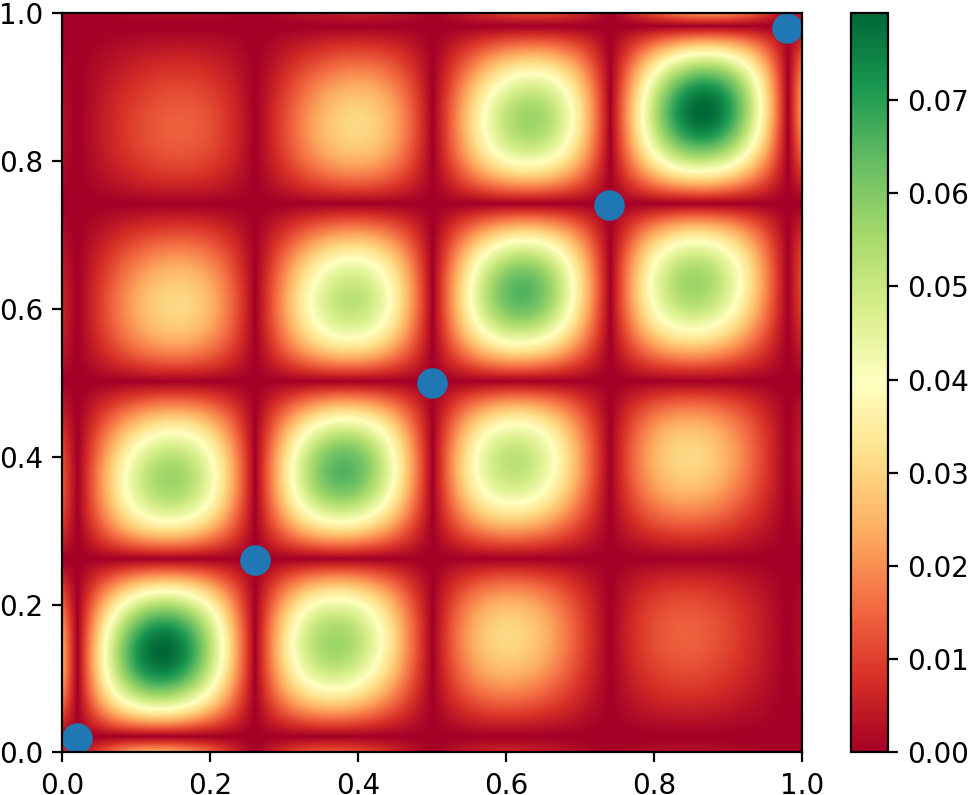}
}
\hspace{10pt}
\subfloat[$\rho=0.4$]{ 
\label{fig:uf-sigma04}
\includegraphics[scale=.35]{./fig/uf-sigma04.png}
}
\caption{$|\RSsig(x,y)|$ over $[0,1]\times [0,1]$ with \emph{uniform} $S$: different $\rho$.}
\label{fig:uf-error}
\end{figure}

\subsection{Pattern of posterior covariance: non-uniform data}
\label{sub:Error pattern2}
Following Section \ref{sub:Error pattern1}, in this experiment, we investigate $\abs{\RSsig}$ with \emph{non-uniform} observation data
$$S=\{0.02,0.12,0.22,0.6,0.98\},$$
where the spacing between the \emph{first} three points (0.1) is much smaller than that of the \emph{last} three points (0.38). 
We test the following three bandwidth values
$$\rho = 0.05,0.1,0.4.$$
The heat map of $|\RSsig(x,y)|$ for each case is shown in Figure \ref{fig:nf-error}.
Though the plots look different from the uniform case in Section \ref{sub:Error pattern1},
they can still be explained by the theory in Section \ref{sec:theory}.

{
\paragraph{On Figure \ref{fig:nf-sigma005}: $\rho=0.05$}
In this plot, the bandwidth $\rho=0.05$ is considered relatively small compared to the spacing between the first three observation points as well as the spacing between the last three observation points.
Such a scenario falls into ``small bandwidth case'' in Section \ref{sub:Case I} and the explanation of the plot is similar to Figure \ref{fig:case1small} and Figure \ref{fig:case1}.
}

\paragraph{On Figure \ref{fig:nf-sigma01}: $\rho=0.1$}
In this plot, the top right region in $[0.22,1]\times [0.22,1]$ displays a typical pattern (similar to Figure \ref{fig:case2}) for the \emph{small} bandwidth case in Section \ref{sub:Case I}, where $\rho$ is considered \emph{small} for the last three observation points. The pattern can be explained using Theorem \ref{thm:error2}. In the lower left region, we can observe a ``tensor-product'' structure containing $3\times 3=9$ blocks instead of a ``banded'' structure, which implies that the bandwidth $\rho$ is not considered small for the first three observation points. 
This ``banded'' pattern of dominant values is similar to Figure \ref{fig:nf-sigma005}.
Note that the large values of $\abs{\RSsig(x,y)}$ in the lower left part (``tensor-product'' region) are much smaller than the large values in the upper right part (``banded'' region).
This is because $K_{xS}K_{SS}^{-1}K_{Sy}$ approximates $\ksig(x,y)$ much better in this region than in the ``small bandwidth case''. 
As a result, $\abs{\RSsig(x,y)}$ is much smaller in the lower left part.


\paragraph{On Figure \ref{fig:nf-sigma04}: $\rho=0.4$}
\cdf{In this case, $\rho=0.4$ represents the \emph{large} bandwidth regime discussed in Section \ref{sub:Case II}. 
The ``tensor-product'' pattern can be observed in $[0.3,1]\times [0.3,1]$.}
The lower left part appears totally red simply because $\abs{\RSsig(x,y)}$ is almost zero, negligible compared to the values in the upper right part.
The largest values are in the green region (at around $[0.65,0.95]^2$). 
This is also consistent with the claim in Section \ref{sub:Case II} that \emph{larger distance to $S$ implies larger $\RSsig(x,y)$} with a ``cumulative'' distance metric as discussed in the end of Section \ref{sub:Case II}.
Such a boundary effect is typical for the \emph{large} bandwidth regime, similar to Figure \ref{fig:uf-sigma04}.

\begin{figure}[htbp] 
    \centering 
\subfloat[$\rho=0.05$]{ 
\label{fig:nf-sigma005}
\includegraphics[scale=.35]{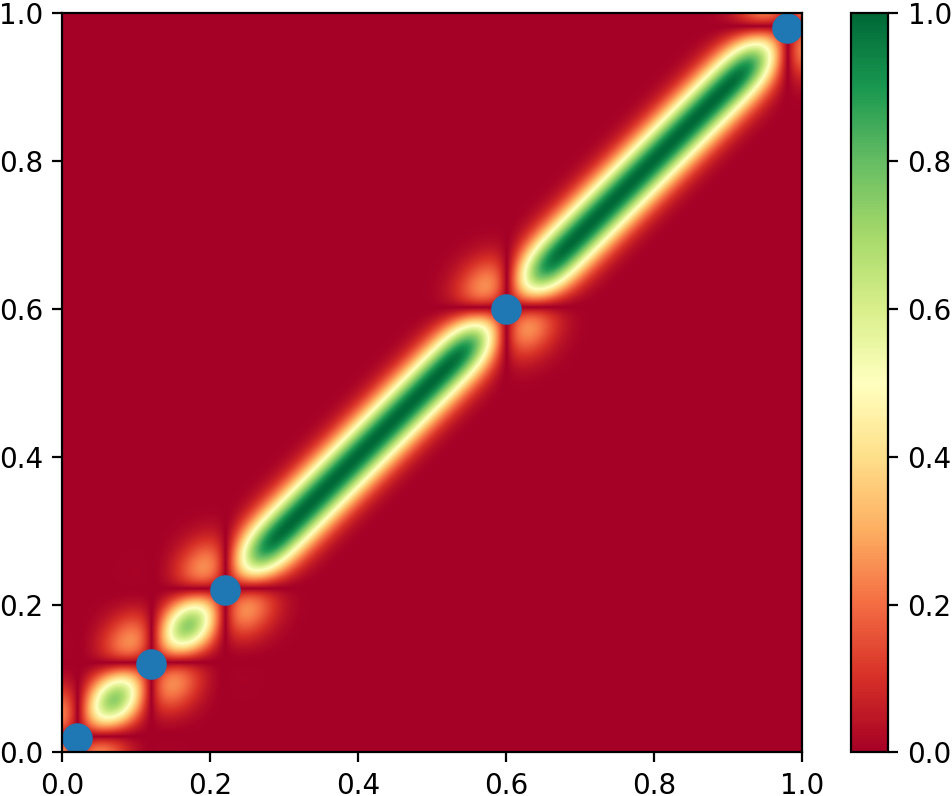}
}
\hspace{10pt}
\subfloat[$\rho=0.1$]{ 
\label{fig:nf-sigma01}
\includegraphics[scale=.35]{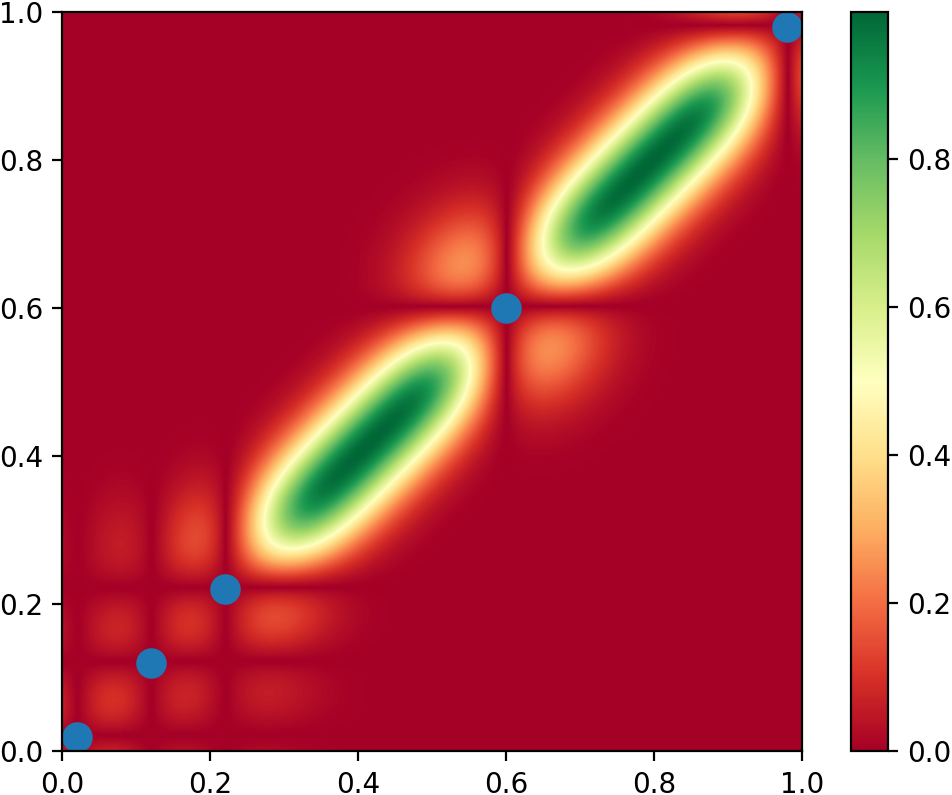}
}
\hspace{10pt}
\subfloat[$\rho=0.4$]{ 
\label{fig:nf-sigma04}
\includegraphics[scale=.35]{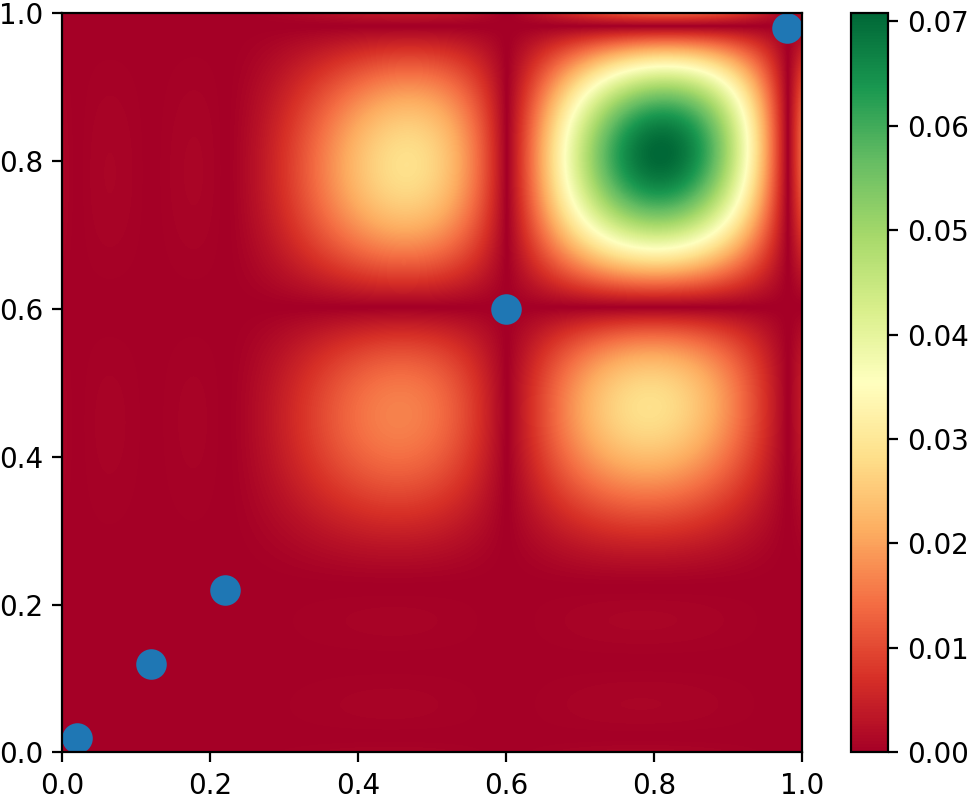}
}
\caption{Function $|\RSsig(x,y)|$ over $[0,1]\times [0,1]$ with \emph{non-uniform} $S$: different $\rho$.}
\label{fig:nf-error}
\end{figure}

\subsection{Pattern of posterior covariance of 2D data}
\label{sub:Error pattern2D}
Sections \ref{sub:Error pattern1} and \ref{sub:Error pattern2} illustrate the bivariate function $|\RSsig(x,y)|$ for one dimensional data.
In this experiment, we present the numerical study for two dimensional data.
We consider points in the disk $B_{0.4}$ centered at the origin with radius $0.4$, and investigate
$\abs{\RSsig(x,y)}$ for $x,y\in B_{0.4}.$
The set of observation points $S\subseteq B_{0.4}$ is fixed, as illustrated in Figure \ref{fig:2d-error} by blue dots.
To see the impact of the bandwidth $\rho$, we test the following values:
\begin{equation*}
    \rho = 0.03, 0.05, 0.1, 0.2, 0.3.
\end{equation*}

Since $\RSsig(x,y)$ is a function over four dimensions $\mathbb{R}^2\times \mathbb{R}^2$, it can not be visualized on the plane.
To visualize the result, we choose a point $x^*\in B_{0.4}$ and plot the univariate function 
$$g_{x^*}(y):=|\RSsig(x^*,y)|$$ over $y\in B_{0.4}$.
The goal is to (1) show how $|\RSsig(x,y)|$ depends on $x$, $y$, $S$ and $\rho$;
(2) analyze the behavior of  $|\RSsig(x,y)|$ using the theoretical results in Section \ref{sec:theory}. 


{
First, we discuss the impact of $x^*$.
Figure \ref{fig:2d-error-FarClose} shows that when $x^*$ is quite close to an observation point in $S$,
$g_{x^*}(y)=|\RSsig(x^*,y)|$ is small for all $y$ in the domain.
This is due to Theorem \ref{thm:vanish} and the continuity of $\RSsig(x,y)$, as already mentioned after Theorem \ref{thm:vanish}.
When $x^*$ is not close to $S$, then the error can be large for $y$ in certain area that will depend on $\rho$ and $S$ as detailed below.

\begin{figure}[htbp] 
    \centering 
    \includegraphics[scale=.8]{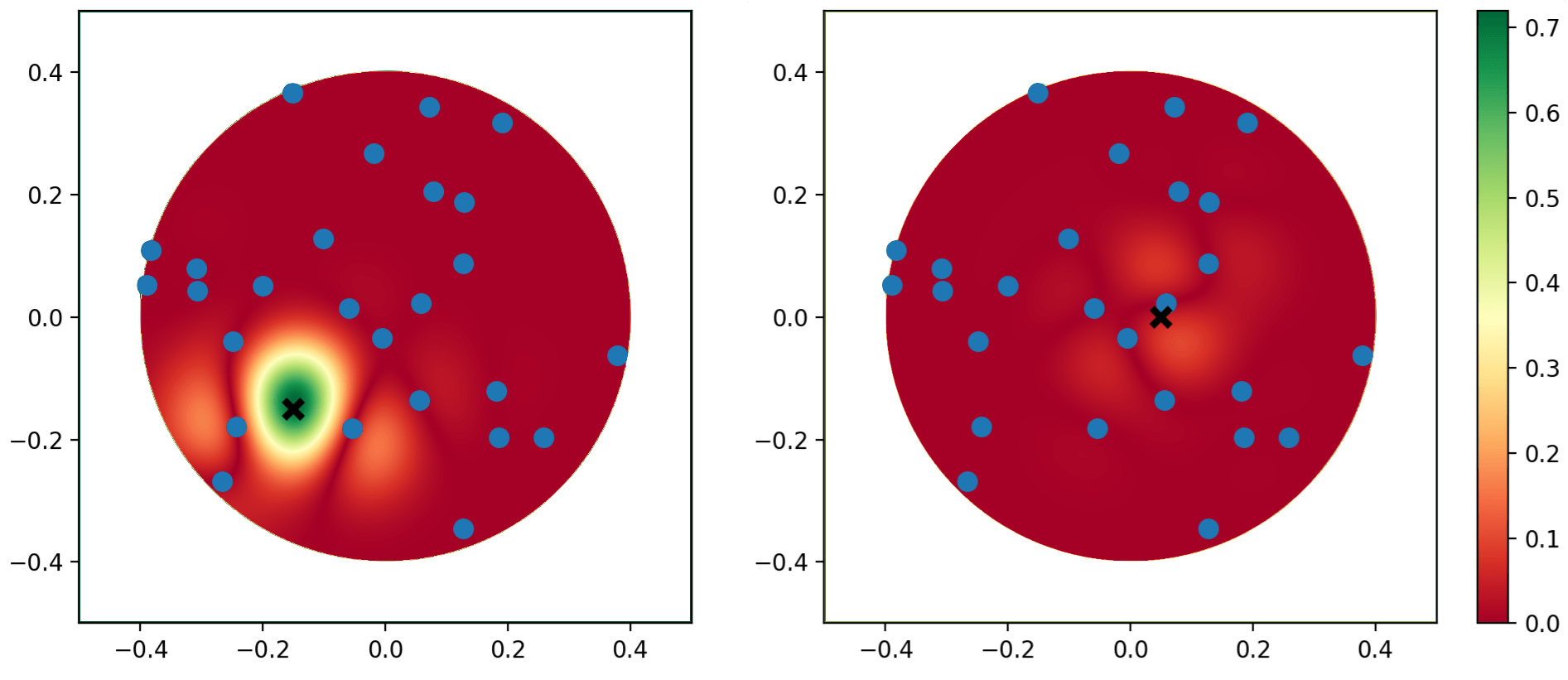}
    \caption{Function $|g_{x^*}(y)|$ over $y\in B_{0.4}$ with $\rho=0.1$ (dark cross: $x^*$; blue dots: $S$). Left: $x^*$ far from $S$; Right: $x^*$ close to $S$.}
    \label{fig:2d-error-FarClose} 
\end{figure}

Now we focus on the impact of $\rho$ relative to $S$, $x$ and $y$ in $\RSsig(x,y)$.
Similar to the one dimensional case discussed in Sections \ref{sub:Error pattern1} and \ref{sub:Error pattern2},
the pattern for $|\RSsig(x,y)|$ according to Figure \ref{fig:2d-error} can be summarized as:
(i) when $\rho$ is relatively small, larger values occur at places where $\norm{x-y}$ is small and $x,y$ are not so close to $S$;
(ii) when $\rho$ is relatively large, larger values occur at places where $x$ or $y$ is relatively far from $S$.
To see (i), note that the green regions (larger values) in Figures \ref{fig:2d-s004} to \ref{fig:2d-s01} consists of points $y$ that are close to $x=x^*$ and not so close to $S$.
To see (ii), we observe that the green regions (larger values) in Figures \ref{fig:2d-s02} to \ref{fig:2d-s03} consists of points $y$ that are relatively far from $S$.
(i) is attributed to Theorem \ref{thm:lowerbound} from the \emph{small} bandwidth case while (ii) can be explained by Theorem \ref{thm:error2} from the \emph{large} bandwidth case.

As $\rho$ increases from $0.04$ in Figure \ref{fig:2d-s004} to $0.3$ in Figure \ref{fig:2d-s03}, 
it can be easily seen that the magnitude of the covariance function $\abs{\RSsig}$ decays.
This is because larger $\rho$ makes the kernel $\ksig$ smoother, thus $\abs{\RSsig(x,y)}$, viewed as the low-rank approximation error for $\ksig(x,y)$ by $K_{xS}K_{SS}^{-1}K_{Sy}$, becomes smaller.
In the limit: $\rho\to\infty$, we have $\ksig(x,y)\to 1$, a constant, and thus can be approximated well with one observation point only.
}

\begin{figure}[htbp] 
    \centering 
\subfloat[$\rho=0.04$]{ 
\label{fig:2d-s004}
\includegraphics[scale=.31]{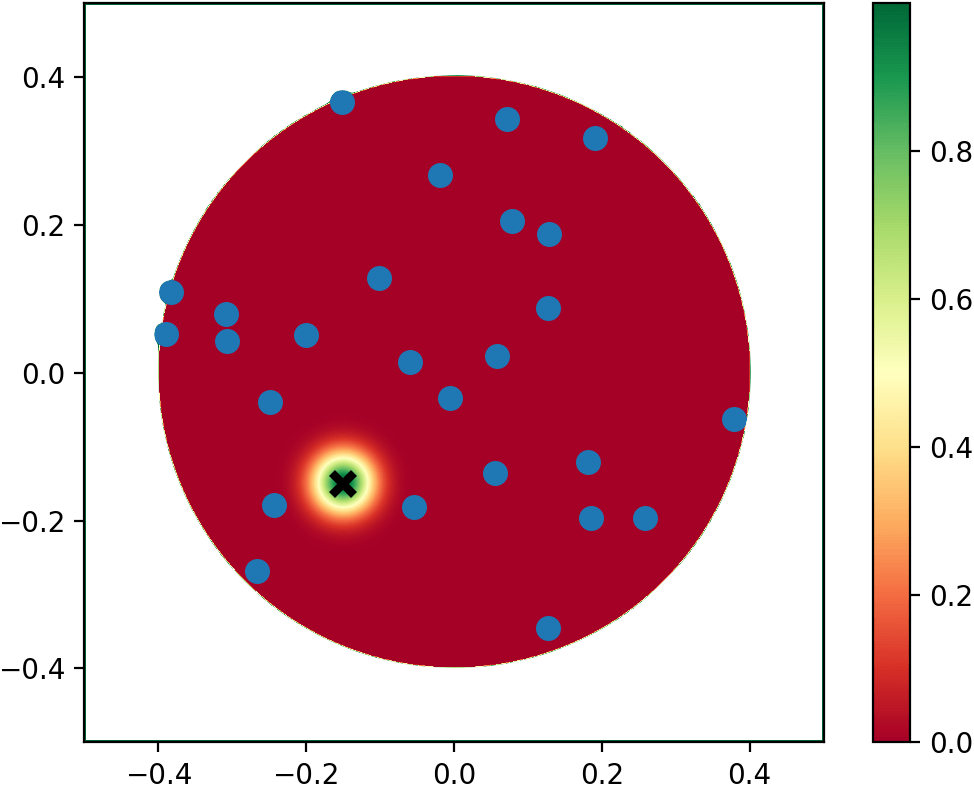}}
\subfloat[$\rho=0.1$]{ 
\label{fig:2d-s01}
\includegraphics[scale=.31]{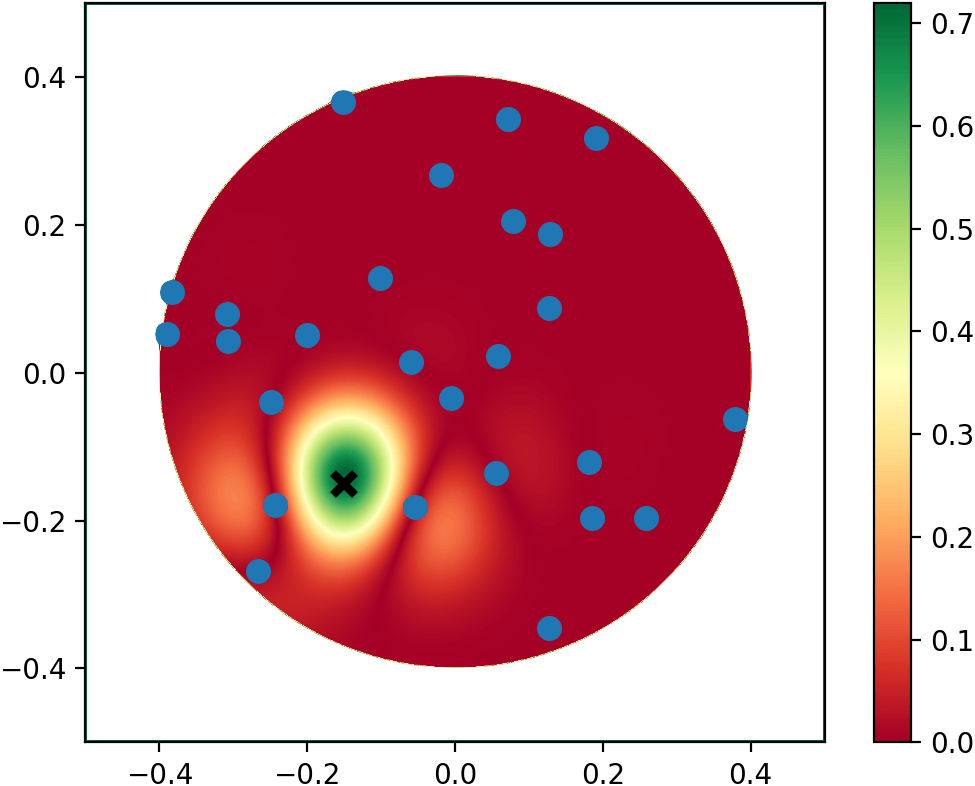}}
\subfloat[$\rho=0.2$]{ 
\label{fig:2d-s02}
\includegraphics[scale=.31]{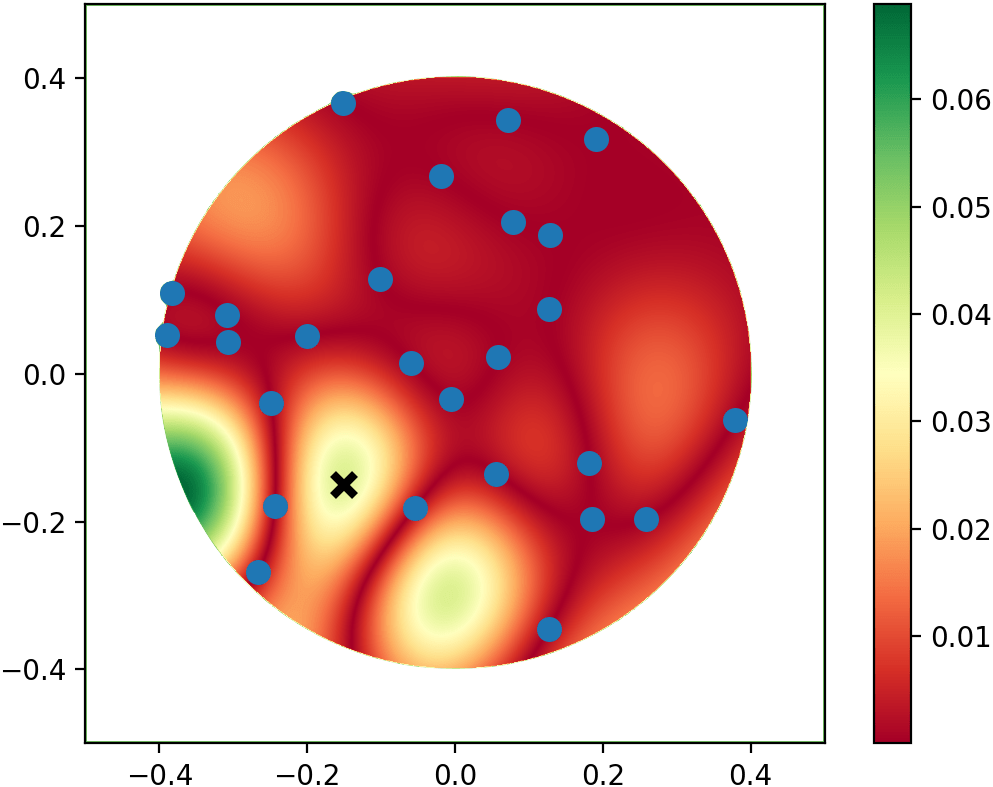}}
\subfloat[$\rho=0.3$]{ 
\label{fig:2d-s03}
\includegraphics[scale=.31]{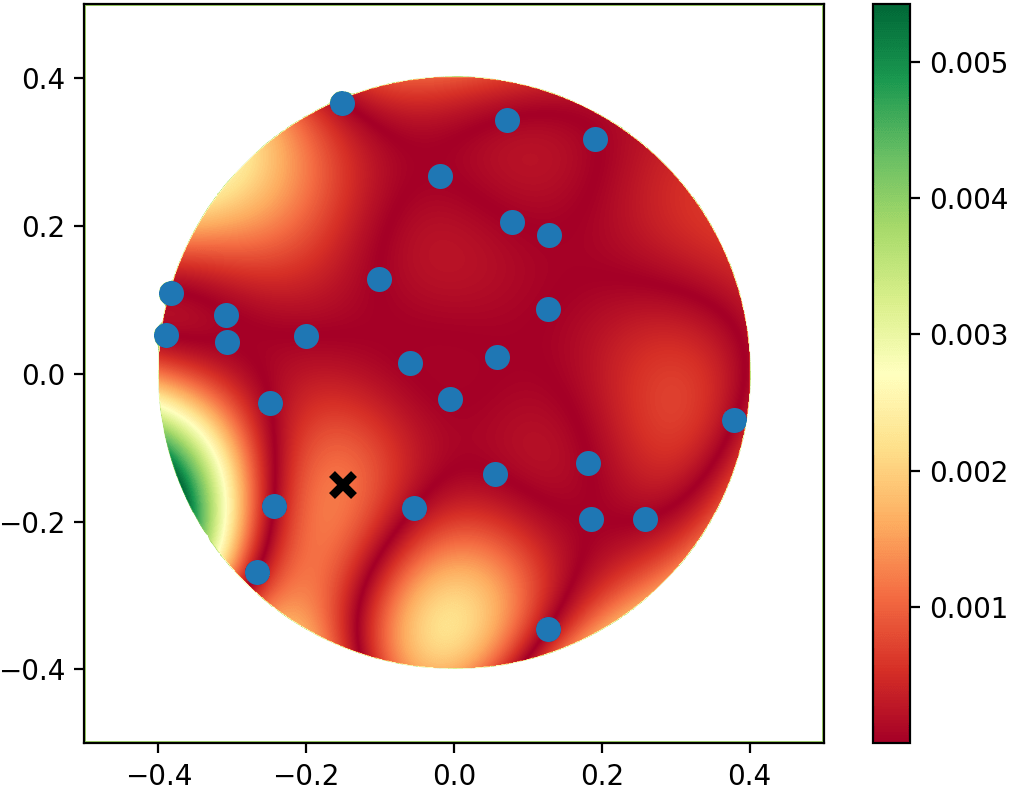}}
\caption{Function $|g_{x^*}(y)|$ over $y\in B_{0.4}$ (dark cross: $x^*$; blue dots: $S$): different $\rho$.}
\label{fig:2d-error}
\end{figure}

\cdf{
\subsection{Comparison to the nearest neighbor approach}
To accelerate the calculation of the posterior variance or covariance, a number of approximations rely on using nearest neighbors, including the local approximate GP \cite{laGP2015} and the Vecchia approximations \cite{vecchia1988,stein2004vecchia,vecchia2021general}.
For approximating the posterior variance $\text{Var}(x)$, 
the $k$ nearest neighbor approach replaces the kernel matrix $K_{SS}$ by $K_{\mathcal{N}_x,\mathcal{N}_x}$, where $\mathcal{N}_x$ denotes the $k$ nearest neighbors of $x$ from $S$, and also replaces other appearances of $S$ by $\mathcal{N}_x$.
For approximating $\RSsig(x,y)=\ksig(x,y)-K_{xS}K_{SS}^{-1}K_{Sy}$, the nearest neighbor approach will replace $K_{xS}$ by $K_{x\mathcal{N}_x}$, $K_{Sy}$ by $K_{\mathcal{N}_y y}$, and $K_{SS}$ by $K_{\mathcal{N}_y \mathcal{N}_x}$.
(Here $\mathcal{N}_x$ and $\mathcal{N}_y$ have the same size in order to obtain a square matrix.)
Note that the nearest neighbor approach does \emph{not} take into account of the different structures of $\abs{\RSsig(x,y)}$ caused by the bandwidth parameter $\rho$. Nor does it offer pointwise estimates of $\RSsig(x,y)$ under different scenarios.
As demonstrated by the experiment below, simply replacing $S$ with a potentially smaller set of neighbors can give incorrect results.
We present an experiment below to compare the $k$ nearest neighbor-based approximation and the proposed indicators in Section \ref{sub:post} for $\RSsig(x,y)$ and $\Var(x)$.

Consider the setup in Section \ref{sub:Error pattern1}, where $S$ consists of $5$ equispaced points.
The bandwidth is set to $\rho=0.4$, representing the \emph{large} bandwidth regime.
For posterior variance $\Var(x)$, we plot the true value, nearest neighbor estimation, and proposed estimation in Figure \ref{fig:knnVar}.
For posterior covariance, we plot in Figure \ref{fig:knnCov} the three curves for $\RSsig(x,0.4)$: true, nearest neighbor, proposed.

It is easy to see from Figures \ref{fig:knnVar} and \ref{fig:knnCov} that the nearest neighbor approach is generally \emph{unable} to capture the behavior of the posterior variance and covariance.
In general, it is accurate for $\Var(x)$ only when $k$ is sufficiently large (see Figure \ref{fig:knnVar}), and is invalid for $\RSsig(x,0.4)$ regardless of the value of $k$ - the number of neighbors. For estimating $\Var(x)$ in Figure \ref{fig:knnVar}, the nearest neighbor approach achieves good results when $k=4$ neighbors are used (left plot), which consists of $80\%$ of points from $S$.
For $k\leq 3$ (middle and right plots), i.e. using at most $60\%$ of points in $S$ as neighbors,
the nearest neighbor approach produces an incorrect estimation of $\Var(x)$ in terms of both shape and magnitude.

For the much more challenging posterior \emph{covariance} estimation in Figure \ref{fig:knnCov}, the nearest neighbor approach fails regardless of the choices of $k$.
The estimated shape (line with dots) has no resemblance to the true curve (solid line).
Consider the case with $k=4$ neighbors ($80\%$ of points from $S$) in Figure \ref{fig:knnCov} (Left).
We see that the true posterior $\RSsig(x,0.4)$ is larger at the two peaks around $x=0.1, 0.9$, and the slightly lower peaks at around $x=0.38,0.63$ have very similar magnitudes. 
This is \emph{not} captured by the nearest neighbor approximation as the curve displays an entirely different pattern: the four local maximum values decreases substantially as $x$ increases.
The result is even worse as $k$ becomes smaller.
With $k=3$ neighbors ($60\%$ of points in $S$) in Figure \ref{fig:knnCov} (Middle), the nearest neighbor curve displays a strange ``cusp'' in the two peaks near $x=0.4, 0.8$, which is not found in the case of $k=4$ or $k=2$.
This highlights the issue of the sensitivity of the result to the choice of $k$.
In practice, the suitable value of $k$ is \emph{not} straightforward to determine.

According to Figures \ref{fig:knnVar}, \ref{fig:knnCov},
the proposed estimates (dashed line) correctly capture the behavior of posterior variance and posterior covariance, compared to the nearest neighbor approach.
Overall, the results here reflect again the important role of the bandwidth parameter, which is not systematically discussed in the existing literature. 
Existing methods, such as the nearest-neighbor approach, can become ineffective when the bandwidth is in a certain range.
The issues can be seen in Figures \ref{fig:knnVar}, \ref{fig:knnCov}.
In contrast, the proposed analysis covers different cases as the bandwidth varies and leads to a more accurate characterization of the posterior distributions.
}

\begin{figure}[htbp] 
    \centering 
    \includegraphics[scale=.38]{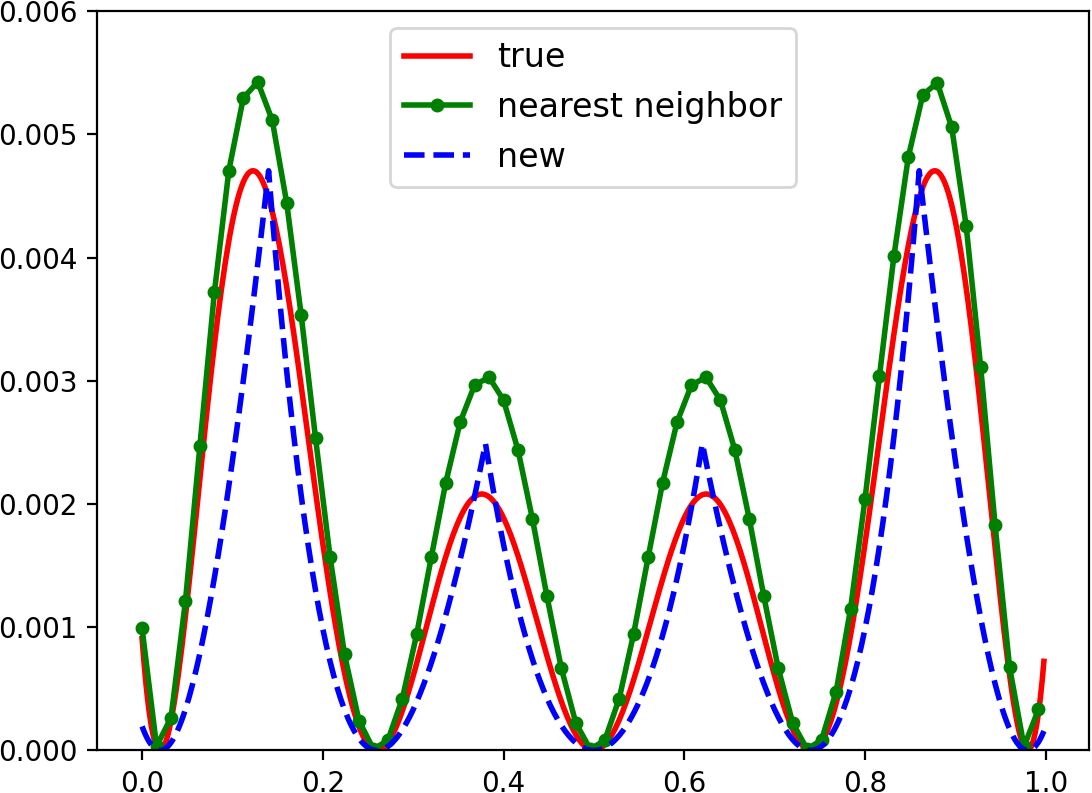} 
    \includegraphics[scale=.38]{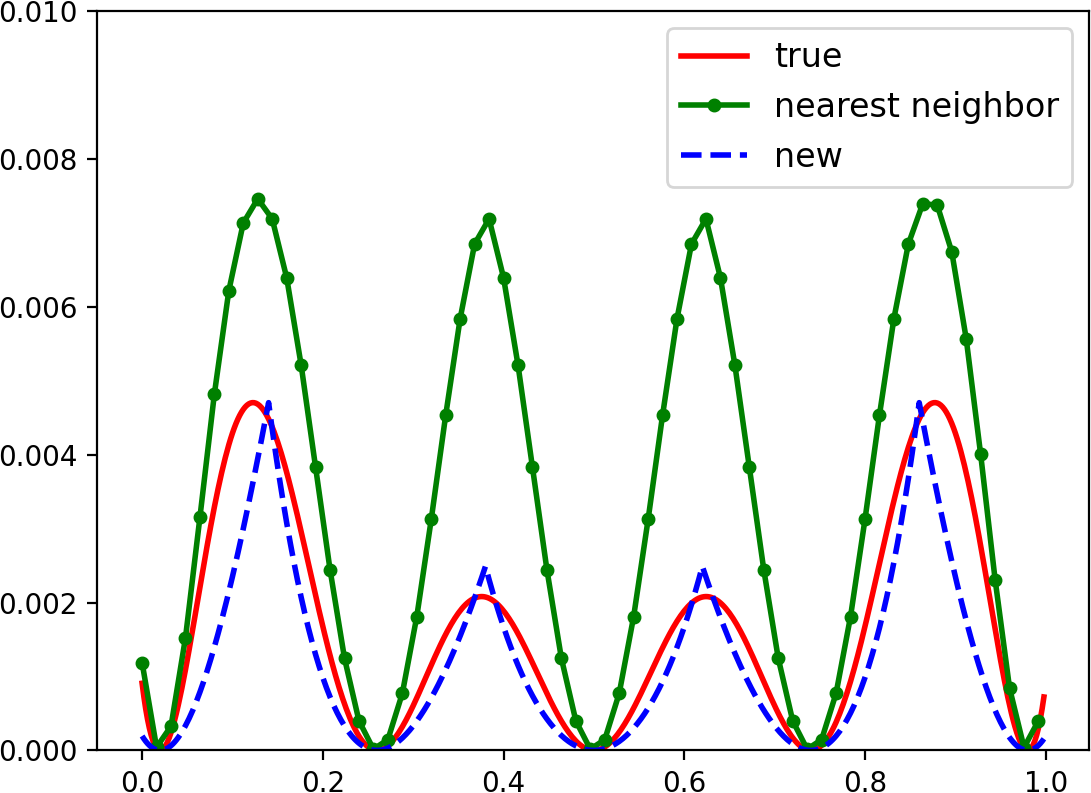} 
    \includegraphics[scale=.38]{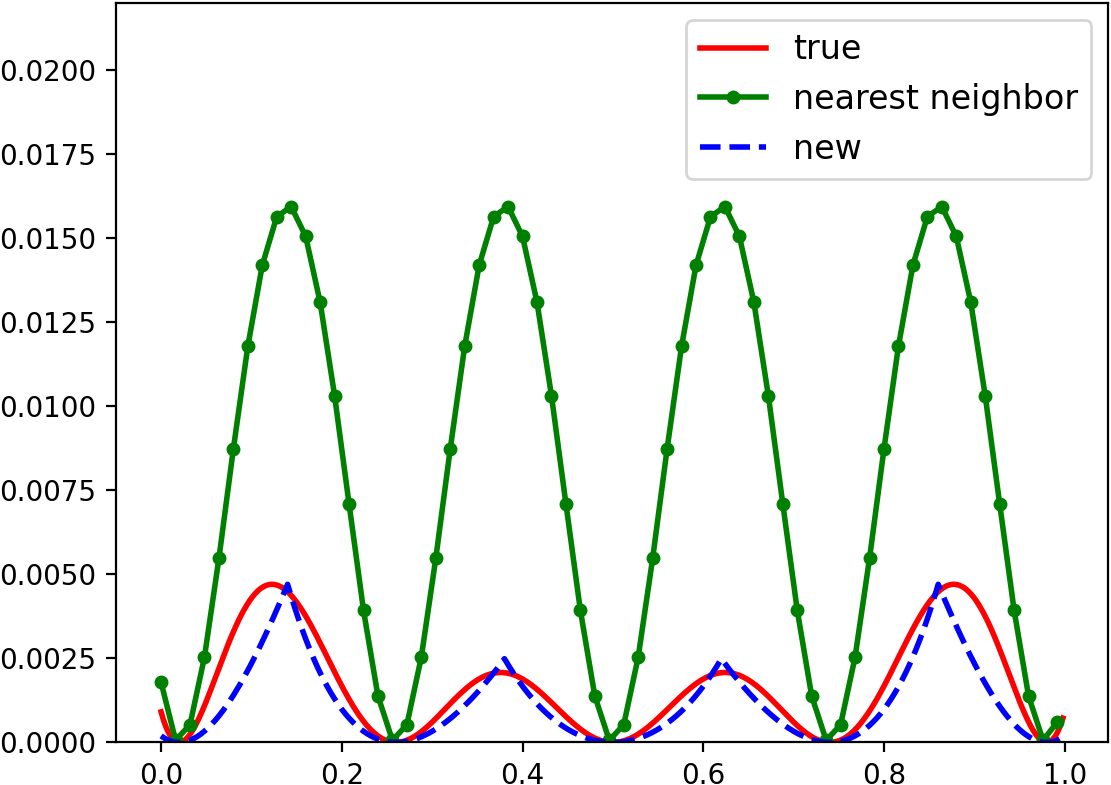} 
    \caption{\cdf{True variance $\text{Var}(x)$ (solid), $k$ nearest neighbor (dot), proposed (dashed). Left to right: $k=4,3,2$ neighbors from $S$ with $5$ points. The curves for true variance and proposed won't change with $k$, but the nearest neighbor approximations will display potentially different scales for different $k$ (thus quite inaccurate)}}
    \label{fig:knnVar}
\end{figure}

\begin{figure}[htbp] 
    \centering 
    \includegraphics[scale=.38]{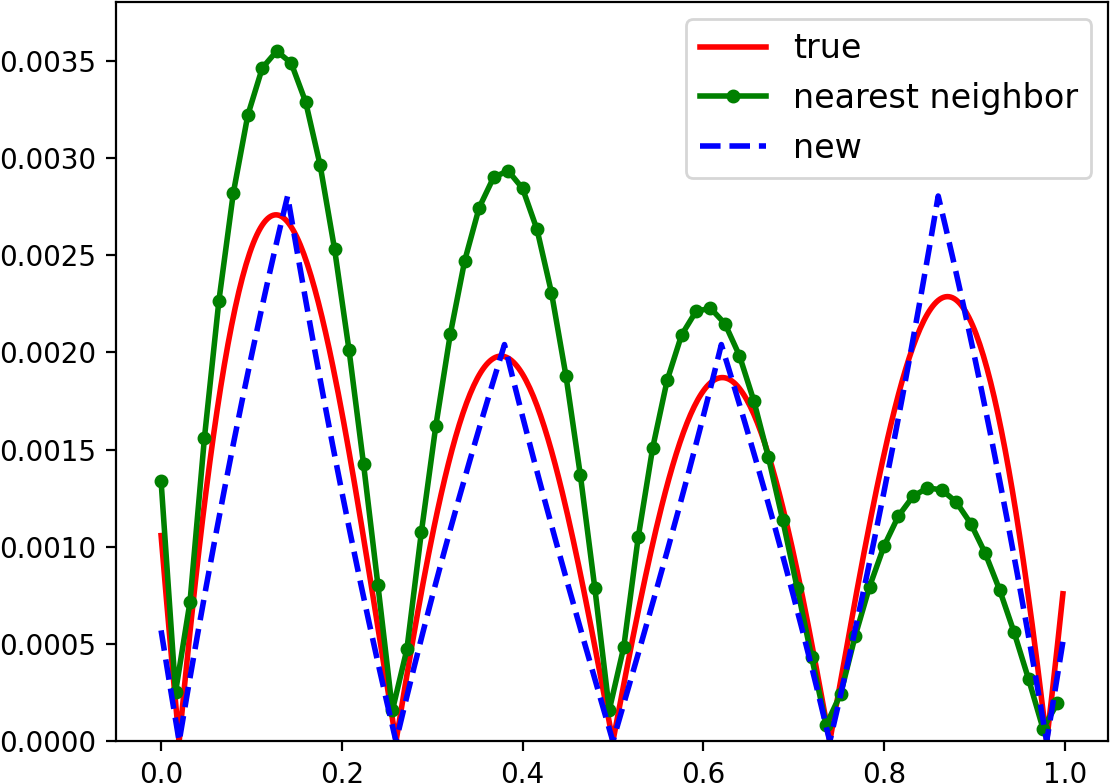} 
    \includegraphics[scale=.38]{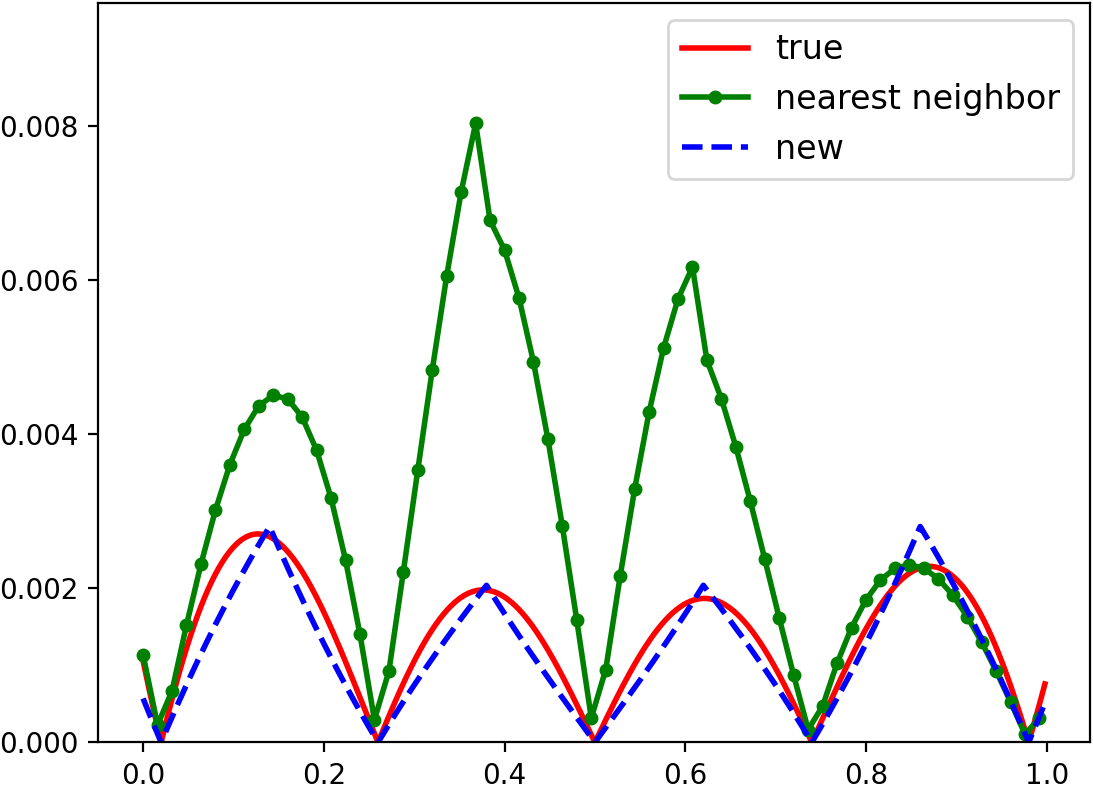} 
    \includegraphics[scale=.38]{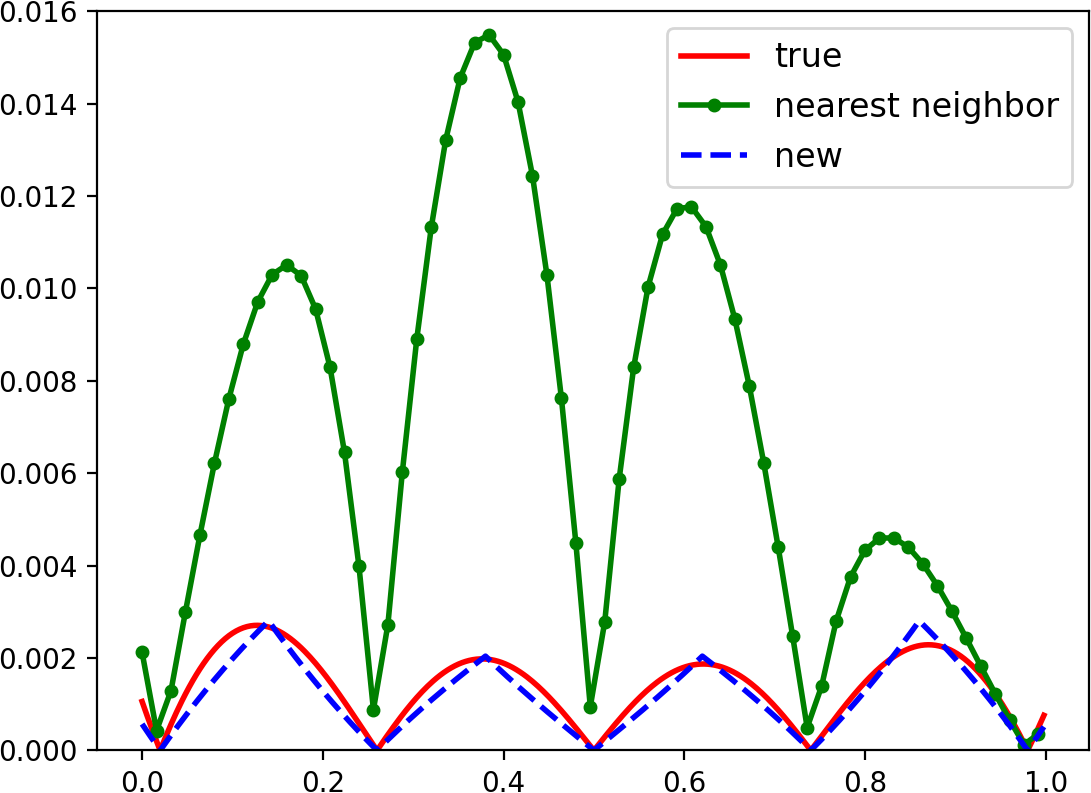} 
    \caption{\cdf{True covariance $\RSsig(x,0.4)$ (solid), $k$ nearest neighbor (dot), proposed (dashed). Left to right: $k=4,3,2$ neighbors from $S$ with $5$ points. The curves for the true and the proposed won't change with $k$, but the nearest neighbor approximations will display potentially different scales for different $k$ (thus quite inaccurate). The nearest neighbor approach fails to capture the true behavior for all choices of $k$ (unless using all points in $S$ as neighbors)}}
    \label{fig:knnCov}
\end{figure}

\subsection{Low-rank plus sparse (LRSP) matrix approximations}
\label{sub:app-Approximation}
In this experiment, we consider economical approximations to the dense kernel matrix $K_{XX}$ associated with a finite set $X$.
Low-rank approximations are widely used as an economical surrogate for $K_{XX}$.
However, affected by the bandwidth parameter $\rho$, $K_{XX}$ may not have fast decaying singular values when $\rho$ is not large and thus low-rank approximation can be inaccurate.
In this case, employing a sparse correction may help to achieve better accuracy than using a low-rank approximation with increased rank.

We compare the low-rank approximation to $K_{XX}$ and the low-rank plus sparse approximation (LRSP).
The low-rank approximation is in the form: 
$K_{XX}\approx K_{XS}K_{SS}^{-1}K_{SX}$.
Here $S\subseteq X$ can be viewed as the observation data in Gaussian processes. 
The residual matrix $K_{XX}-K_{XS}K_{SS}^{-1}K_{SX}$ can be viewed as the evaluation of the posterior covariance function $\RSsig(x,y)$ over a finite set of points $X\times X$.
Hence the theory in Section \ref{sec:theory} can be used to detect large entries in the residual matrix for the construction of the sparse correction matrix.
Starting with a baseline low-rank approximation, 
we investigate the effect of low-rank correction and sparse correction for increasing the approximation power (or reducing the approximation error) as follows.
The baseline low-rank approximation is chosen as the \nys approximation \cite{nys2001}:
$$\hat{K}_{r_0}=K_{XS}K_{SS}^{-1}K_{SX},$$ 
where $S$ is a subset of $X$ with $r_0=100$ randomly chosen points.
The low-rank correction increases the size of $S$ to contain $r>r_0$ points, leading to a more accurate approximation.
The sparse correction utilizes the LRSP approximation
$$K_{XX}\approx \hat{K}_{r_0}+R_{\text{sp}},$$
where the sparse correction $R_{\text{sp}}$ is chosen to be a submatrix of $\RSsig(X,X)=K_{XX}-K_{XS}K_{SS}^{-1}K_{SX}$.

In the experiment,
we choose $X$ to be 1000 random samples from the standard normal distribution in three dimensions.
The set $S$ contains uniform random samples from $X$.
We remark that the goal is to show how the sparse correction with a well-chosen nonzero pattern can improve the accuracy of the low-rank approximation .
The choice of $S$ for the low-rank approximation is not the focus.
For discussions on how to choose the landmark points $S$ for \nys low-rank approximations, we refer to \cite{nys2008kmeans,nys2005,nys2012,nys2017,nys2010kmeans,anchornet,ddblock}.
We choose the bandwidth to be $\rho=0.5$.
To carry out a fair comparison between the ``rank-$r$'' approximation and the ``rank-$r_0$ plus sparse'' approximation,
we compare the error under the same storage.
To this end, we define the \emph{cost-equivalent rank} for the LRSP approach as the number $k$ that satisfies the following equation:
\begin{equation}
\label{eq:rankmap}
    k^2+Nk = r_0^2 + Nr_0 + \texttt{nnz},
\end{equation}
where \texttt{nnz} denotes the \emph{number of nonzeros} in matrix $R_{\text{sp}}$.
The right-hand side of \eqref{eq:rankmap} is the number of matrix entries required (which represents the storage cost) for the LRSP format.
The left-hand side of \eqref{eq:rankmap} is the cost of an ``imaginary'' low-rank format such that the storage is the same as the LRSP format, where $k$ represents the size of the imaginary observation set $S$.
By using the \emph{cost-equivalent rank}, we are able to illustrate the error curve for the LRSP approximation in the same ``error-rank'' plot as the low-rank approximation.

The result is shown in Figure \ref{fig:randn-s05}.
The rank for the low-rank approximation increases from 100 to 660. 
For LRSP approximation, with fixed rank $r_0=100$,  sparse corrections with increasing density are used to match the rank increase in the low-rank approach.
The nonzero pattern in the sparse correction is chosen as follows.
For $\rho=0.5$, the threshold for choosing the pairs of nearby points (based on $\dist{x,y}$) goes from $\rho$ to $10\rho$. 
Namely, if $\dist{x,y}$ is within the threshold, then the corresponding entry will be nonzero in the sparse correction matrix.
A larger threshold yields more nonzero entries in the sparse correction.
Note that here we ignore $\dist{x,S}$ or $\dist{y,S}$ and the choice can include $x$ or $y$ close to $S$. However, this is not a concern since in the small bandwidth case, the number of points close to $S$ is almost negligible compared to the number of points away from $S$, as can be seen in Figure \ref{fig:case1smallsigma} where the green band (large values of $\abs{\RSsig(x,y)}$) almost covers the entire diagonal, especially when $\rho$ is small.

The error matrix $E$ is evaluated in the max norm $\norm{E}_{\max}=\max_{i,j} \abs{E_{ij}}$ as well as the approximate 2-norm: $\norm{Ev}/\norm{v}$ where $v$ is a random vector whose entries are drawn independently from the standard normal distribution.
It can be seen from Figure \ref{fig:randn-s05} that LRSP offers much better accuracy than the low-rank approximation under the same storage.
As the approximation rank increases, the low-rank approximation achieves no improvement for the max-norm approximation accuracy and little improvement for the 2-norm approximation accuracy.


\begin{figure}[htbp] 
    \centering 
\subfloat[Error matrix in max norm]{ 
\label{fig:randn-max-s05}
\includegraphics[scale=.39]{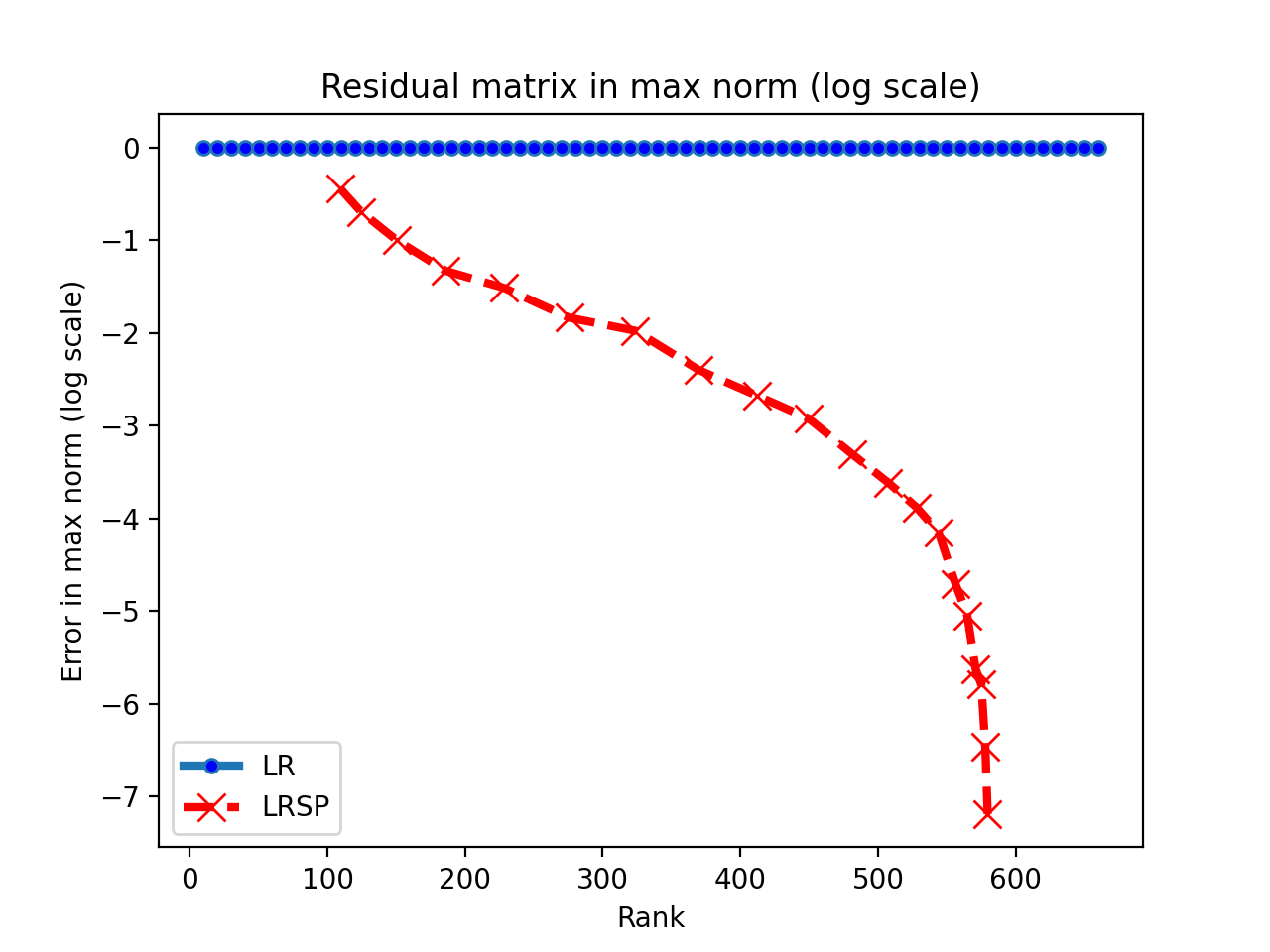}}
\hspace{10pt}
\subfloat[Error matrix in operator 2-norm]{
\label{fig:randn-L2-s05}
\includegraphics[scale=.39]{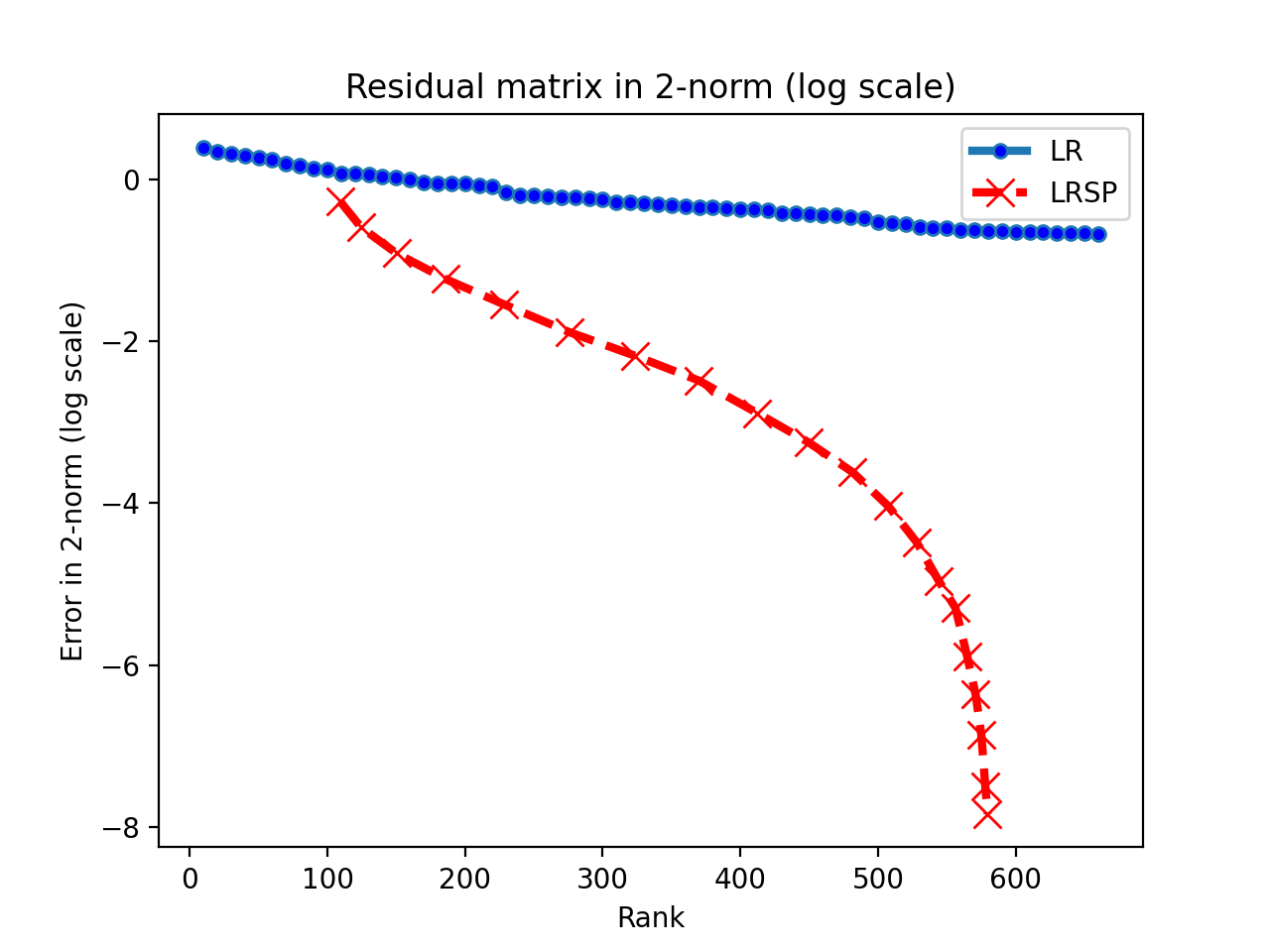}}
\caption{LR and LRSP approximations for $K_{XX}$ with $\rho=0.5$ and $X\subseteq \mathbb{R}^3$.}
\label{fig:randn-s05}
\end{figure}


\subsection{Preconditioning} 
\label{sub:app-Preconditioning}

In this experiment, we consider using the new result to guide the selection of the sparsity pattern used in the approximate inverse type preconditioners (\cite{fsai1993,fsai2000}) for solving Gaussian linear systems $Kz=b$,
which arise in kernel ridge regression \cite{10027498}, radial basis function interpolation, Gaussian processes \cite{10.1145/3589335.3651456}, etc.
For the Gaussian kernel matrix, the most recent development of FSAI-type preconditioners is the adaptive factorized \nys (AFN) preconditioner \cite{AFN}. AFN reorders the coefficient matrix into a $2\times 2$ block form corresponding to $S$ and $X\backslash S$, and applies FSAI preconditioner to approximate the Schur complement inverse. This FSAI approximation step coincides with the kernel matrix associated with the kernel $\RSsig(x,y)$. 

More specifically, let $S$ denote a small subset of $X$ selected for the skeleton low-rank approximation $K_{XX}\approx K_{XS}K_{SS}^{-1}K_{SX}$, and define $T=X\backslash S$.
Without loss of generality, assume the matrix $K_{XX}$ is partitioned as
$K_{XX} = \begin{bmatrix}    K_{SS} & K_{ST}\\
    K_{TS} & K_{TT}
\end{bmatrix}.$
Denote $R_{TT}$ as the 
Schur complement $K_{TT}-K_{TS}K_{SS}^{-1}K_{ST}$.
In AFN \cite{AFN}, FSAI \cite{fsai1993,fsai2000} is then used to approximate $R_{TT}^{-1}$:
\begin{equation}
    R_{TT}^{-1}\approx G^T G,
\end{equation}
where $G$ denotes the sparse lower triangular Cholesky factor for the approximate inverse of $R_{TT}$. Finally, the AFN preconditioner $M$ takes the following form:
\begin{equation}
    M =      
    \begin{bmatrix}
        L & 0 \\
        K_{TS}L^{-\top} & G^{-1}
    \end{bmatrix}
    \begin{bmatrix}
        L^{\top} & L^{-1}K_{ST} \\
        0 & G^{-\top}
    \end{bmatrix},
\end{equation}
where $L$ is the Cholesky factor of $K_{SS}$.
{
For FSAI-type preconditioners, specifying the nonzero pattern of the triangular factor $G$ when preconditioning $R_{TT}$ usually relies on the sparsity pattern for $R_{TT}$.
In Gaussian processes, $R_{TT}$ is unknown a priori and computing the entire matrix $R_{TT}$ can be quite costly for large scale data.
Since $R_{TT}=\RSsig(T,T)$, without computing $R_{TT}$ exactly, the analysis in Section \ref{sec:theory} presents a characterization of the dominant entries in $R_{TT}$, which can be used to efficiently specify the nonzero pattern of the sparse factor $G$, similar to Section \ref{sub:app-Approximation}.
We perform experiments below to illustrate the effectiveness.
}

To evaluate the performance of preconditioners, we compare the following methods:
\begin{itemize}
    \item Method 1: Solve $Kz=b$ using unpreconditioned CG.
    \item Method 2: Solve $Kz=b$ using preconditioned CG where the preconditioner is the AFN preconditioner \cite{AFN} using a \emph{random} nonzero pattern for the triangular factor $G$. The number of nonzeros per row does not exceed 10\% of the column size.
    \item Method 3 (proposed):
Solve $Kz=b$ using preconditioned CG where the preconditioner is the AFN preconditioner \cite{AFN} using a nonzero pattern for $G$ corresponding to the geometric condition $\norm{x_i-x_j}\leq \delta$ (with $i\leq j$ for the lower triangular structure), inspired by Theorem \ref{thm:lowerbound} in Section \ref{sub:Case II} for characterizing large entries in $\abs{\RSsig(x,y)}$.
\end{itemize}

The experiment setup is below.
$X$ is either synthetic data (samples from the standard normal distribution) or real world data in $\mathbb{R}^d$.
The real world data we use is the California Housing data set subsampled to 5000 points randomly.
The real world data set is standardized such that it has zero mean and unit variance.
The choice of bandwidth follows \cite{caputo2002}.
It is chosen as the value at the 2nd percentile of all the pairwise distances $\dist{x_i,x_j}$ ($i\neq j$) ordered increasingly.
The rank for the skeleton approximation is chosen as $r=0.2N$.
For the proposed method, we choose $\delta=2\rho$.
For CG, the maximum number of iterations is set to be 1000 and the stopping tolerance for the (absolute) residual norm is set to be  $10^{-5}$.
The results are shown in Tables \ref{tab:precond1} -- \ref{tab:precond2}.
It can be seen that Method 3 achieves the best accuracy with only a small number of iterations.
FSAI with a random nonzero pattern (Method 2) for the triangular factor $G$ yields poor performance.
We remark that, if the bandwidth $\rho$ is relatively large and the kernel matrix is more numerically low-rank, it is more efficient to use the low-rank preconditioners based on the Sherman-Morrison-Woodbury formula.

\begin{table}
\caption{Preconditioning test for synthetic data \texttt{randn(1000,3)}.}
\label{tab:precond1}
\begin{center}
\begin{tabular}{c|c|c|c}
\hline
Method & \# iterations & $\frac{\norm{z-\hat{z}}}{\norm{z}}$ & $\norm{K\hat{z}-b}$\\
\hline 
Method 1: CG & 1000 & 1.16E-1 & 1.77E-5 \\
Method 2: random-precond-CG (FSAI nnz: $9.51\%$) & 1000 & 2.09E+02 & 1.00E-2 \\
Method 3: geometric-precond-CG (FSAI nnz: $6.95\%$) & 22 & 3.35E-2 & 9.22E-6\\
\hline 
\end{tabular}
\end{center}
\end{table}

\begin{table}
\caption{Preconditioning test for subsampled California Housing data ($N=5000,d=8$).}
\label{tab:precond2}
\begin{center}
\begin{tabular}{c|c|c|c}
\hline
Method & \# iterations & $\frac{\norm{z-\hat{z}}}{\norm{z}}$ & $\norm{K\hat{z}-b}$\\
\hline 
Method 1: CG & 1000 & 1.04E-1 & 8.80E-5 \\
Method 2: random-precond-CG (FSAI nnz: $9.50\%$) & 1000 & 1.24E+1 & 4.36E-2 \\
Method 3: geometric-precond-CG (FSAI nnz: $7.00\%$) & 8 & 3.52E-3 & 5.45E-6\\
\hline 
\end{tabular}
\end{center}
\end{table}

\subsection{Non-Gaussian covariance}
\label{sub:nonGaussian}
\cdf{The discussion so far is focused on the Gaussian covariance in \eqref{eq:cov} but can be translated to other covariance functions.
In this section, we present preliminary empirical results for the exponential covariance $\ksig(x,y)=\exp(-\frac{\norm{x-y}}{\rho})$.
The goal is to show that the idea makes sense for the exponential covariance as well, even though the indicator is \emph{not} specifically designed for the exponential covariance.

We first consider the setup in Figure \ref{fig:GP} but for the exponential covariance.
The indicator is directly from Section \ref{sub:post} by simply replacing the Gaussian covariance by the exponential covariance.
For Gaussian process regression, the learned hyperparameters are $\rho=0.0973812538296844$ and $\sigma^2=0.4632273115010307$.
The result is presented in Figure \ref{fig:GP-exp}.
We see that though the result is not as accurate as the Gaussian case (since the indicator is \emph{not} adapted to the exponential covariance), it still characterizes the uncertainty reasonably to some extent.

We now present two simple examples to compare the posterior covariance indicator to the true posterior covariance, where uniform and non-uniform observations are both tested, similar to Section \ref{sub:post}.
The plots are shown in Figure \ref{fig:est-exp}.
We see that the indicator roughly characterizes the behavior of the true posterior covariance.
There is one noticeable difference from the Gaussian case.
It can be seen that the green region (where the value of $\abs{\RSsig(x,y)}$ is relatively large) is ``sharper'' than the Gaussian case.
This is expected since the exponential covariance is less smooth than the Gaussian covariance.
Refined analysis to address this issue will be studied in future work.
}

\begin{figure}
    \centering
    \includegraphics[scale=0.55]{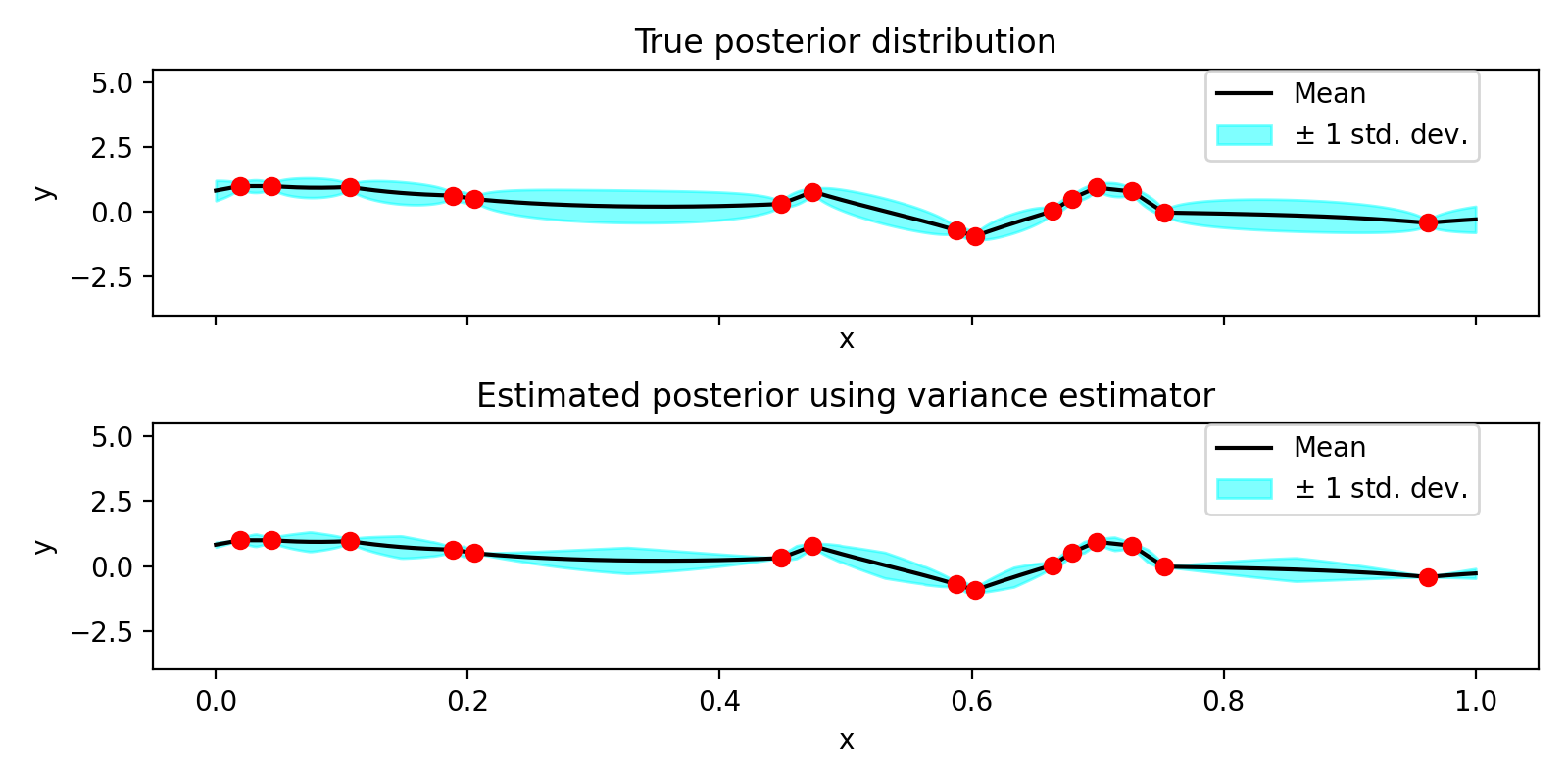}
    \caption{{Posterior regression curve from Gaussian process with exponential covariance. Shaded uncertainty: Top: true standard deviation; Bottom: indicator $\sqrt{\mathcal{V}}$ from Section \ref{sub:post}}}
    \label{fig:GP-exp}
\end{figure}

\begin{figure}
    \centering
    \includegraphics[scale=1.2]{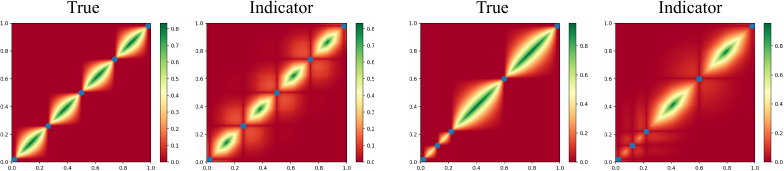}
    \caption{Posterior covariance vs the indicator for exponential prior covariance $\ksig(x,y)=\exp(-\frac{\lVert x-y\rVert}{\rho})$.}
    \label{fig:est-exp}
\end{figure}

\cdf{
\subsection{Noisy observations}
In this experiment, we consider noisy observations with various noise levels 
$$\tau=0.01,0.05,0.1,0.5,1.$$
Namely, the observed values are $f(x)+\epsilon$, with $\epsilon\sim\mathcal{N}(0,\tau^2)$.
The observation set $S$ in this experiment is the same as in Figure \ref{fig:GP}.
We use the same posterior variance indicator as in Figure \ref{fig:GP}.
The results are shown in Figures \ref{fig:GPnoise001} to \ref{fig:GPnoise1} with increasing noise level. 
It can be seen that, in every case, the estimated uncertainty is similar to the true uncertainty.
Hence the proposed indicators are able to capture the behavior of the true uncertainty for noisy input, even though the noise level is not the major focus of the paper.
}

\begin{figure}[htbp] 
    \centering 
    \includegraphics[scale=.7]{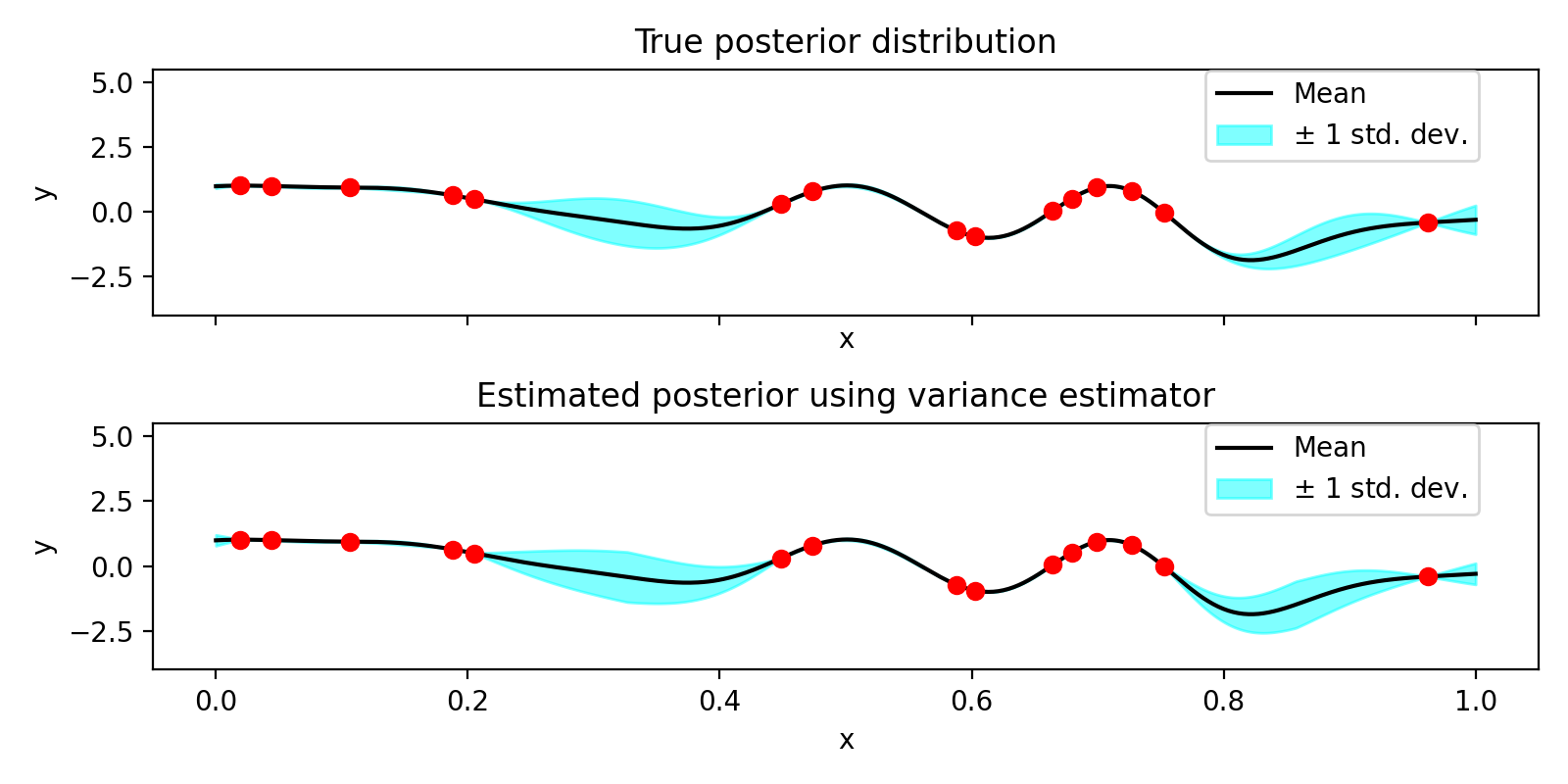}
    \caption{GP regression plots with noisy data  (noise level $\tau=0.01$)}
    \label{fig:GPnoise001}
\end{figure}
\begin{figure}[htbp] 
    \centering 
    \includegraphics[scale=.7]{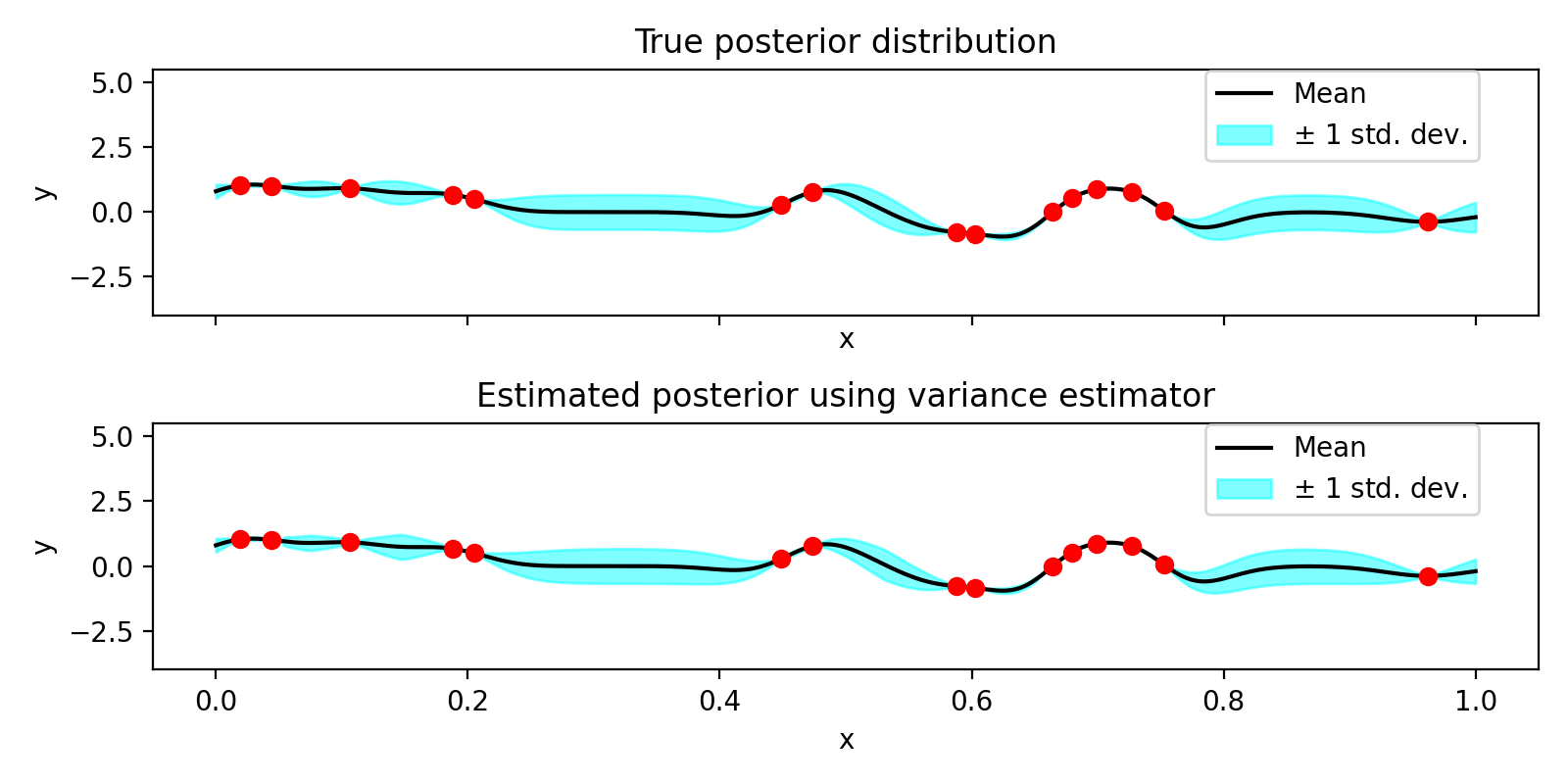}
    \caption{GP regression plots with noisy data  (noise level $\tau=0.05$)}
    \label{fig:GPnoise005}
\end{figure}
\begin{figure}[htbp] 
    \centering 
    \includegraphics[scale=.7]{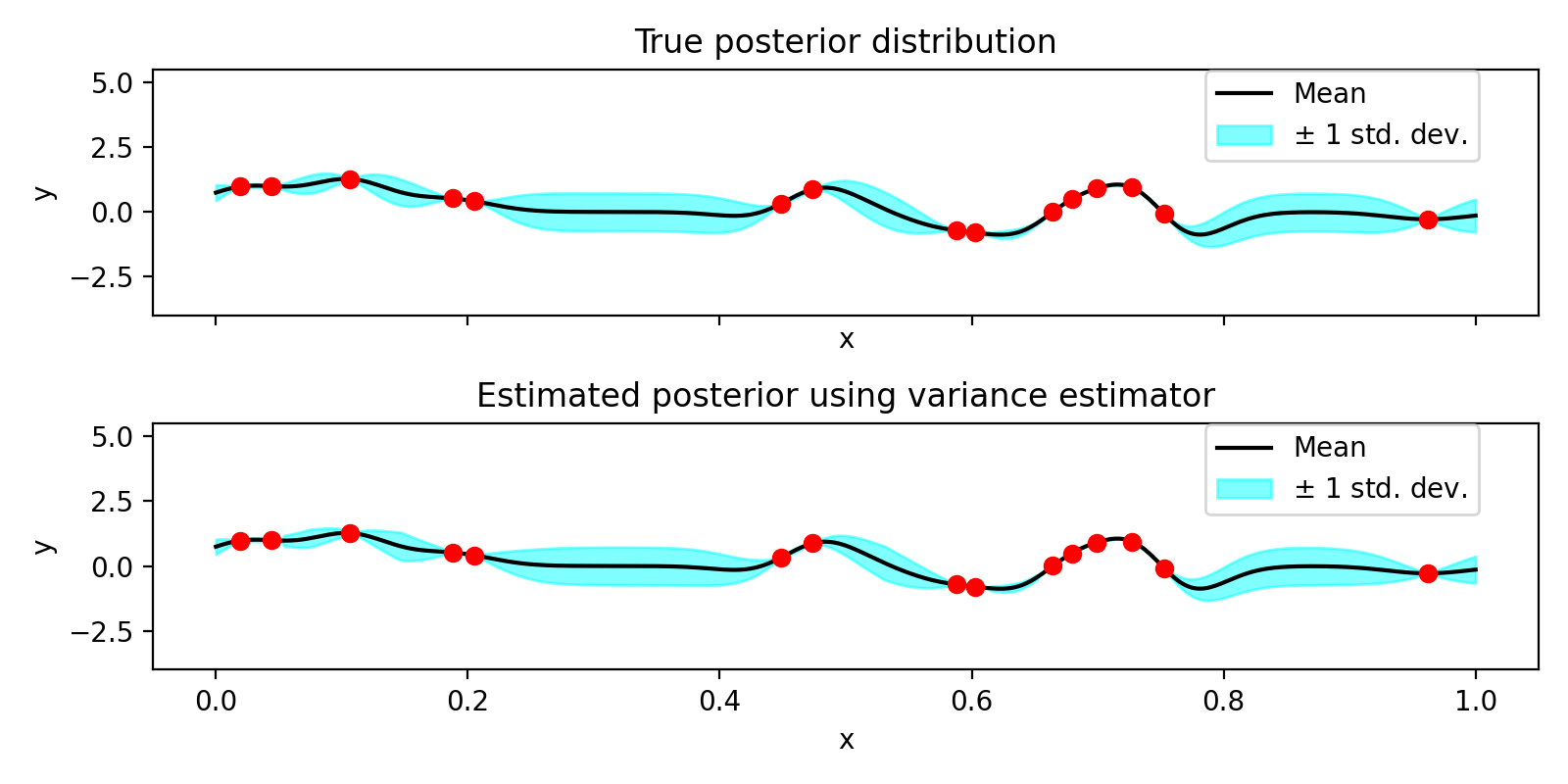}
    \caption{GP regression plots with noisy data (noise level $\tau=0.1$)}
    \label{fig:GPnoise01}
\end{figure}
\begin{figure}[htbp] 
    \centering 
    \includegraphics[scale=.7]{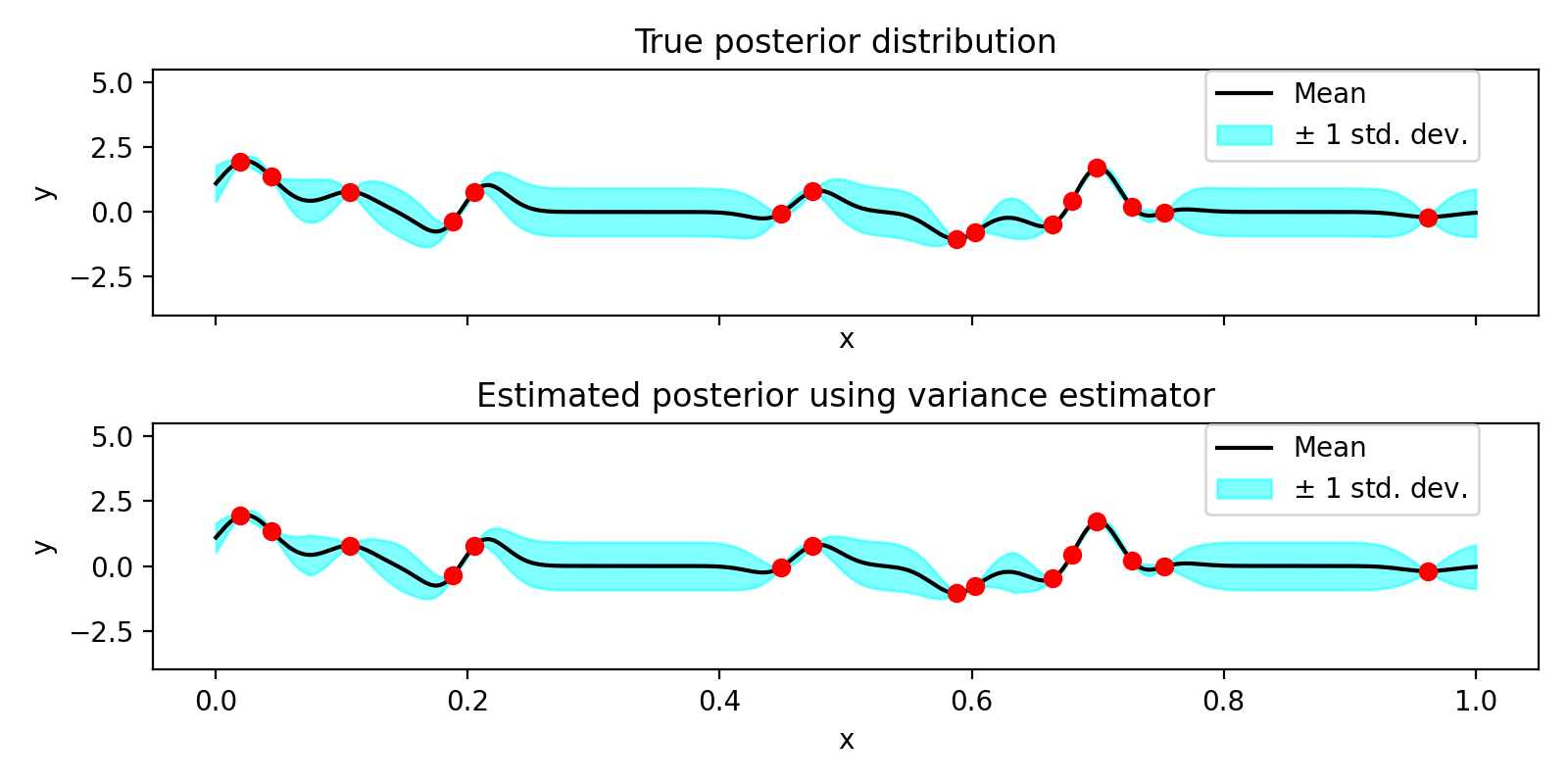}
    \caption{GP regression plots with noisy data  (noise level $\tau=0.5$)}
    \label{fig:GPnoise05}
\end{figure}
\begin{figure}[htbp] 
    \centering 
    \includegraphics[scale=.7]{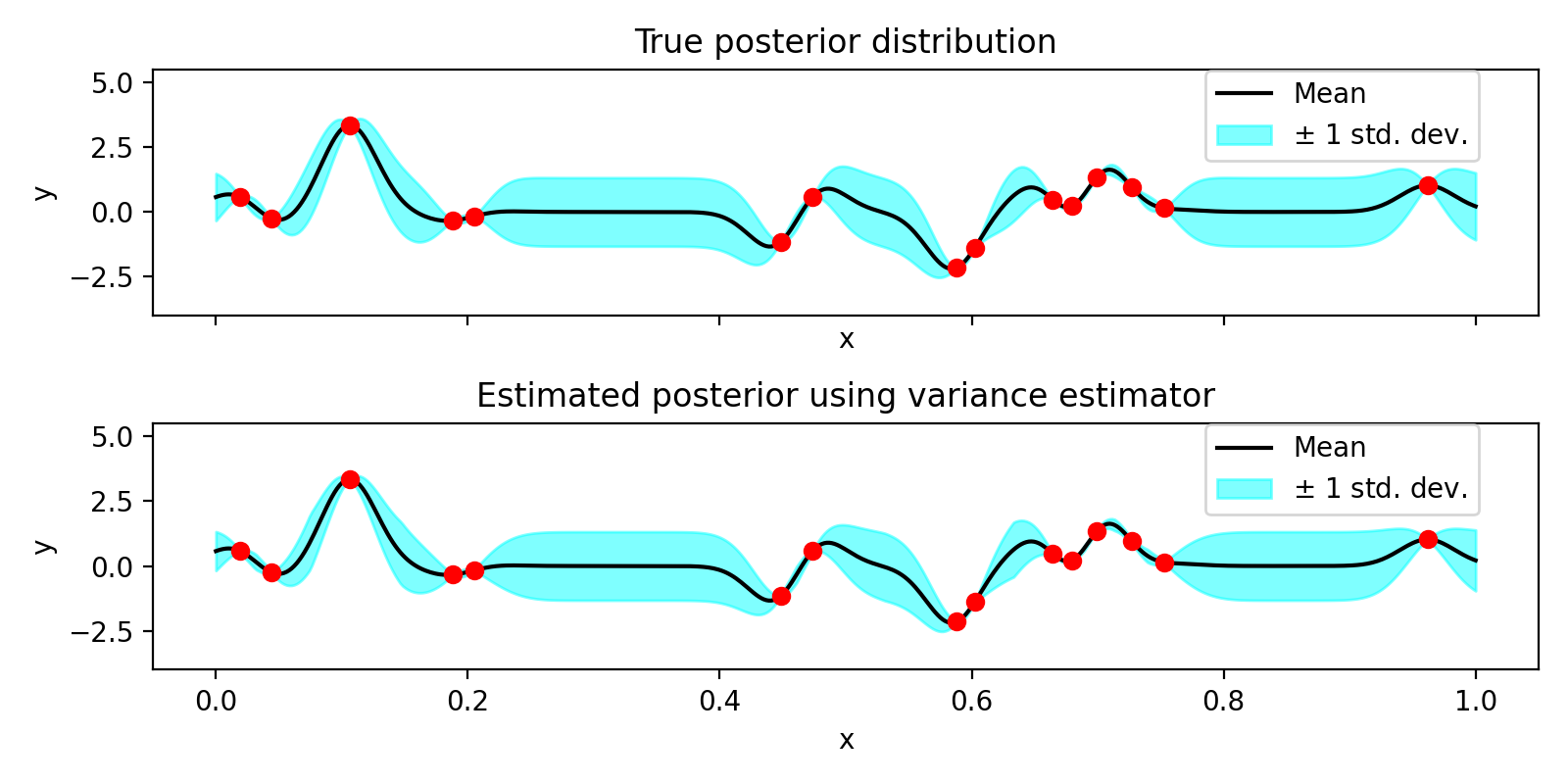}
    \caption{GP regression plots with noisy data  (noise level $\tau=1$)}
    \label{fig:GPnoise1}
\end{figure}

\section{Conclusion} 
\label{sec:conclusion}
We presented a detailed study on the posterior covariance function with an emphasis on the impact of the bandwidth and the observation data.
The result illustrates how the magnitude of the posterior covariance function at $(x,y)\in\Omega\times\Omega$ depends on the bandwidth parameter $\rho$ in the Gaussian kernel in \eqref{eq:cov}, the distance between the two points $\dist{x,y}$, and the closeness of $x,y$ to the set of observation points.
The theoretical results are accompanied by numerical demonstrations.
Inspired by the theoretical understanding and the a posteriori error estimation in the adaptive finite element method, practical indicators are presented to estimate the absolute posterior covariance function efficiently \emph{without} computing the matrix inverse in the definition in \eqref{eq:R}.
Applications to kernel matrix approximation, uncertainty quantification, and preconditioning are also discussed in numerical experiments.
Based on the current findings, we plan to study other covariance kernels and the case with noisy observations ($\tau>0$) in the future.
Another future direction is to study the use of the low-rank plus sparse representations and the preconditioning techniques to make the existing uncertainty quantification algorithms more scalable.

\bibliography{main}
\bibliographystyle{siamplain}
\end{document}